\documentclass{article}

\usepackage[preprint,main,nonatbib]{neurips_2026}         

\usepackage[utf8]{inputenc}
\usepackage[T1]{fontenc}
\usepackage{microtype}
\usepackage{comment}
\usepackage[backend=biber]{biblatex}
\usepackage{array}
\usepackage{amsmath}
\usepackage{amssymb}
\usepackage{amsthm}
\usepackage{mathrsfs}
\usepackage{mathtools}
\usepackage{bm}
\usepackage{nicefrac}

\usepackage[table]{xcolor}
\colorlet{revcolor}{black}

\usepackage{graphicx}
\usepackage{enumitem}
\usepackage{booktabs}
\usepackage{subcaption}
\usepackage{float}
\usepackage{tikz}
\usepackage{siunitx}
\usepackage{algorithm}
\usepackage{algpseudocode}
\makeatletter
\providecommand{\theHALG@line}{}
\renewcommand{\theHALG@line}{\thealgorithm.\arabic{ALG@line}}
\makeatother

\usepackage[colorlinks=true, bookmarksopen=true]{hyperref}

\usepackage{cleveref}

\definecolor{darkblue}{rgb}{0.0,0.2,0.6}

\theoremstyle{plain}
\newtheorem{thm}{Theorem}

\newtheorem{cor}[thm]{Corollary}

\newtheorem{prop}[thm]{Proposition}
\theoremstyle{definition}

\crefname{thm}{theorem}{theorems}
\Crefname{thm}{Theorem}{Theorems}
\crefname{lem}{lemma}{lemmas}
\Crefname{lem}{Lemma}{Lemmas}
\crefname{cor}{corollary}{corollaries}
\Crefname{cor}{Corollary}{Corollaries}
\crefname{prop}{proposition}{propositions}
\Crefname{prop}{Proposition}{Propositions}
\crefname{defn}{definition}{definitions}
\Crefname{defn}{Definition}{Definitions}
\crefname{ex}{example}{examples}
\Crefname{ex}{Example}{Examples}

\newcommand{\bb}[1]{\mathbb{#1}}
\newcommand{\diff }{\,\mathrm{d}}

\newcommand{\R}{\bb{R}}

\newcommand{\Id}{I_{\mathcal{X}}}
\newcommand{\sigk}{\sigma_\kappa}

\title{Gaussian Mixture Models in Hilbert Spaces via Kernel Methods}

\author{%
  Daniel López-Montero%
  \thanks{Equal contribution and shared first authorship.} \\
  Friedrich-Alexander-Universität Erlangen-Nürnberg \\
  \texttt{dani.lopez@fau.de} \\
  \And
  Antonio Álvarez-López\footnotemark[1] \\
  Universidad Autónoma de Madrid \\
Nanyang Technological University \\
  \texttt{antonio.al@ntu.edu.sg} \\
  \And
  Marcos Matabuena\thanks{Senior author}  \\
  Mohamed bin Zayed University of Artificial Intelligence \\
  \texttt{marcos.matabuena@mbzuai.ac.ae} \\
}

\begin{document}

\maketitle

\begin{abstract}
Modern datasets across many disciplines increasingly consist of time-evolving, potentially infinite-dimensional random objects, such as dynamic functional data, which are naturally modeled in Hilbert spaces. In these settings, characterizing probability measures, for example, through densities, can be ill-defined or technically challenging. Motivated by clustering applications, we propose a Gaussian mixture framework for Hilbert-space-valued data based on kernel mean embeddings and develop efficient optimization algorithms for estimation. We establish theoretical guarantees showing that the proposed algorithm is well-defined and that the model yields a dense class of approximations in infinite-dimensional spaces. We evaluate the framework through extensive experiments on diverse data geometries, including $L^2$-functional data and graph signals arising in modern medical applications. 
\end{abstract}

\section{Introduction}\label{sec:intro}

Technological progress across multiple domains is generating high-dimensional datasets that demand new analytical tools.  Modern data structures, such as random functions, spatial fields, images, signals, graphs, or trajectories, are often more naturally represented as random objects in Hilbert spaces than as finite-dimensional Euclidean vectors in $\mathbb{R}^d$ (\Cref{fig:data_overview}). Analyzing the data directly using Hilbert-space methods preserves their intrinsic geometric structure and allows their full potential to be exploited \cite{ramsayFunctionalDataAnalysis2005,ferratyNonparametricFunctionalData2006}.

In many applications, observed data are heterogeneous and may arise from $K$ latent populations, each characterized by a probability law $P_i$, with $i=1,\dots,K$. Mixture models \cite{fraleyModelBasedClusteringDiscriminant2002} provide a flexible framework for representing such data-generating mechanisms and, under suitable regularity conditions, can approximate broad classes of probability measures. This paper proposes a mixture-based framework for probability measures on a separable Hilbert space $\mathcal{X}$, extending the classical idea of Gaussian mixture models (GMMs) for finite-dimensional data \cite{dempsterMaximumLikelihoodIncomplete1977}.

Extending GMMs beyond finite-dimensional spaces is challenging. Infinite-dimensional Hilbert spaces do not admit a Lebesgue measure, so densities cannot be defined in the standard way, and likelihood-based methods typically require the choice of a reference measure \cite{dapratoStochasticEquationsInfinite2014,luSequentialMonteCarlo2026}.
We therefore adopt a density-free perspective: we fit Gaussian mixture models by minimizing the Maximum Mean Discrepancy (MMD) \cite{grettonKernelTwoSampleTest2012a} between the empirical measure and the target mixture model via their kernel mean embeddings.

\paragraph{Model and objective.}
Let $P$ be an unknown probability measure on a separable Hilbert space $\mathcal{X}$. A Gaussian measure on $\mathcal{X}$, denoted by $\mathcal{N}(m,\mathcal{K})$, is uniquely determined by a mean element $m\in\mathcal{X}$ and a covariance operator $\mathcal{K}\colon\mathcal{X}\to\mathcal{X}$, which generalizes the finite-dimensional covariance matrix. To ensure non-Dirac Gaussian components with finite total variance, $\mathcal{K}$ must be nonzero, self-adjoint, positive semidefinite, and trace-class.

For a fixed number of components $K \in \mathbb{N}$, we approximate $P$ with a finite Gaussian mixture
\begin{equation}\label{eq:mixture_model}
Q_\theta \coloneqq \sum_{k=1}^K \pi_k\,\mathcal{N}(m_k,\mathcal{K}_k),
\qquad
 \pi_k\ge 0\quad \forall k,\qquad \sum_{k=1}^K \pi_k=1,
\end{equation}
parameterized by $\theta \coloneqq (\pi_k, m_k, \mathcal{K}_k)_{k=1}^K \in \Theta_K$, where $\Theta_{K}$ denotes the admissible parameter space.

We fit $Q_\theta$ to $P$ by minimizing the Maximum Mean Discrepancy induced by a continuous positive-definite kernel $\kappa \colon \mathcal{X} \times \mathcal{X} \to \mathbb{R}$:
\begin{equation}\label{eq:objective}
\operatorname{MMD}_\kappa^2(P,Q_\theta) = \mathbb{E}_{X,X'\sim P}[\kappa(X,X')]
-2\,\mathbb{E}_{X\sim P,\;Y\sim Q_\theta}[\kappa(X,Y)]
+\mathbb{E}_{Y,Y'\sim Q_\theta}[\kappa(Y,Y')],
\end{equation}
where $X, X' \sim P$ and $Y, Y' \sim Q_\theta$ are mutually independent. This leads to a density-free fitting criterion that depends only on kernel expectations. When the kernel $\kappa$ is \emph{characteristic}, the MMD is a metric on probability measures. Important examples include Gaussian radial kernels on $\mathbb{R}^d$ \cite{sriperumbudurUniversalityCharacteristicKernels2011} and, more generally, on separable Hilbert spaces \cite{wynneKernelTwosampleTest2022}. The formal definitions and statistical background are deferred to \Cref{sec:prelim} in the supplementary material.

In practice, given $X_1,\dots,X_n \overset{\text{i.i.d.}}{\sim} P$, we replace $P$ by the empirical measure $P_n\coloneqq\frac{1}{n}\sum_{i=1}^{n}\delta_{X_i}$ and solve the problem
\begin{equation}\label{eq:objective_emp}
\min_{\theta\in\Theta_K}\operatorname{MMD}_\kappa^2(P_n,Q_\theta).
\end{equation}

\paragraph{Notation.}
Throughout, $\mathcal{X}$ denotes a separable real Hilbert space with inner product $\langle\cdot,\cdot\rangle_{\mathcal{X}}$, norm $\|\cdot\|_{\mathcal{X}}$, and identity operator $\Id$. We let $\mathcal{L}(\mathcal{X})$ denote the space of bounded linear operators on $\mathcal{X}$. For a self-adjoint trace-class operator $T \in \mathcal{L}(\mathcal{X})$ with eigenvalues $(\lambda_j)_{j\ge 1}$, its Fredholm determinant is $\det_F(\Id + T) \coloneqq \prod_{j=1}^\infty (1 + \lambda_j)$. Let $\mathcal{P}(\mathcal{X})$ denote the set of Borel probability measures on $\mathcal{X}$, and for $p \ge 1$, let $\mathcal{P}_p(\mathcal{X})$ be the subset of measures with finite $p$-th moments. For finite-dimensional spaces, boldface lowercase letters ($\mathbf{x},\mathbf{m}$) denote vectors and boldface uppercase letters ($\mathbf{K},\mathbf{I}$) denote matrices. We write $\mathbf{1}$ for the all-ones vector and $\Delta^{K-1}$ for the probability simplex.

\subsection{Contributions and paper organization}

\begin{figure}[t]
  \centering
   \includegraphics[width=0.15\linewidth]{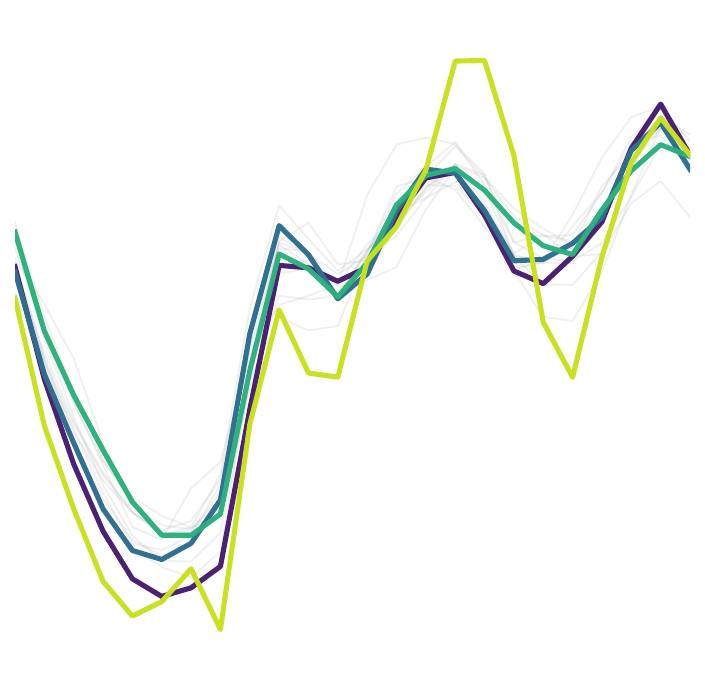}\qquad\qquad
    \includegraphics[width=0.15\linewidth]{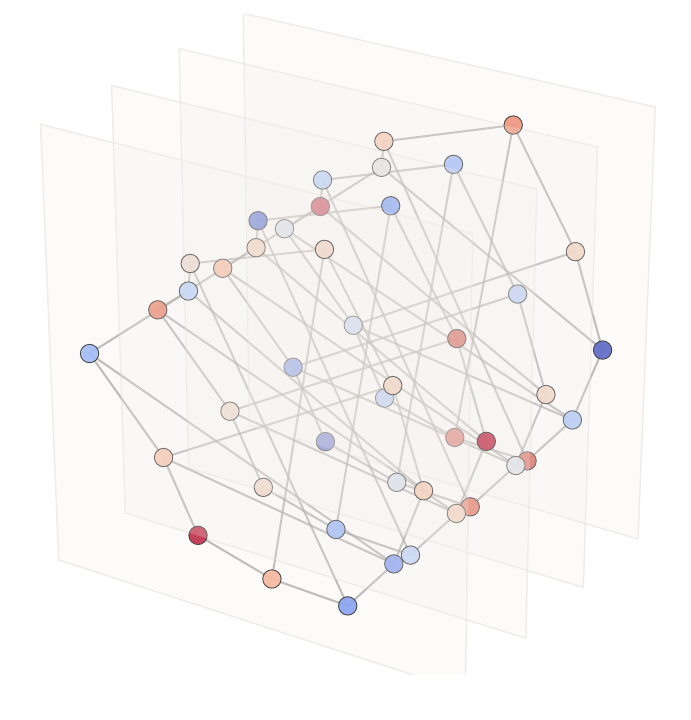}\qquad\qquad
  \includegraphics[width=0.15\linewidth]{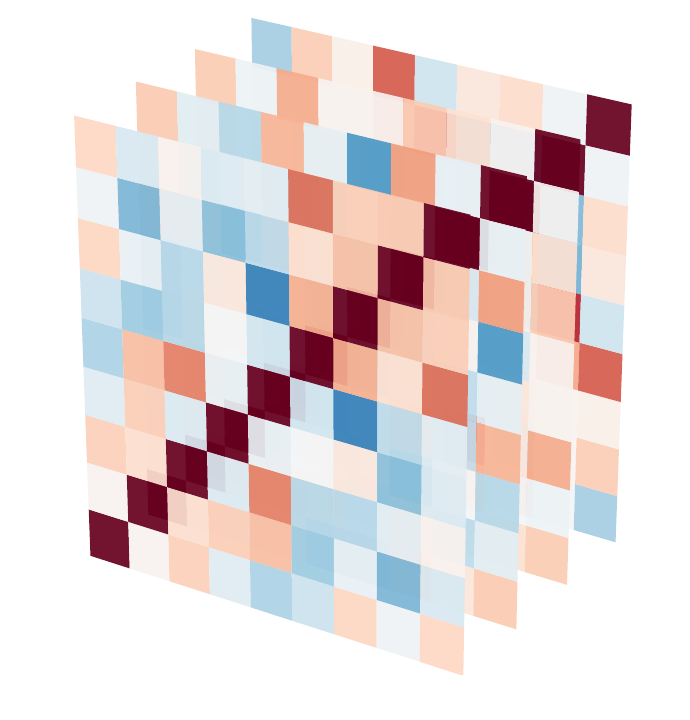}
  \caption{Representative structured data samples. \textbf{Left:} Hourly continuous glucose monitoring paths. \textbf{Center:} Signals on a fixed graph with varying node intensities. \textbf{Right:} Correlation matrices.}
  \label{fig:data_overview}
\end{figure}

The main contributions of this paper are threefold:
\begin{itemize}[leftmargin=*]
\item \textbf{Closed-form MMD fitting for Hilbert-space mixtures.}
We derive exact closed-form expressions for $\operatorname{MMD}_{\kappa}^{2}(P_n,Q_{\theta})$
for Gaussian radial basis and polynomial kernels, extending known
\(\mathbb{R}^d\) formulas to separable Hilbert spaces
\cite{alvarez-lopezContinuousTemporalLearning2025,briolDictionaryClosedFormKernel2025}. Moreover, we prove that, for every \(p \geq 1\), finite Gaussian mixtures of the form
\eqref{eq:mixture_model} are dense in \(\mathcal P_p(\mathcal X)\) under the
Wasserstein distance, providing theoretical support for the approximation capacity of the proposed model.

\item \textbf{Consistent finite-dimensional computation.}
We introduce a spectral projection framework that reduces infinite-dimensional computation to standard linear algebra in \(\mathbb{R}^M\).
We prove that both the projected MMD objective and the posterior probabilities converge to their exact infinite-dimensional counterparts as the truncation dimension \(M\) tends to infinity; see Propositions \ref{prop:projection_consistency} and~\ref{prop:projection_responsibilities}.

\item \textbf{Empirical validation and digital health case study.}
We evaluate the proposed method across diverse data structures, including random functions in \(L^2\), spatial fields, rotations in \(\mathrm{SO}(3)\), and graph signals, achieving competitive results against state-of-the-art approaches on established clustering benchmarks (\Cref{sec:experiments,app:synthetic_experiments}).
To further illustrate its utility on dynamic functional data, we analyze continuous glucose monitoring (CGM) data from a clinical trial to identify subgroups of individuals with heterogeneous responses to therapy.
\end{itemize}

All proofs and additional analytical results are deferred to the supplementary material. 

Code and Lean~4 proofs: \url{https://github.com/dani2442/rkhs-mixture-models/}.

\subsection*{Related work}

\emph{Functional data analysis and clustering.}
Functional data analysis (FDA) studies samples naturally viewed as functions, curves, or surfaces, typically using basis expansions, smoothing, and functional principal component analysis \cite{ramsayFunctionalDataAnalysis2005,ferratyNonparametricFunctionalData2006}. 
Our spectral projection framework is related to Karhunen--Lo\`eve expansions for Gaussian measures \cite{ramsayFunctionalDataAnalysis2005,ferratyNonparametricFunctionalData2006,bayKarhunenLoeveDecomposition2019}, but uses projection as a computational approximation to the MMD objective. 
Model-based functional clustering methods, such as Gaussian mixtures, usually assume latent groups with group-specific low-dimensional structure \cite{bouveyronModelbasedClusteringTime2011,chiouFunctionalClusteringIdentifying2007,jamesClusteringSparselySampled2003,giacofciWaveletBasedClusteringMixedEffects2013,jacquesFunclustCurvesClustering2013}. 
Such likelihood-based approaches require a reference measure, which is delicate in infinite dimensions \cite{bogachevGaussianMeasures1998,dapratoStochasticEquationsInfinite2014,delaigleDefiningProbabilityDensity2010}. 
Distance-based methods \cite{luzlopezgarciaKmeansAlgorithmsFunctional2015,ferreiraComparisonHierarchicalMethods2009,maLearningRepresentationsTime2019} avoid this issue, but typically lack the explicit probabilistic and temporal interpretation that a generative model provides.

\emph{Kernel mean embeddings and MMD.}
Kernel mean embeddings represent probability measures as elements of a reproducing kernel Hilbert space (RKHS) \cite{smolaHilbertSpaceEmbedding2007,muandetKernelMeanEmbedding2017,berlinetReproducingKernelHilbert2004}, inducing the Maximum Mean Discrepancy (MMD) \cite{grettonKernelTwoSampleTest2012a}. 
For characteristic kernels, MMD vanishes if and only if the two laws coincide \cite{sriperumbudurUniversalityCharacteristicKernels2011,sejdinovicEquivalenceDistancebasedRKHSbased2013}, making it a natural density-free fitting objective.
MMD fitting has been used in generative modeling \cite{liGenerativeMomentMatching2015,briolStatisticalInferenceGenerative2019}, parametric inference \cite{cherief-abdellatifFiniteSampleProperties2022,tolstikhinMinimaxEstimationKernel2017}, kernel-based clustering \cite{francaKernelKGroupsHartigans2017,matabuenaKernelBiclusteringAlgorithm2025}, distributional reinforcement learning \cite{antonettiAnalysisDistributionalReinforcement2025}, and temporal distributional learning \cite{alvarez-lopezContinuousTemporalLearning2025}. 
These approaches often rely on closed-form kernel mean embeddings for Gaussian distributions under Gaussian and polynomial kernels \cite{briolDictionaryClosedFormKernel2025}; related work also studies MMD testing on function spaces \cite{wynneKernelTwosampleTest2022}. 
We extend these formulas and fitting methods from $\R^d$ to separable Hilbert spaces, using characteristic-kernel MMD to fit finite Gaussian mixtures with closed-form objective evaluation, approximation guarantees, and projection-consistent computation. 

\emph{Mixture models.}
Gaussian mixture models are a classical tool for model-based clustering and density estimation \cite{dempsterMaximumLikelihoodIncomplete1977,fraleyModelBasedClusteringDiscriminant2002,chenAchievingOptimalClustering2024}. Extensions beyond Gaussian components have been proposed \cite{titteringtonStatisticalAnalysisFinite1985, tangWassersteinDistributionalLearning2023}, and more recent work applies mixtures of Gaussian processes to functional data \cite{trespMixturesGaussianProcesses2000,huangEstimatingMixtureGaussian2014,mcdowellClusteringGeneExpression2018} and to more general infinite-dimensional settings \cite{luSequentialMonteCarlo2026}. These methods remain largely likelihood- or projected-density-based; in contrast, we estimate Hilbert-space Gaussian mixtures by minimizing MMD, giving a density-free fitting criterion that is more natural in infinite-dimensional regimes.

\emph{Digital health.}
Digital health technologies generate unprecedented amounts of near-continuous data, motivating dynamic, time-varying functional models \cite{alvarez-lopezContinuousTemporalLearning2025,dubey2020functional,dubey2022modeling}. 
Mixture models for these general data structures remain underdeveloped, despite their potential in clinical trials for identifying subgroups of individuals with heterogeneous treatment responses.

\section{MMD fitting}\label{sec:mmd}

To minimize \eqref{eq:objective_emp}, we evaluate kernel expectations under $\nu_k \coloneqq \mathcal N(m_k,\mathcal K_k)$. Expanding the squared MMD between $P_n=\tfrac1n\sum_{i=1}^n\delta_{X_i}$ and $Q_\theta=\sum_{k=1}^K \pi_k\nu_k$ yields
\begin{equation}\label{eq:mmd_empirical}
\begin{aligned}
\operatorname{MMD}_\kappa^2(P_n,Q_\theta) &=
\frac{1}{n^2}\sum_{i,j=1}^{n} \kappa(X_i,X_j)
-\frac{2}{n}\sum_{i=1}^{n}\sum_{k=1}^K
\pi_k\,\mathbb{E}_{Y\sim \nu_k}[\kappa(X_i,Y)] \\
&\qquad + \sum_{k,s=1}^K
\pi_k\pi_s\,
\mathbb{E}_{Y\sim \nu_k,\;Y'\sim \nu_s}[\kappa(Y,Y')].
\end{aligned}
\end{equation}
Because the first data--data term is constant in $\theta$, optimization depends entirely on the data-to-component and component-to-component expectations:
\begin{equation}\label{eq:JI_definitions}
J_{i,k} \coloneqq \mathbb{E}_{Y\sim \nu_k}[\kappa(X_i,Y)],
\qquad
I_{k,s} \coloneqq \mathbb{E}_{Y\sim \nu_k,\;Y'\sim \nu_s}[\kappa(Y,Y')].
\end{equation}
To make the objective a tractable, differentiable function of $\theta$, we consider Gaussian and polynomial kernels, for which these expectations admit closed-form expressions (\Cref{app:mmd_closed_forms}):
\[
\kappa_A(x,y) = \exp\!\left(-\frac12\langle x-y,A(x-y)\rangle_{\mathcal X}\right)
\qquad \text{and} \qquad
\kappa(x,y) = (\langle x,y\rangle_{\mathcal X}+c)^p,
\]
where $A\in\mathcal L(\mathcal X)$ is bounded, self-adjoint, and positive semidefinite (often isotropic, $A=\sigma_\kappa^{-2}\Id$), $c \ge 0$, and $p \in \mathbb{N}$. The kernel $\kappa_A$ is characteristic whenever $A$ is injective \cite{wynneKernelTwosampleTest2022}. For the polynomial kernel of degree $p$, the objective depends on $Q_\theta$ only through its moment tensors up to order $p$, so fitting amounts to matching these moments.

\subsection{Fitting and optimization}\label{sec:fitting}

The parameters of the empirical objective \eqref{eq:mmd_empirical} naturally split into two groups: mixture weights $\pi \in \Delta^{K-1}$ and component parameters $\{(m_k,\mathcal{K}_k)\}_{k=1}^K$. This enables an alternating scheme that combines a convex weight update with a nonconvex update of the means and covariances.

To make the dependence on $\pi$ explicit, define the averaged cross-expectation vector $\mathbf{J} \in \mathbb{R}^K$ with entries $J_k \coloneqq \frac{1}{n}\sum_{i=1}^n J_{i,k}$ and the Gram matrix $\mathbf{I} \in \mathbb{R}^{K \times K}$ with entries $I_{k,s}$. Since the data--data term is constant, minimizing \eqref{eq:mmd_empirical} over $\pi$ for fixed components reduces to the quadratic program
\begin{equation}\label{eq:mmd_quadratic}
    \min_{\pi\in\mathbb{R}^K}\ \pi^\top \mathbf{I}\,\pi - 2\mathbf{J}^\top \pi
    \qquad \text{s.t.} \qquad
    \pi \ge 0, \quad \mathbf{1}^\top \pi = 1.
\end{equation}
Because \(I_{k,s} = \langle \mu_{\nu_k}, \mu_{\nu_s} \rangle_{\mathcal{H}_\kappa}\) is the inner product of the kernel mean embeddings $\mu_{\nu_k}, \mu_{\nu_s}$ in the RKHS $\mathcal{H}_\kappa$, \(\mathbf{I}\) is a Gram matrix; hence it is positive semidefinite and the program is convex. Geometrically, the optimal weights \(\pi^\star\) give the projection of the empirical mean embedding \(\mu_{P_n}\) onto the convex hull of the component embeddings \(\{\mu_{\nu_k}\}_{k=1}^K\).

Since joint optimization over the full parameter $\theta$ remains nonconvex, we employ an alternating minimization strategy, analogous to the expectation--maximization (EM) algorithm \cite{dempsterMaximumLikelihoodIncomplete1977}. In each iteration, we first solve \eqref{eq:mmd_quadratic} for $\pi$ with fixed components, and then fix $\pi$ to update the means and covariances via a finite-dimensional gradient step (\Cref{sec:numerics}). Pseudocode and implementation details are deferred to \Cref{app:algorithms}.

\subsection{Numerical computation}\label{sec:numerics}

For practical implementation, we project the Hilbert-space objects onto a finite-dimensional subspace. Fixing an orthonormal basis $(e_j)_{j\ge 1}$ of $\mathcal{X}$ (e.g., a cosine or data-driven Karhunen--Lo\`eve basis \cite{ramsayFunctionalDataAnalysis2005,bayKarhunenLoeveDecomposition2019}), we apply the orthogonal projector $\Pi_M:\mathcal{X}\to\mathcal{X}_M \coloneqq \mathrm{span}\{e_1,\dots,e_M\}$. This identifies $\mathcal{X}_M$ with $\mathbb{R}^M$, allowing each object to be represented by its coordinate vector.

For each component $\nu_k = \mathcal{N}(m_k,\mathcal{K}_k)$, let $\mathbf{m}_{k,M} \in \mathbb{R}^M$ and $\mathbf{K}_{k,M} \in \mathbb{R}^{M \times M}$ be the coordinate representations of the projected mean $\Pi_M m_k$ and covariance $\Pi_M\mathcal{K}_k\Pi_M$, and let $\mathbf{x}_{i,M} \in \mathbb{R}^M$ be the coordinates of $\Pi_M X_i$. Under this identification, the pushforward $\nu_{k,M} \coloneqq \nu_k\circ\Pi_M^{-1}$ is the Gaussian $\mathcal{N}(\mathbf{m}_{k,M},\mathbf{K}_{k,M})$ on $\mathbb{R}^M$.
The projected mixture and empirical measure on \(\mathbb{R}^M\) are
\begin{equation*}
    Q_{\theta,M}\coloneqq \sum_{k=1}^K\pi_k\nu_{k,M},
    \qquad
    P_{n,M}\coloneqq \frac{1}{n}\sum_{i=1}^n\delta_{\mathbf{x}_{i,M}}.
\end{equation*}
The projected cross-expectations are
\begin{equation}\label{eq:proj.cross.exp}
	J^{(M)}_{i,k}\coloneqq \mathbb{E}_{Y\sim \nu_k}\big[\kappa(\Pi_M X_i, \Pi_M Y)\big],
	\qquad
	I^{(M)}_{k,s}\coloneqq \mathbb{E}_{Y\sim \nu_k,\,Y'\sim \nu_s}\big[\kappa(\Pi_M Y, \Pi_M Y')\big].
\end{equation}
Computationally, these expectations are evaluated exactly in \(\mathbb{R}^M\) using the coordinate vectors \(\mathbf{x}_{i,M}\) and the finite-dimensional Gaussian measures \(\nu_{k,M}\). Explicit matrix formulas for \eqref{eq:proj.cross.exp} under Gaussian and polynomial kernels are deferred to \Cref{app:explicit_formulas_rm}. The finite-dimensional counterpart of \eqref{eq:mmd_empirical} is
\begin{equation}\label{eq:projected_mmd}
	\operatorname{MMD}_\kappa^2(P_{n,M}, Q_{\theta,M})
	= \frac{1}{n^2}\sum_{i,j=1}^n \kappa(\Pi_M X_i,\Pi_M X_j)
- 	\frac{2}{n}\sum_{i=1}^n\sum_{k=1}^K \pi_k J^{(M)}_{i,k}
	+ \sum_{k,s=1}^K \pi_k\pi_s I^{(M)}_{k,s}.
\end{equation}
Disregarding the constant data--data term, the per-iteration cost of evaluating \eqref{eq:projected_mmd} is $O(K^2 M^3 + n K M^2)$ for both the Gaussian radial and polynomial kernels. Moreover, if the projected covariances are assumed diagonal, the cost reduces to $O(K^2 M + n K M)$, making the method scalable to large sample sizes and projection dimensions. CPU timings ($n=10^5$, $K\leq200$) and details on numerical conditioning appear in \Cref{app:scaling,app:numerical_stability}.

\begin{figure}[t]
    \centering
    \begin{subfigure}[t]{0.32\textwidth}
        \centering
        \includegraphics[width=0.85\textwidth]{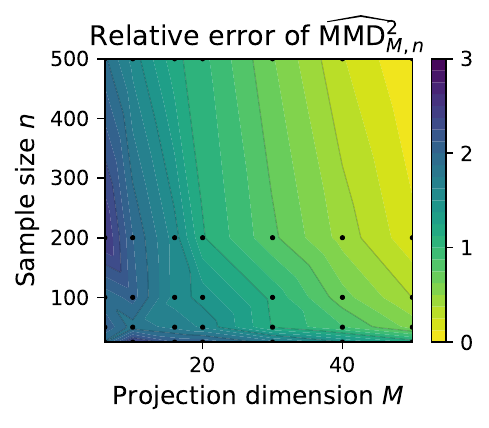}
        \caption{MMD$^2$ versus projection dimension $M$ and sample size $n$.}
        \label{fig:ablation_consistency}
    \end{subfigure}
    \hfill
    \begin{subfigure}[t]{0.32\textwidth}
        \centering
        \includegraphics[width=0.85\textwidth]{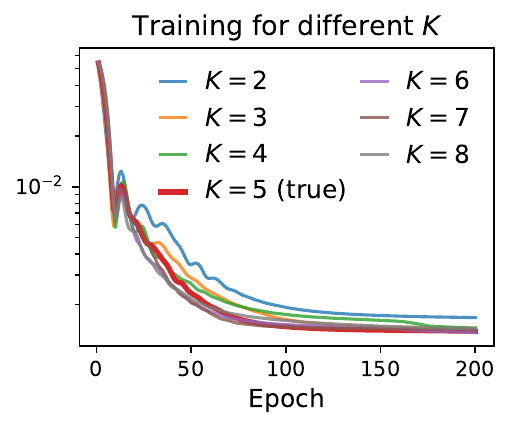}
        \caption{Training convergence for different numbers of components $K$.}
        \label{fig:ablation_loss_K}
    \end{subfigure}
    \hfill
    \begin{subfigure}[t]{0.32\textwidth}
        \centering
        \includegraphics[width=0.85\textwidth]{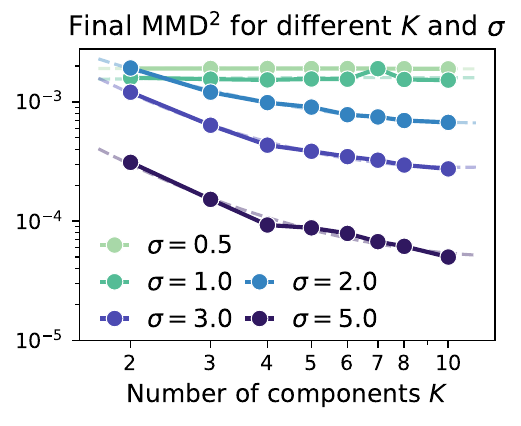}
        \caption{Final MMD$^2$ across component counts $K$ and kernel bandwidths $\sigma$.}
        \label{fig:ablation_k_sigma}
    \end{subfigure}
    \caption{
    Numerical sensitivity of the MMD objective to various hyperparameters and sample sizes.
    }
    \label{fig:numerics_ablation}
\end{figure}

\begin{prop}[Consistency of projected MMD]\label{prop:projection_consistency}
Let $\kappa$ be a continuous kernel on $\mathcal{X}$. Assume there exist constants $C,q\ge 0$ such that
\begin{equation}\label{eq:kernel_bound}
    |\kappa(x,y)| \le C\bigl(1+\|x\|_{\mathcal X}^q\bigr)\bigl(1+\|y\|_{\mathcal X}^q\bigr)
    \qquad\text{for all }x,y\in\mathcal X.
\end{equation}
Then, the projected MMD objective converges to the exact MMD:
\begin{equation}\label{eq:proj.conv.as}
  \operatorname{MMD}_\kappa^2(P_{n,M},Q_{\theta,M}) \xrightarrow{M\to\infty} \operatorname{MMD}_\kappa^2(P_n,Q_\theta) \qquad\text{a.s.}
\end{equation}
\end{prop}
Together with the known empirical stability of MMD that controls the error from replacing \(P\) by \(P_n\) (see \cite[Proposition~A.1]{tolstikhinMinimaxEstimationKernel2017}), Proposition \ref{prop:projection_consistency} shows that the finite-dimensional surrogate \eqref{eq:projected_mmd} first approximates the Hilbert-space empirical criterion as $M\to\infty$, and then the population criterion as $n\to\infty$. This justifies \eqref{eq:projected_mmd} as a computationally tractable proxy for the infinite-dimensional MMD objective (see \Cref{fig:numerics_ablation} for an empirical illustration).

\paragraph{Posterior probabilities.}
The mixture admits the standard generative interpretation: draw a latent assignment \(Z \in \{1,\dots,K\}\) with prior \(\mathbb{P}(Z=k)=\pi_k\), and then sample \(X\mid\{Z=k\}\sim\nu_k\). The \emph{posterior component probability}, or responsibility, is defined by the Radon--Nikodym derivative
\begin{equation}\label{eq:responsibility}
	\gamma_k(X) \coloneqq \mathbb{P}(Z=k\mid X) = \frac{\diff(\pi_k\nu_k)}{\diff Q_\theta}(X), \qquad Q_\theta\text{-a.s.},
\end{equation}
which is well-defined because $\pi_k\nu_k \ll Q_\theta$ whenever $\pi_k>0$.
In the finite-dimensional setting, whenever all projected covariances $\mathbf{K}_{s,M}$ are nonsingular, each $\nu_{s,M}$ admits a Lebesgue density $p_{s,M}$ on $\mathbb{R}^M$ and \eqref{eq:responsibility} reduces to the familiar Bayes formula
\begin{equation}\label{eq:responsibility_density}
	\gamma_k^{(M)}(x) = \frac{\pi_k\,p_{k,M}(x)}{\sum_{s=1}^K \pi_s\,p_{s,M}(x)}.
\end{equation}
As $M$ grows, $\sigma(\Pi_M X)$ accumulates the information of $\sigma(X)$, yielding the following convergence.

\begin{prop}\label{prop:projection_responsibilities}
Let $(Z,X)$ be as above and let $\gamma_k^{(M)}\coloneqq \frac{\diff(\pi_k\nu_{k,M})}{\diff Q_{\theta,M}}$. Then the projected responsibility satisfies the martingale property
\begin{equation}\label{eq:martingale_property}
    \gamma_k^{(M)}(\Pi_M X) = \mathbb{E}[\mathbf{1}_{\{Z=k\}}\mid \sigma(\Pi_M X)] \quad \text{a.s.},
\end{equation}
and $\gamma_k^{(M)}(\Pi_M X) \to \gamma_k(X)$ as $M\to\infty$ both almost surely and in $L^1$.
\end{prop}

The proofs of Propositions \ref{prop:projection_consistency} and \ref{prop:projection_responsibilities} are deferred to \Cref{sec:proof_projection_consistency} and \Cref{sec:proof_projection_responsibilities}, respectively.


\section{Approximation theory}\label{sec:theory}

In this section, we establish theoretical guarantees that justify the use of finite Gaussian mixtures in Hilbert spaces as a flexible approximation class.

\begin{prop}[Wasserstein and MMD density]
\label{prop:approximation}
Let $\mathcal X\neq\{0\}$ be a separable Hilbert space.

\begin{enumerate}[label=(\roman*), leftmargin=*, topsep=0pt, itemsep=0pt]
\item For every $p\ge 1$, finite mixtures of Gaussian measures on $\mathcal X$ of the form \eqref{eq:mixture_model} are dense in $\mathcal P_p(\mathcal X)$ with respect to the Wasserstein distance $W_p$.

\item Let $\kappa$ be a continuous positive-definite kernel satisfying $\sup_{x\in\mathcal X}\kappa(x,x)\le 1$. Then, for all $P\in\mathcal P(\mathcal X)$ and all $K\ge 1$, there exists $Q_K$ of the form \eqref{eq:mixture_model} such that
\[
\operatorname{MMD}_\kappa(P,Q_K)\le \frac{2}{\sqrt K}.
\]
\end{enumerate}
\end{prop}

Part~(i) shows that the model class is expressive under Wasserstein metrics, which are widely used in mathematical statistics to study the convergence of estimators, resampling schemes such as the bootstrap, and related distributional properties \cite{panaretosStatisticalAspectsWasserstein2019}. This result also implies MMD density for characteristic bounded continuous kernels. Part~(ii), however, applies to arbitrary probability measures in $\mathcal{P}(\mathcal{X})$ and provides an explicit dimension-free approximation rate for the MMD minimized by our algorithm. The proof of Proposition \ref{prop:approximation} is deferred to \Cref{sec:proof_density}.


\section{Numerical experiments}\label{sec:experiments}

The first goal of this section is to assess the empirical clustering performance of our method against established approaches in both finite- and infinite-dimensional regimes. 
We then focus on a digital health case study to illustrate the versatility of our framework for analyzing complex clinical biomarkers. In particular, we summarize glucose time series using different representations, including random functions and graphs, which can be viewed as dynamic functional objects. Additional numerical comparisons are reported in \Cref{app:synthetic_experiments}.

\subsection{Clustering benchmark}\label{sec:benchmarks}

We compare the mixture model described in \Cref{sec:numerics} against several baselines:
(i) partitioning and centroid-based methods, including K-Medoids and K-Means
\cite{schubertFastEagerMedoids2021,sculleyWebscaleKmeansClustering2010};
(ii) agglomerative hierarchical methods, including average-linkage, Ward, and BIRCH
\cite{murtaghAlgorithmsHierarchicalClustering2012,wardHierarchicalGroupingOptimize1963,zhangBIRCHEfficientData1996};
(iii) density-based methods, including DBSCAN and OPTICS
\cite{esterDensitybasedAlgorithmDiscovering1996,ankerstOPTICSOrderingPoints1999};
(iv) model-based methods, including Gaussian mixtures
\cite{fraleyModelBasedClusteringDiscriminant2002,luSequentialMonteCarlo2026}; and
(v) graph-based and dimension-reduction methods, such as spectral clustering
\cite{ngSpectralClusteringAnalysis2001}.

\textbf{Datasets.} We evaluate our approach using representative synthetic and real-world datasets.
\begin{itemize}[leftmargin=*, topsep=0pt, parsep=0pt]
    \item \emph{Synthetic environments:} (i)~\(\R^d\): five standard
    \texttt{sklearn} clustering datasets with \(n=500\), \(d=2\), and
    \(K\in\{2,3\}\)~\cite{scikit-learn}; and (ii)~\(\mathrm{SO}(3)\):
    ten synthetic datasets of rotations under an Euler-angle parameterization,
    projected onto Wigner-D matrix coefficients, with \(n=200\) and \(K=3\).

    \item \emph{Real-world datasets:} (iii)~\(L^2\): {six} standard functional
    data benchmarks~\cite{ramsayFunctionalDataAnalysis2005,breimanClassificationRegressionTrees2017,ferratyNonparametricFunctionalData2006,predaCategoricalFunctionalData2021,olszewskiGeneralizedFeatureExtraction2001};
    (iv)~CGM: continuous glucose monitoring curves from \(n=98\) individuals with
    \(K=2\), detailed in \Cref{sec:glucodensity}; and (v)~Molecular: MUTAG graph
    signals with \(n=188\) and \(K=2\)~\cite{debnathStructureactivityRelationshipMutagenic1991}.
\end{itemize}
All implementation details and evaluation protocols are reported in \Cref{app:experimental_details}. Performance is evaluated using the Adjusted Rand Index (ARI), which compares unsupervised cluster assignments against ground-truth labels. We fix~$K$ to the known number of classes, as our goal here is to compare clustering~accuracy.

\textbf{Results.}
Although designed for general random objects in Hilbert spaces, our method is competitive on finite-dimensional \(\R^d\) benchmarks and achieves the best performance on the CGM and \(\mathrm{SO}(3)\) datasets, which fall outside the scope of most existing clustering methods. It also performs well in the infinite-dimensional \(L^2\) setting. Overall, the proposed model combines strong empirical accuracy and computational efficiency with a unified probabilistic interpretation. This flexibility enables time-dependent interpretations such as the digital health case study below.

\begin{table*}[!th]
\centering
\caption{
Benchmark of clustering algorithms comparing the operating spaces; the computational complexity of training, inference, and memory as functions of the number of samples \(n\) and clusters \(K\); and clustering performance (ARI, mean$_{\text{std}}$) across several datasets grouped by geometric space.
}
\label{tab:unified_benchmark}
\setlength{\tabcolsep}{4pt}
\resizebox{\textwidth}{!}{
\scriptsize
\begin{tabular}{l l c c c c c c c c}
\toprule
Method & Space & Train & Infer & Memory & $\mathbb{R}^d$ & $L^2$ & CGM & $\mathrm{SO}(3)$ & Molecular \\
\midrule
MMD GMM (Gaussian) & Hilbert & $O(nK)$ & $O(nK)$ & $O(nK)$ & $0.692_{0.379}$ & $0.400_{0.202}$ & $\mathbf{0.053}_{0.040}$ & $\mathbf{0.867}_{0.140}$ & $0.104_{0.032}$ \\
MMD GMM (Polynomial) & Hilbert & $O(nK)$ & $O(nK)$ & $O(nK)$ & $0.615_{0.334}$ & $0.373_{0.132}$ & $0.051_{0.022}$ & $0.675_{0.217}$ & $\mathbf{0.151}_{0.009}$ \\
Projected GMM (Appendix \ref{sec:likelihood}) & Hilbert & $O(nK)$ & $O(nK)$ & $O(nK)$ & --- & $0.345_{0.268}$ & $0.034_{0.000}$ & $0.828_{0.172}$ & $0.104_{0.059}$ \\
\midrule
K-Medoids~\cite{kaufmanFindingGroupsData1990,schubertFastEagerMedoids2021} & Metric & $O(n^2)$ & $O(nK)$ & $O(n^2)$ & $0.569_{0.328}$ & $\mathbf{0.483}_{0.192}$ & $0.020_{0.000}$ & $0.828_{0.109}$ & $-0.005_{0.000}$ \\
Hierarchical (Avg)~\cite{ferreiraComparisonHierarchicalMethods2009, lanceGeneralTheoryClassificatory1967} & Metric & $O(n^2)$ & --- & $O(n^2)$ & $0.555_{0.320}$ & $0.211_{0.275}$ & $0.002_{0.000}$ & $0.490_{0.315}$ & $0.081_{0.000}$ \\
DBSCAN~\cite{esterDensitybasedAlgorithmDiscovering1996} & Metric & $O(n^2)$ & --- & $O(n^2)$ & $0.844_{0.164}$ & $0.000_{0.000}$ & $0.000_{0.000}$ & $0.627_{0.273}$ & $0.074_{0.000}$ \\
HDBSCAN~\cite{campelloHierarchicalDensityEstimates2015} & Metric & $O(n^2)$ & --- & $O(n^2)$ & $0.870_{0.171}$ & $0.143_{0.217}$ & $0.000_{0.000}$ & $0.691_{0.227}$ & $0.117_{0.000}$ \\
K-Center~\cite{gonzalezClusteringMinimizeMaximum1985} & Metric & $O(nK)$ & $O(nK)$ & $O(n)$ & $0.392_{0.258}$ & $0.162_{0.109}$ & $0.002_{0.000}$ & $0.188_{0.169}$ & $0.074_{0.000}$ \\
Kernel k-Groups~\cite{francaKernelKGroupsHartigans2017} & Metric & $O(Kn^2)$ & --- & $O(n^2)$ & $0.567_{0.328}$ & $0.405_{0.202}$ & $0.019_{0.000}$ & $0.751_{0.189}$ & $0.125_{0.000}$ \\
\midrule
GPmix~\cite{akewejeGPmixModelbasedClustering2024} & $L^2$ & $O(nK+n^3)$ & --- & $O(n^2)$ & --- & $0.364_{0.237}$ & --- & --- & --- \\
FDA K-Means~\cite{luzlopezgarciaKmeansAlgorithmsFunctional2015} & $L^2$ & $O(nK)$ & $O(nK)$ & $O(nK)$ & --- & $0.451_{0.222}$ & $0.023_{0.000}$ & --- & --- \\
FDA Fuzzy C-Means~\cite{bezdekPatternRecognitionFuzzy1981} & $L^2$ & $O(nK)$ & $O(nK)$ & $O(nK)$ & --- & $0.390_{0.154}$ & $0.020_{0.000}$ & --- & --- \\
Funclust~\cite{jacquesFunclustCurvesClustering2013} & $L^2$ & $O(nK)$ & $O(nK)$ & $O(nK)$ & --- & $0.429_{0.157}$ & $0.015_{0.000}$ & --- & --- \\
FunHDDC~\cite{bouveyronModelbasedClusteringTime2011} & $L^2$ & $O(nK)$ & $O(nK)$ & $O(nK)$ & --- & $0.420_{0.138}$ & $0.013_{0.000}$ & --- & --- \\
fclust~\cite{jamesClusteringSparselySampled2003} & $L^2$ & $O(nK)$ & $O(nK)$ & $O(nK)$ & --- & $0.421_{0.240}$ & $0.030_{0.000}$ & --- & --- \\
K-Centres~\cite{chiouFunctionalClusteringIdentifying2007} & $L^2$ & $O(nK)$ & $O(nK)$ & $O(nK)$ & --- & $0.401_{0.178}$ & $0.001_{0.000}$ & --- & --- \\
Curvclust~\cite{giacofciWaveletBasedClusteringMixedEffects2013} & $L^2$ & $O(nK)$ & $O(nK)$ & $O(nK)$ & --- & $0.432_{0.191}$ & $0.012_{0.000}$ & --- & --- \\
\midrule
MiniBatch KMeans~\cite{sculleyWebscaleKmeansClustering2010} & $\mathbb{R}^d$ & $O(nK)$ & $O(nK)$ & $O(n)$ & $0.554_{0.327}$ & --- & --- & --- & --- \\
Spectral~\cite{ngSpectralClusteringAnalysis2001} & $\mathbb{R}^d$ & $O(n^3)$ & --- & $O(n^2)$ & $\mathbf{0.949}_{0.074}$ & --- & --- & --- & --- \\
Ward~\cite{wardHierarchicalGroupingOptimize1963} & $\mathbb{R}^d$ & $O(n^2)$ & --- & $O(n^2)$ & $0.613_{0.330}$ & --- & --- & --- & --- \\
BIRCH~\cite{zhangBIRCHEfficientData1996} & $\mathbb{R}^d$ & $O(n)$ & $O(nK)$ & $O(n)$ & $0.525_{0.303}$ & --- & --- & --- & --- \\
Gaussian Mixture (EM)~\cite{fraleyModelBasedClusteringDiscriminant2002} & $\mathbb{R}^d$ & $O(nK)$ & $O(nK)$ & $O(nK)$ & $0.683_{0.387}$ & --- & --- & --- & --- \\
Affinity Propagation~\cite{freyClusteringPassingMessages2007} & $\mathbb{R}^d$ & $O(n^2)$ & $O(nK)$ & $O(n^2)$ & $0.489_{0.308}$ & --- & --- & --- & --- \\
MeanShift~\cite{yizongchengMeanShiftMode1995} & $\mathbb{R}^d$ & $O(n^2)$ & $O(nK)$ & $O(n)$ & $0.564_{0.328}$ & --- & --- & --- & --- \\
OPTICS~\cite{ankerstOPTICSOrderingPoints1999} & $\mathbb{R}^d$ & $O(n^2)$ & --- & $O(n)$ & $0.809_{0.219}$ & --- & --- & --- & --- \\
\bottomrule
\end{tabular}}
\end{table*}

\subsection{Case study: different representations of digital health time series}
\label{sec:glucodensity}

We now illustrate the proposed model in a digital health case study based on longitudinal continuous glucose monitoring (CGM) data from a pediatric clinical trial \cite{wadwa2023trial}. A central question in longitudinal studies is whether and how the distribution of individuals changes over time, for example in response to treatment \cite{krouwerReviewStandardsStatistics2010,clarkeStatisticalToolsAnalyze2009}. In the original study \cite{wadwa2023trial}, the authors showed that a new insulin pump system improves several summary glycemic metrics after the intervention. Here, we examine these outcomes in a continuous-time formulation using dynamic functional representations of high-resolution glucose trajectories. Starting from multi-day glucose time series, we construct three daily representations: (i) functional observations in {$H^1$}; (ii) correlation matrices describing the dependence between glucose values at different hours of the day; and (iii) {signals on a similarity graph between individuals.} Similar representations have been considered as biomarkers in clinical fMRI studies \cite{supekarNetworkAnalysisIntrinsic2008,wuParkinsonsDiseaseRelated2015}.

To analyze temporal changes across these three CGM representations, we extend \eqref{eq:mixture_model} by allowing the mixture weights to vary over time while keeping the component means and covariance operators time-invariant. Our formulation is related to that of \cite{alvarez-lopezContinuousTemporalLearning2025}; however, whereas that work models time-evolving distributions of finite-dimensional random vectors, our framework applies to general separable Hilbert spaces.


\paragraph{Background and data.}
\label{sec:glucodensity_case_study}
Continuous glucose monitors record interstitial glucose concentration every 5 minutes, producing 288 measurements per day.
The data come from the {Pediatric Artificial Pancreas} (PEDAP) randomized controlled trial~\cite{wadwa2023trial}{, which evaluated} a hybrid closed-loop insulin delivery system {in young children} with type 1 diabetes; the CGM records span more than six months. For each individual, a sliding window of $W$ days with a stride of one day produces one observation per window. Within each window, the 288 time slots are averaged across valid days; slots with fewer than $\lceil 0.5\,W\rceil$ valid readings are linearly interpolated from neighboring slots. 
{This yields one intraday glucose curve $f\colon [0,24]\to\mathbb{R}$ per individual and window.} After preprocessing, the dataset comprises 98 individuals, with 34 in the control group and 64 in the treatment group. 

\paragraph{Time-varying mixture weights.}
\label{sec:time_varying}

Let $X_1,\dots,X_n$ be i.i.d.\ copies of an $\mathcal{X}$-valued measurable stochastic process $X(t)$. For each $t$, denote by $P(t)$ the law of $X(t)$, and define the empirical distribution by $P_n(t) \coloneqq \frac{1}{n}\sum_{i=1}^n \delta_{X_i(t)}$.
We fit the model
\begin{equation}\label{eq:model.timevarying}
Q_\theta(t) \coloneqq \sum_{k=1}^K \pi_k(t)\,\mathcal{N}(m_k, \mathcal{K}_k),
\qquad
\pi(t)\coloneqq \bigl(\pi_1(t),\dots,\pi_K(t)\bigr)\in \Delta^{K-1}.
\end{equation}
Conceptually, the shared components $(m_k, \mathcal{K}_k)$ capture stationary data archetypes, such as distinct physiological regimes or individual phenotypes, while the time-varying weight vector $\pi(t)$ describes population-level shifts between these archetypes over time.

Following the kernel $k$-groups framework \cite{francaKernelKGroupsHartigans2017}, we select $K$ by evaluating the within-cluster MMD across candidate values {(see \Cref{app:time-varying})}. Model fitting minimizes the integrated MMD$^2$,
\begin{equation}
\label{eq:time_varying_loss}
\min_{\theta(t)=\{\pi_k(t), m_k, \mathcal{K}_k\}_k}
    \frac{1}{T}\int_0^T
    \operatorname{MMD}_\kappa^2\bigl(P_n(t), Q_\theta(t)\bigr)\diff t.
\end{equation}
We parameterize $\pi(t) = \mathrm{softmax}\bigl(\mathbf{z}(t)\bigr)$, where the logit path $\mathbf{z}\colon[0,T]\to\mathbb{R}^K$ is the solution to a neural ODE \cite{chenNeuralOrdinaryDifferential2019}, initialized at $\mathbf{z}(0)=\log \pi(0)$.
This parameterization introduces a natural bias toward smooth temporal dynamics.
Further details on density and consistency results, the projected $\mathbb{R}^M$ approximation of the integrated objective, and implementation specifics are deferred to \Cref{app:time-varying}.

\begin{figure}[t]
    \centering
    \includegraphics[width=0.9\textwidth]{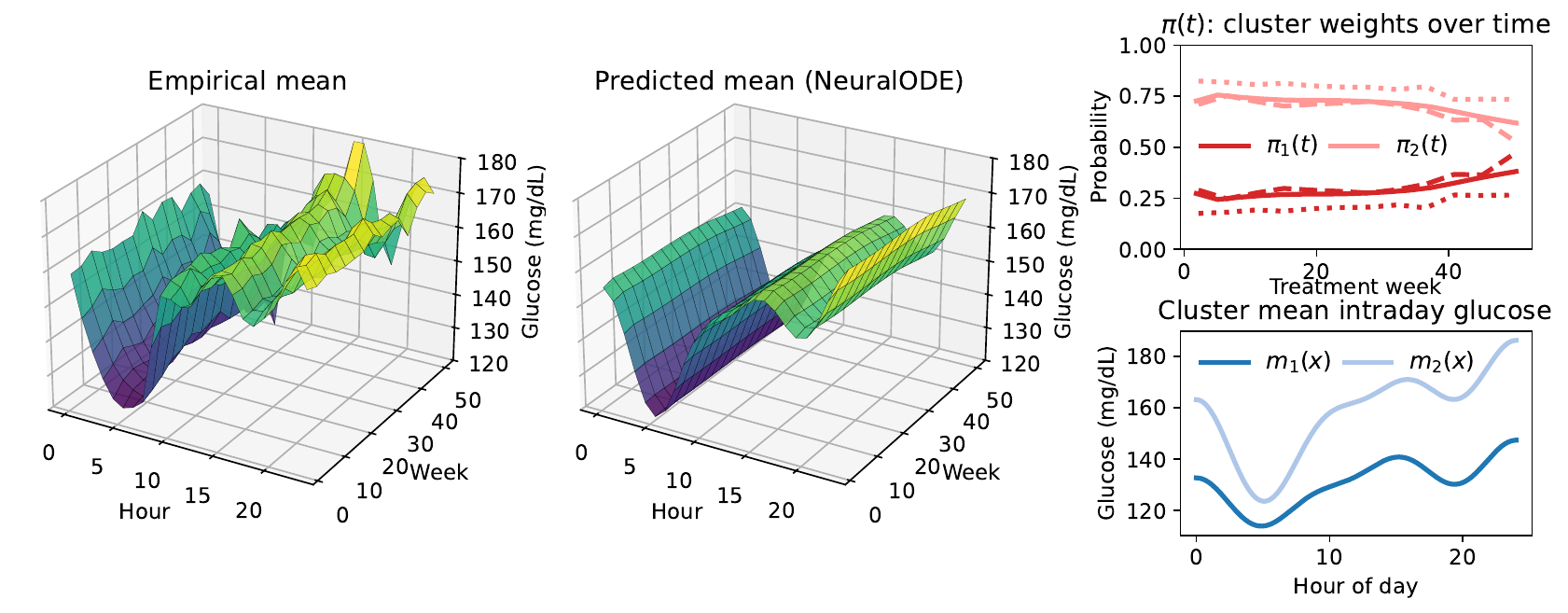}
    \caption{Temporal glucose mixture in $H^1$. \textbf{Left:} Empirical (\textbf{top}) and model-predicted (\textbf{bottom}) mean glucose surfaces. \textbf{Top right:} Global cluster weights $\pi_k(t)$ (solid lines) with group-level posteriors $\bar{\gamma}_{\mathrm{ctrl}}(t)$ for control (dotted) and $\bar{\gamma}_{\mathrm{treat}}(t)$ for treatment (dashed). \textbf{Bottom right:} Learned means.}
    \label{fig:glucodensity_summary}
\end{figure}

\paragraph{Group-level comparison.}
To compare the control and treatment arms, we compute each individual's posterior cluster membership $\gamma_k(X_i(t),t)$ at each time point and aggregate these memberships into group-level trajectories. Specifically, for each group $g\in\{\mathrm{ctrl},\mathrm{treat}\}$ (control and treatment, respectively), we compute
\[
\bar{\gamma}_g(t)
=
(\bar{\gamma}_{g,1}(t),\dots,\bar{\gamma}_{g,K}(t)),
\qquad
\bar{\gamma}_{g,k}(t)
\coloneqq
\frac{1}{|G_g(t)|}
\sum_{i\in G_g(t)}
\gamma_k(X_i(t),t),
\]
where $G_g(t)$ denotes the set of individuals in group $g$ observed at time $t$. We then quantify the divergence between the study arms using the total variation distance
$
\mathrm{TV}(t)
=
\frac{1}{2}
\left\|
\bar{\gamma}_{\mathrm{treat}}(t)
-
\bar{\gamma}_{\mathrm{ctrl}}(t)
\right\|_1.$ This metric ranges from 0 to 1, where 0 reflects identical group-level cluster usage and 1 indicates complete separation, providing an interpretable measure of treatment-induced distributional shifts over time.

\subsubsection{Results}

\textbf{Intraday curves.} We view each curve as an element of the Sobolev space $H^1(0,24;\R)$, project it onto a truncated cosine basis, and fit the model via~\eqref{eq:time_varying_loss} with $K=2$ latent clusters. The two learned means contrast a flatter, well-regulated profile ($m_1$) with an elevated, variable profile ($m_2$); the empirical and model-predicted mean surfaces agree closely (\Cref{fig:glucodensity_summary}, left). While the control group remains approximately time-invariant, the treatment group gradually shifts toward~$m_1$, with the between-group $\mathrm{TV}$ distance increasing from ${\approx}\,0.11$ at baseline to ${\approx}\,0.19$ at the end of the study, reflecting a progressive improvement in glycemic regulation.

\textbf{Correlation matrices.} Instead of clustering intraday curves, we cluster the running $24\times 24$ \emph{inter-hour correlation matrices} to capture how the \emph{dependence structure} between hours of the day evolves over the study. Each sliding window ($W=14$~days) is summarized into 24 hourly averages per day, from which the sample correlation matrix is estimated via linear shrinkage toward the identity ($\alpha=0.1$) and projected onto the positive semidefinite cone. {Since correlation matrices do not form a Hilbert space, we embed them in the ambient space $\mathrm{Sym}(24)$ of symmetric $24\times 24$ matrices, equipped with the Frobenius inner product, and fit the temporal mixture model with $K=3$ clusters.} The learned clusters in \Cref{fig:corr_ode_training} identify three inter-hour dependence patterns: moderate local correlation ($m_1$), broadly strong correlation ($m_2$), and sharper early-hour dependence ($m_3$). Over time, membership shifts from $m_3$ toward the more regulated $m_1$ regime, consistent with improved glycemic regulation.

\begin{figure}[t]
    \centering
    \includegraphics[width=0.8\textwidth]{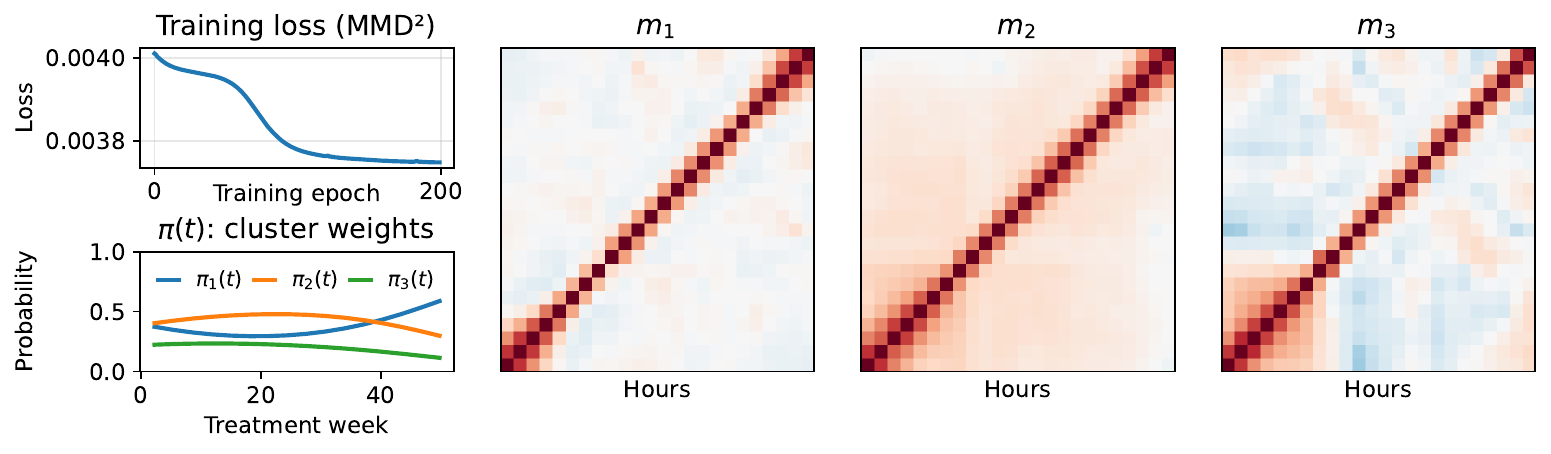}
    \caption{Correlation-based temporal mixture in $\mathrm{Sym}(24)$. \textbf{Left:} MMD$^2$ loss versus training epoch (\textbf{top}) and cluster weights $\pi_k(t)$ versus treatment week (\textbf{bottom}). \textbf{Right:} Learned mean matrices in $\mathrm{Sym}(24)$.}
    \label{fig:corr_ode_training}
\end{figure}
\textbf{Individual similarity graph signals.} We now move from hour- to individual-level structure. For each sliding window ($W=21$ days), we compute pairwise $H^1$ similarities between individuals to obtain a time-varying similarity matrix $S(t)\in\R^{N\times N}$. A fixed individual graph is then built by thresholding the time-averaged similarity matrix at the 85th percentile, and the graph Laplacian eigenbasis (using $R=4$ eigenvectors) is used to project each $S(t)$ onto graph signals. We then cluster these projected matrices over time with a $K=3$ mixture model. As the trial progresses (\Cref{fig:individual_graph}), the network structure evolves from diffuse inter-group connections toward denser within-group similarity in the treatment arm, consistent with individuals progressively achieving improved glycemic control.

\begin{figure}[tb]
    \centering
    \includegraphics[width=0.8\textwidth]{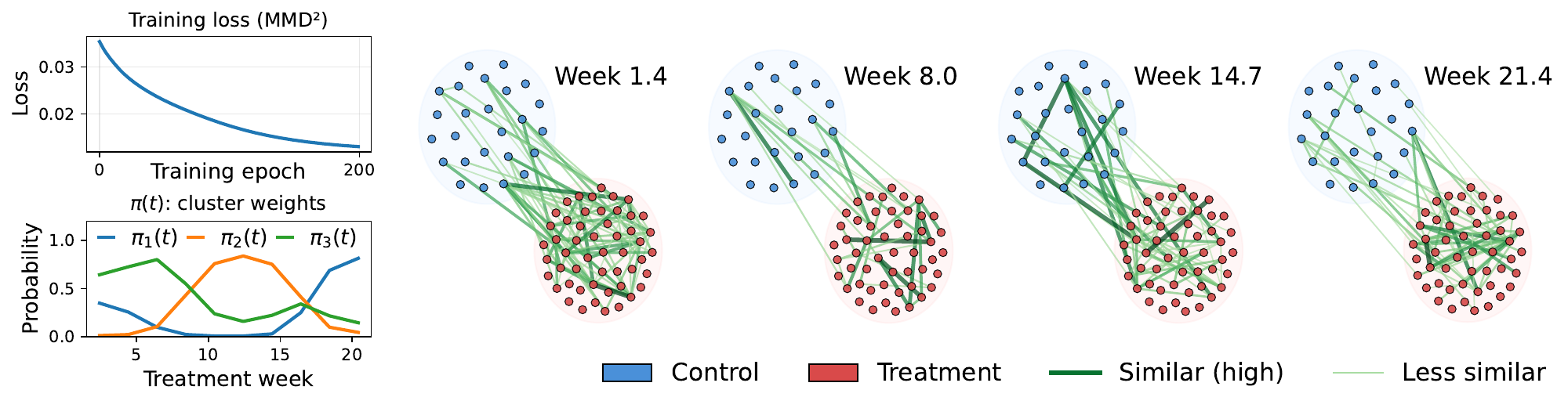}
    \caption{Temporal mixture on the individual similarity graph. \textbf{Left:} MMD$^2$ loss versus training epoch (\textbf{top}) and cluster weights $\pi_k(t)$ versus treatment week (\textbf{bottom}). \textbf{Right:} Network snapshots at four time points; node color distinguishes control (\textcolor{blue}{blue}) from treatment (\textcolor{red}{red}), and edge width encodes similarity.}
    \label{fig:individual_graph}
\end{figure}

\section{Final remarks}\label{sec:limitations}

The main contribution of this paper is a Gaussian mixture model for random elements in separable Hilbert spaces, estimated via an MMD-based optimization procedure. A key methodological and practical advantage of the proposed approach is that it bypasses the need for densities or dominating measures, which are generally intractable in infinite-dimensional settings. In addition, the framework flexibly accommodates diverse data structures, is computationally efficient, and demonstrates competitive empirical performance.

Future work will focus on developing specialized temporal algorithms for specific Hilbert spaces, extending the framework to general metric spaces, and quantifying uncertainty in both parameter estimates and model predictions. {On the theoretical side, building upon the methodological foundation established here, deriving convergence guarantees for the proposed algorithms, proving statistical consistency of the fitted parameters, and characterizing the identifiability of the mixture decomposition represent important avenues for future investigation.}

\begin{ack}
A. Álvarez-López (AÁ) was supported by an FPU fellowship (FPU21/05673) from the Spanish Ministry of Universities. D.~López-Montero (DL) was supported by the Alexander von Humboldt Professorship programme of the Alexander von Humboldt Foundation, funded by the German Federal Ministry of Education and Research (BMBF), at the Chair for Dynamics, Control, Machine Learning, and Numerics, Friedrich-Alexander-Universität Erlangen-Nürnberg. AÁ and DL were also supported by project PID2023-146872OB-I00 of MICIU, Spain.

\end{ack}

\printbibliography

@misc{alvarez-lopezContinuousTemporalLearning2025,
      title={Continuous-Time Learning of Probability Distributions: A Case Study in a Digital Trial of Young Children with Type 1 Diabetes}, 
      author = {{\'A}lvarez-L{\'o}pez, Antonio and Matabuena, Marcos},
      year={2026},
      eprint={2603.24427},
      archivePrefix={arXiv},
  url = {https://arxiv.org/abs/2603.24427},
  howpublished = {arXiv:2603.24427},
}

@article{dubey2022modeling,
  title = {Modeling Time-Varying Random Objects and Dynamic Networks},
  author = {Dubey, Paromita and M{\"u}ller, Hans-Georg},
  year = {2022},
  journal = {Journal of the American Statistical Association},
  volume = {117},
  number = {540},
  pages = {2252--2267},
  doi = {10.1080/01621459.2021.1917416},
}

@article{dubey2020functional,
  title = {Functional Models for Time-Varying Random Objects},
  author = {Dubey, Paromita and M{\"u}ller, Hans-Georg},
  year = {2020},
  journal = {Journal of the Royal Statistical Society Series B: Statistical Methodology},
  volume = {82},
  number = {2},
  pages = {275--327},
  doi = {10.1111/rssb.12337},
}

@inproceedings{ankerstOPTICSOrderingPoints1999,
  title = {{{OPTICS}}: Ordering Points to Identify the Clustering Structure},
  booktitle = {Proceedings of the 1999 {{ACM SIGMOD}} International Conference on {{Management}} of Data},
  author = {Ankerst, Mihael and Breunig, Markus M. and Kriegel, Hans-Peter and Sander, J{\"o}rg},
  year = {1999},
  pages = {49--60},
  publisher = {ACM},
  doi = {10.1145/304182.304187},
}

@article{antonettiAnalysisDistributionalReinforcement2025,
  title = {An Analysis of Distributional Reinforcement Learning with {{Gaussian}} Mixtures},
  author = {Antonetti, Mathis and Don{\^a}ncio, Henrique and Forbes, Florence},
  year = {2026},
  journal = {Transactions on Machine Learning Research},
  url = {https://openreview.net/forum?id=b4VgI1RTv8}
}

@article{bayKarhunenLoeveDecomposition2019,
  title = {Karhunen--{{Lo{\`e}ve}} Decomposition of {{Gaussian}} Measures on {{Banach}} Spaces},
  author = {Bay, Xavier and Croix, Jean-Charles},
  year = {2019},
  journal = {Probability and Mathematical Statistics},
  volume = {39},
  number = {2},
  pages = {279--297},
  doi = {10.19195/0208-4147.39.2.3},
}

@book{berlinetReproducingKernelHilbert2004,
  title = {Reproducing {{Kernel Hilbert Spaces}} in {{Probability}} and {{Statistics}}},
  author = {Berlinet, Alain and Thomas-Agnan, Christine},
  year = {2004},
  publisher = {Springer US},
  doi = {10.1007/978-1-4419-9096-9}
}

@book{bezdekPatternRecognitionFuzzy1981,
  title = {Pattern {{Recognition}} with {{Fuzzy Objective Function Algorithms}}},
  author = {Bezdek, James C.},
  year = {1981},
  publisher = {Springer US},
  doi = {10.1007/978-1-4757-0450-1}
}

@book{bishopPatternRecognitionMachine2006,
  title = {Pattern Recognition and Machine Learning},
  author = {Bishop, Christopher M.},
  year = {2006},
  series = {Information Science and Statistics},
  publisher = {Springer}
}

@book{bogachevGaussianMeasures1998,
  title = {Gaussian {{Measures}}},
  author = {Bogachev, Vladimir I.},
  year = {1998},
  series = {Mathematical {{Surveys}} and {{Monographs}}},
  volume = {62},
  publisher = {American Mathematical Society},
  doi = {10.1090/surv/062}
}

@article{bouveyronModelbasedClusteringTime2011,
  title = {Model-Based Clustering of Time Series in Group-Specific Functional Subspaces},
  author = {Bouveyron, Charles and Jacques, Julien},
  year = {2011},
  journal = {Advances in Data Analysis and Classification},
  volume = {5},
  number = {4},
  pages = {281--300},
  doi = {10.1007/s11634-011-0095-6}
}

@book{breimanClassificationRegressionTrees2017,
  title = {Classification and Regression Trees},
  author = {Breiman, Leo and Friedman, Jerome H. and Olshen, Richard A. and Stone, Charles J.},
  year = {2017},
  publisher = {Routledge},
  doi = {10.1201/9781315139470},
}

@inproceedings{briolDictionaryClosedFormKernel2025,
  title = {A {{Dictionary}} of {{Closed-Form Kernel Mean Embeddings}}},
  author = {Briol, Fran{\c{c}}ois-Xavier and Karvonen, Toni and Gessner, Alexandra and Mahsereci, Maren},
  year = {2025},
  booktitle = {Proceedings of the First International Conference on Probabilistic Numerics},
  volume = {271},
  series = {Proceedings of Machine Learning Research},
  pages = {84--94},
  publisher = {PMLR},
  url = {https://proceedings.mlr.press/v271/briol25a.html}
}

@misc{briolStatisticalInferenceGenerative2019,
  title = {Statistical Inference for Generative Models with Maximum Mean Discrepancy},
  author = {Briol, Fran{\c{c}}ois-Xavier and Barp, Alessandro and Duncan, Andrew B. and Girolami, Mark},
  year = {2019},
  eprint = {1906.05944},
  archivePrefix = {arXiv},
  url = {https://arxiv.org/abs/1906.05944},
  howpublished = {arXiv:1906.05944},
}

@article{campelloHierarchicalDensityEstimates2015,
  title = {Hierarchical {{Density Estimates}} for {{Data Clustering}}, {{Visualization}}, and {{Outlier Detection}}},
  author = {Campello, Ricardo J. G. B. and Moulavi, Davoud and Zimek, Arthur and Sander, J{\"o}rg},
  year = {2015},
  journal = {ACM Transactions on Knowledge Discovery from Data},
  volume = {10},
  number = {1},
  pages = {5:1--5:51},
  doi = {10.1145/2733381},
}

@inproceedings{chenNeuralOrdinaryDifferential2019,
  title = {Neural Ordinary Differential Equations},
  author = {Chen, Ricky T. Q. and Rubanova, Yulia and Bettencourt, Jesse and Duvenaud, David K.},
  year = {2018},
  booktitle = {Advances in Neural Information Processing Systems},
  volume = {31},
  publisher = {Curran Associates, Inc.}
}

@article{cherief-abdellatifFiniteSampleProperties2022,
  title = {Finite Sample Properties of Parametric {{MMD}} Estimation: {{Robustness}} to Misspecification and Dependence},
  author = {Ch{\'e}rief-Abdellatif, Badr-Eddine and Alquier, Pierre},
  year = {2022},
  journal = {Bernoulli},
  volume = {28},
  number = {1},
  pages = {181--213},
  doi = {10.3150/21-BEJ1338},
}

@article{chiouFunctionalClusteringIdentifying2007,
  title = {Functional {{Clustering}} and {{Identifying Substructures}} of {{Longitudinal Data}}},
  author = {Chiou, Jeng-Min and Li, Pai-Ling},
  year = {2007},
  journal = {Journal of the Royal Statistical Society Series B: Statistical Methodology},
  volume = {69},
  number = {4},
  pages = {679--699},
  doi = {10.1111/j.1467-9868.2007.00605.x},
}

@article{clarkeStatisticalToolsAnalyze2009,
  title = {Statistical Tools to Analyze Continuous Glucose Monitor Data},
  author = {Clarke, William and Kovatchev, Boris},
  year = {2009},
  journal = {Diabetes Technology \& Therapeutics},
  volume = {11},
  pages = {S45--S54},
  number = {Supplement 1},
  doi = {10.1089/dia.2008.0138}
}

@book{dapratoStochasticEquationsInfinite2014,
  title = {Stochastic {{Equations}} in {{Infinite Dimensions}}},
  author = {Da Prato, Giuseppe and Zabczyk, Jerzy},
  year = {2014},
  series = {Encyclopedia of {{Mathematics}} and Its {{Applications}}},
  publisher = {Cambridge University Press},
  edition = {2},
  volume = {152},
  doi = {10.1017/CBO9781107295513},
}

@article{debnathStructureactivityRelationshipMutagenic1991,
  title = {Structure-Activity Relationship of Mutagenic Aromatic and Heteroaromatic Nitro Compounds. {{Correlation}} with Molecular Orbital Energies and Hydrophobicity},
  author = {Debnath, Asim Kumar and Lopez De Compadre, Rosa L. and Debnath, Gargi and Shusterman, Alan J. and Hansch, Corwin},
  year = {1991},
  journal = {Journal of Medicinal Chemistry},
  volume = {34},
  number = {2},
  pages = {786--797},
  doi = {10.1021/jm00106a046},
}

@article{delaigleDefiningProbabilityDensity2010,
  title = {Defining Probability Density for a Distribution of Random Functions},
  author = {Delaigle, Aurore and Hall, Peter},
  year = {2010},
  journal = {The Annals of Statistics},
  volume = {38},
  number = {2},
  pages = {1171--1193},
  doi = {10.1214/09-AOS741}
}

@article{dempsterMaximumLikelihoodIncomplete1977,
  title = {Maximum Likelihood from Incomplete Data via the {EM} Algorithm},
  author = {Dempster, A. P. and Laird, N. M. and Rubin, D. B.},
  year = {1977},
  journal = {Journal of the Royal Statistical Society Series B: Statistical Methodology},
  volume = {39},
  number = {1},
  pages = {1--22},
  doi = {10.1111/j.2517-6161.1977.tb01600.x}
}

@inproceedings{esterDensitybasedAlgorithmDiscovering1996,
  title = {A Density-Based Algorithm for Discovering Clusters in Large Spatial Databases with Noise},
  booktitle = {Proceedings of the {{Second International Conference}} on {{Knowledge Discovery}} and {{Data Mining}}},
  author = {Ester, Martin and Kriegel, Hans-Peter and Sander, J{\"o}rg and Xu, Xiaowei},
  year = {1996},
  series = {{{KDD}}'96},
  pages = {226--231},
  publisher = {AAAI Press}
}

@book{ferratyNonparametricFunctionalData2006,
  title = {Nonparametric Functional Data Analysis: Theory and Practice},
  author = {Ferraty, Fr{\'e}d{\'e}ric and Vieu, Philippe},
  year = {2006},
  series = {Springer {{Series}} in {{Statistics}}},
  publisher = {Springer New York},
  doi = {10.1007/0-387-36620-2},
}

@article{ferreiraComparisonHierarchicalMethods2009,
  title = {A {{Comparison}} of {{Hierarchical Methods}} for {{Clustering Functional Data}}},
  author = {Ferreira, Laura and Hitchcock, David B.},
  year = {2009},
  journal = {Communications in Statistics - Simulation and Computation},
  volume = {38},
  number = {9},
  pages = {1925--1949},
  doi = {10.1080/03610910903168603},
}

@article{fraleyModelBasedClusteringDiscriminant2002,
  title = {Model-{{Based Clustering}}, {{Discriminant Analysis}}, and {{Density Estimation}}},
  author = {Fraley, Chris and Raftery, Adrian E},
  year = {2002},
  journal = {Journal of the American Statistical Association},
  volume = {97},
  number = {458},
  pages = {611--631},
  doi = {10.1198/016214502760047131},
}

@article{francaKernelKGroupsHartigans2017,
  title = {Kernel K-{{Groups}} via {{Hartigan}}'s {{Method}}},
  author = {Fran{\c{c}}a, Guilherme and Rizzo, Maria L. and Vogelstein, Joshua T.},
  year = {2021},
  journal = {IEEE Transactions on Pattern Analysis and Machine Intelligence},
  volume = {43},
  number = {12},
  pages = {4411--4425},
  doi = {10.1109/TPAMI.2020.2998120}
}

@article{freyClusteringPassingMessages2007,
  title = {Clustering by {{Passing Messages Between Data Points}}},
  author = {Frey, Brendan J. and Dueck, Delbert},
  year = {2007},
  journal = {Science},
  volume = {315},
  number = {5814},
  pages = {972--976},
  doi = {10.1126/science.1136800},
}

@article{giacofciWaveletBasedClusteringMixedEffects2013,
  title = {Wavelet-{{Based Clustering}} for {{Mixed}}-{{Effects Functional Models}} in {{High Dimension}}},
  author = {Giacofci, M. and Lambert-Lacroix, S. and Marot, G. and Picard, F.},
  year = {2013},
  journal = {Biometrics},
  volume = {69},
  number = {1},
  pages = {31--40},
  doi = {10.1111/j.1541-0420.2012.01828.x},
}

@article{gonzalezClusteringMinimizeMaximum1985,
  title = {Clustering to Minimize the Maximum Intercluster Distance},
  author = {Gonzalez, Teofilo F.},
  year = {1985},
  journal = {Theoretical Computer Science},
  volume = {38},
  pages = {293--306},
  doi = {10.1016/0304-3975(85)90224-5},
}

@article{grettonKernelTwoSampleTest2012a,
  title = {A {{Kernel Two-Sample Test}}},
  author = {Gretton, Arthur and Borgwardt, Karsten M. and Rasch, Malte J. and Sch{\"o}lkopf, Bernhard and Smola, Alexander},
  year = {2012},
  journal = {Journal of Machine Learning Research},
  volume = {13},
  number = {25},
  pages = {723--773}
}

@article{huangEstimatingMixtureGaussian2014,
  title = {Estimating Mixture of Gaussian Processes by Kernel Smoothing},
  author = {Huang, Mian and Li, Runze and Wang, Hansheng and Yao, Weixin},
  year = {2014},
  journal = {Journal of Business \& Economic Statistics},
  volume = {32},
  number = {2},
  pages = {259--270},
  doi = {10.1080/07350015.2013.868084},
}

@article{jacquesFunclustCurvesClustering2013,
  title = {Funclust: {{A}} Curves Clustering Method Using Functional Random Variables Density Approximation},
  author = {Jacques, Julien and Preda, Cristian},
  year = {2013},
  journal = {Neurocomputing},
  volume = {112},
  pages = {164--171},
  doi = {10.1016/j.neucom.2012.11.042},
}

@article{jamesClusteringSparselySampled2003,
  title = {Clustering for {{Sparsely Sampled Functional Data}}},
  author = {James, Gareth M and Sugar, Catherine A},
  year = {2003},
  journal = {Journal of the American Statistical Association},
  volume = {98},
  number = {462},
  pages = {397--408},
  doi = {10.1198/016214503000189},
}

@book{kaufmanFindingGroupsData1990,
  title = {Finding {{Groups}} in {{Data}}: {{An Introduction}} to {{Cluster Analysis}}},
  author = {Kaufman, Leonard and Rousseeuw, Peter J.},
  year = {1990},
  series = {Wiley {{Series}} in {{Probability}} and {{Statistics}}},
  publisher = {Wiley},
  doi = {10.1002/9780470316801},
}

@article{krouwerReviewStandardsStatistics2010,
  title = {A Review of Standards and Statistics Used to Describe Blood Glucose Monitor Performance},
  author = {Krouwer, Jan S and Cembrowski, George S},
  year = {2010},
  journal = {Journal of Diabetes Science and Technology},
  volume = {4},
  number = {1},
  pages = {75--83},
  doi = {10.1177/193229681000400110},
}

@article{lanceGeneralTheoryClassificatory1967,
  title = {A {{General Theory}} of {{Classificatory Sorting Strategies}}: 1. {{Hierarchical Systems}}},
  author = {Lance, G. N. and Williams, W. T.},
  year = {1967},
  journal = {The Computer Journal},
  volume = {9},
  number = {4},
  pages = {373--380},
  doi = {10.1093/comjnl/9.4.373},
}

@inproceedings{liGenerativeMomentMatching2015,
  title = {Generative {{Moment Matching Networks}}},
  booktitle = {Proceedings of the 32nd {{International Conference}} on {{Machine Learning}}},
  author = {Li, Yujia and Swersky, Kevin and Zemel, Rich},
  year = {2015},
  pages = {1718--1727},
  publisher = {PMLR},
  volume = {37},
  url = {https://proceedings.mlr.press/v37/li15.html}
}

@misc{luSequentialMonteCarlo2026,
  title = {Sequential {{Monte Carlo}} with {{Gaussian Mixture Approximation}} for {{Infinite-Dimensional Statistical Inverse Problems}}},
  author = {Lu, Haoyu and Jia, Junxiong and Meng, Deyu},
  year = {2026},
  eprint = {2503.16028},
  archivePrefix = {arXiv},
  url = {https://arxiv.org/abs/2503.16028},
  howpublished = {arXiv:2503.16028 (version 5, 2026)},
}

@article{luzlopezgarciaKmeansAlgorithmsFunctional2015,
  title = {K-Means Algorithms for Functional Data},
  author = {L{\'o}pez Garc{\'i}a, Mar{\'i}a Luz and Garc{\'i}a-R{\'o}denas, Ricardo and Gonz{\'a}lez G{\'o}mez, Antonia},
  year = {2015},
  journal = {Neurocomputing},
  volume = {151},
  pages = {231--245},
  doi = {10.1016/j.neucom.2014.09.048}
}

@article{matabuenaKernelBiclusteringAlgorithm2025,
  title = {Kernel Biclustering Algorithm in {{Hilbert}} Spaces},
  author = {Matabuena, Marcos and Vidal, Juan C. and Madrid Padilla, Oscar Hernan and Sejdinovic, Dino},
  year = {2026},
  journal = {Advances in Data Analysis and Classification},
  volume = {20},
  number = {2},
  pages = {479--520},
  doi = {10.1007/s11634-024-00619-0}
}

@article{mcdowellClusteringGeneExpression2018,
  title = {Clustering Gene Expression Time Series Data Using an Infinite Gaussian Process Mixture Model},
  author = {McDowell, Ian C. and Manandhar, Dinesh and Vockley, Christopher M. and Schmid, Amy K. and Reddy, Timothy E. and Engelhardt, Barbara E.},
  year = {2018},
  journal = {PLOS Computational Biology},
  volume = {14},
  number = {1},
  pages = {e1005896},
  doi = {10.1371/journal.pcbi.1005896},
}

@article{muandetKernelMeanEmbedding2017,
  title = {Kernel {{Mean Embedding}} of {{Distributions}}: {{A Review}} and {{Beyond}}},
  author = {Muandet, Krikamol and Fukumizu, Kenji and Sriperumbudur, Bharath and Sch{\"o}lkopf, Bernhard},
  year = {2017},
  journal = {Foundations and Trends in Machine Learning},
  volume = {10},
  number = {1--2},
  pages = {1--141},
  doi = {10.1561/2200000060},
}

@article{murtaghAlgorithmsHierarchicalClustering2012,
  title = {Algorithms for Hierarchical Clustering: An Overview},
  author = {Murtagh, Fionn and Contreras, Pedro},
  year = {2012},
  journal = {WIREs Data Mining and Knowledge Discovery},
  volume = {2},
  number = {1},
  pages = {86--97},
  doi = {10.1002/widm.53},
}

@inproceedings{ngSpectralClusteringAnalysis2001,
  title = {On Spectral Clustering: Analysis and an Algorithm},
  booktitle = {Advances in Neural Information Processing Systems},
  author = {Ng, Andrew Y. and Jordan, Michael I. and Weiss, Yair},
  year = {2001},
  pages = {849--856},
  publisher = {MIT Press},
  volume = {14}
}

@phdthesis{olszewskiGeneralizedFeatureExtraction2001,
  title = {Generalized Feature Extraction for Structural Pattern Recognition in Time-Series Data},
  author = {Olszewski, Robert Thomas},
  year = {2001},
  school = {Carnegie Mellon University}
}

@article{predaCategoricalFunctionalData2021,
  title = {Categorical Functional Data Analysis. The {cfda} {R} Package},
  author = {Preda, Cristian and Grimonprez, Quentin and Vandewalle, Vincent},
  year = {2021},
  journal = {Mathematics},
  volume = {9},
  number = {23},
  pages = {3074},
  doi = {10.3390/math9233074}
}

@article{ramos-carrenoScikitfdaPythonPackage2024,
  title = {{scikit-fda}: A {Python} Package for Functional Data Analysis},
  author = {Ramos-Carre{\~n}o, Carlos and Torrecilla, Jos{\'e} Luis and Carbajo-Berrocal, Miguel and Marcos, Pablo and Su{\'a}rez, Alberto},
  year = {2024},
  journal = {Journal of Statistical Software},
  volume = {109},
  number = {2},
  pages = {1--37},
  doi = {10.18637/jss.v109.i02},
}

@book{ramsayFunctionalDataAnalysis2005,
  title = {Functional {{Data Analysis}}},
  author = {Ramsay, J. O. and Silverman, B. W.},
  year = {2005},
  series = {Springer {{Series}} in {{Statistics}}},
  publisher = {Springer New York},
  edition = {2},
  doi = {10.1007/b98888}
}

@book{rasmussenGaussianProcessesMachine2008,
  title = {Gaussian Processes for Machine Learning},
  author = {Rasmussen, Carl Edward and Williams, Christopher K. I.},
  year = {2006},
  series = {Adaptive Computation and Machine Learning},
  publisher = {MIT Press}
}

@article{schubertFastEagerMedoids2021,
  title = {Fast and Eager {$k$}-Medoids Clustering: {$O(k)$} Runtime Improvement of the {PAM}, {CLARA}, and {CLARANS} Algorithms},
  author = {Schubert, Erich and Rousseeuw, Peter J.},
  year = {2021},
  journal = {Information Systems},
  volume = {101},
  pages = {101804},
  doi = {10.1016/j.is.2021.101804}
}

@article{scikit-learn,
  title = {{Scikit-learn}: Machine Learning in {Python}},
  author = {Pedregosa, F. and Varoquaux, G. and Gramfort, A. and Michel, V. and Thirion, B. and Grisel, O. and Blondel, M. and Prettenhofer, P. and Weiss, R. and Dubourg, V. and Vanderplas, J. and Passos, A. and Cournapeau, D. and Brucher, M. and Perrot, M. and Duchesnay, E.},
  year = {2011},
  journal = {Journal of Machine Learning Research},
  volume = {12},
  pages = {2825--2830},
  number = {85},
}

@inproceedings{sculleyWebscaleKmeansClustering2010,
  title = {Web-Scale k-Means Clustering},
  booktitle = {Proceedings of the 19th International Conference on {{World}} Wide Web},
  author = {Sculley, D.},
  year = {2010},
  pages = {1177--1178},
  publisher = {ACM},
  doi = {10.1145/1772690.1772862},
}

@article{sejdinovicEquivalenceDistancebasedRKHSbased2013,
  title = {Equivalence of Distance-Based and {{RKHS-based}} Statistics in Hypothesis Testing},
  author = {Sejdinovic, Dino and Sriperumbudur, Bharath and Gretton, Arthur and Fukumizu, Kenji},
  year = {2013},
  journal = {The Annals of Statistics},
  volume = {41},
  number = {5},
  pages = {2263--2291},
  doi = {10.1214/13-AOS1140}
}

@inproceedings{shahroudyNTURGB+DLarge2016,
  title = {{NTU RGB+D}: A Large Scale Dataset for {3D} Human Activity Analysis},
  author = {Shahroudy, Amir and Liu, Jun and Ng, Tian-Tsong and Wang, Gang},
  year = {2016},
  booktitle = {Proceedings of the IEEE Conference on Computer Vision and Pattern Recognition (CVPR)},
  pages = {1010--1019},
  doi = {10.1109/CVPR.2016.115},
}

@inproceedings{smolaHilbertSpaceEmbedding2007,
  title = {A {{Hilbert Space Embedding}} for {{Distributions}}},
  booktitle = {Algorithmic {{Learning Theory}}},
  author = {Smola, Alex and Gretton, Arthur and Song, Le and Sch{\"o}lkopf, Bernhard},
  editor = {Hutter, Marcus and Servedio, Rocco A. and Takimoto, Eiji},
  year = {2007},
  pages = {13--31},
  publisher = {Springer},
  doi = {10.1007/978-3-540-75225-7_5}
}

@article{sriperumbudurUniversalityCharacteristicKernels2011,
  title = {Universality, {{Characteristic Kernels}} and {{RKHS Embedding}} of {{Measures}}},
  author = {Sriperumbudur, Bharath K. and Fukumizu, Kenji and Lanckriet, Gert R. G.},
  year = {2011},
  journal = {Journal of Machine Learning Research},
  volume = {12},
  number = {70},
  pages = {2389--2410}
}

@article{supekarNetworkAnalysisIntrinsic2008,
  title = {Network {{Analysis}} of {{Intrinsic Functional Brain Connectivity}} in {{Alzheimer's Disease}}},
  author = {Supekar, Kaustubh and Menon, Vinod and Rubin, Daniel and Musen, Mark and Greicius, Michael D.},
  year = {2008},
  journal = {PLOS Computational Biology},
  volume = {4},
  number = {6},
  pages = {e1000100},
  doi = {10.1371/journal.pcbi.1000100},
}

@inproceedings{tangWassersteinDistributionalLearning2023,
  title = {Wasserstein {{Distributional Learning}} via {{Majorization-Minimization}}},
  booktitle = {Proceedings of {{The}} 26th {{International Conference}} on {{Artificial Intelligence}} and {{Statistics}}},
  author = {Tang, Chengliang and Lenssen, Nathan and Wei, Ying and Zheng, Tian},
  year = {2023},
  pages = {10703--10731},
  publisher = {PMLR},
  volume = {206},
  url = {https://proceedings.mlr.press/v206/tang23b.html}
}

@article{thorndikeWhoBelongsFamily1953,
  title = {Who Belongs in the Family?},
  author = {Thorndike, Robert L.},
  year = {1953},
  journal = {Psychometrika},
  volume = {18},
  number = {4},
  pages = {267--276},
  doi = {10.1007/BF02289263},
}

@article{tibshiraniEstimatingNumberClusters2001,
  title = {Estimating the Number of Clusters in a Data Set via the Gap Statistic},
  author = {Tibshirani, Robert and Walther, Guenther and Hastie, Trevor},
  year = {2001},
  journal = {Journal of the Royal Statistical Society Series B: Statistical Methodology},
  volume = {63},
  number = {2},
  pages = {411--423},
  doi = {10.1111/1467-9868.00293},
}

@article{tolstikhinMinimaxEstimationKernel2017,
  title = {Minimax {{Estimation}} of {{Kernel Mean Embeddings}}},
  author = {Tolstikhin, Ilya and Sriperumbudur, Bharath K. and Muandet, Krikamol},
  year = {2017},
  journal = {Journal of Machine Learning Research},
  volume = {18},
  number = {86},
  pages = {1--47}
}

@book{titteringtonStatisticalAnalysisFinite1985,
  title = {Statistical Analysis of Finite Mixture Distributions},
  author = {Titterington, D. M. and Smith, A. F. M. and Makov, U. E.},
  year = {1985},
  publisher = {Wiley}
}

@inproceedings{trespMixturesGaussianProcesses2000,
  title = {Mixtures of {{Gaussian Processes}}},
  booktitle = {Advances in {{Neural Information Processing Systems}}},
  author = {Tresp, Volker},
  year = {2000},
  volume = {13},
  publisher = {MIT Press},
  pages = {654--660}
}

@article{wardHierarchicalGroupingOptimize1963,
  title = {Hierarchical {{Grouping}} to {{Optimize}} an {{Objective Function}}},
  author = {Ward, Jr., Joe H.},
  year = {1963},
  journal = {Journal of the American Statistical Association},
  volume = {58},
  number = {301},
  pages = {236--244},
  doi = {10.1080/01621459.1963.10500845}
}

@article{wuParkinsonsDiseaseRelated2015,
  title = {Parkinson's {{Disease-Related Spatial Covariance Pattern}} Identified with {{Resting-State Functional MRI}}},
  author = {Wu, Tao and Ma, Yilong and Zheng, Zheng and Peng, Shichun and Wu, Xiaoli and Eidelberg, David and Chan, Piu},
  year = {2015},
  journal = {Journal of Cerebral Blood Flow \& Metabolism},
  volume = {35},
  number = {11},
  pages = {1764--1770},
  doi = {10.1038/jcbfm.2015.118},
}

@article{wynneKernelTwosampleTest2022,
  title = {A Kernel Two-Sample Test for Functional Data},
  author = {Wynne, George and Duncan, Andrew B},
  year = {2022},
  journal = {Journal of Machine Learning Research},
  volume = {23},
  number = {73},
  pages = {1--51}
}

@article{yizongchengMeanShiftMode1995,
  title = {Mean Shift, Mode Seeking, and Clustering},
  author = {Cheng, Yizong},
  year = {1995},
  journal = {IEEE Transactions on Pattern Analysis and Machine Intelligence},
  volume = {17},
  number = {8},
  pages = {790--799},
  doi = {10.1109/34.400568}
}

@article{zhangBIRCHEfficientData1996,
  title = {{{BIRCH}}: An Efficient Data Clustering Method for Very Large Databases},
  author = {Zhang, Tian and Ramakrishnan, Raghu and Livny, Miron},
  year = {1996},
  journal = {ACM SIGMOD Record},
  volume = {25},
  number = {2},
  pages = {103--114},
  doi = {10.1145/235968.233324},
}

@article{panaretosStatisticalAspectsWasserstein2019,
  title = {Statistical Aspects of Wasserstein Distances},
  author = {Panaretos, Victor M. and Zemel, Yoav},
  journal = {Annual Review of Statistics and Its Application},
  volume = {6},
  pages = {405--431},
  year = {2019},
  doi = {10.1146/annurev-statistics-030718-104938}
}

@article{wadwa2023trial,
  title={Trial of hybrid closed-loop control in young children with type 1 diabetes},
  author = {Wadwa, R. Paul and Reed, Zachariah W. and Buckingham, Bruce A. and DeBoer, Mark D. and Ekhlaspour, Laya and Forlenza, Gregory P. and Schoelwer, Melissa and Lum, John and Kollman, Craig and Beck, Roy W. and Breton, Marc D.},
  journal={New England Journal of Medicine},
  volume={388},
  number={11},
  pages={991--1001},
  year={2023},
  note = {For the PEDAP Trial Study Group},
  doi = {10.1056/NEJMoa2210834}
}

@inproceedings{maLearningRepresentationsTime2019,
  title = {Learning {{Representations}} for {{Time Series Clustering}}},
  booktitle = {Advances in {{Neural Information Processing Systems}}},
  author = {Ma, Qianli and Zheng, Jiawei and Li, Sen and Cottrell, Gary W.},
  year = {2019},
  volume = {32},
  publisher = {Curran Associates, Inc.}
}

@inproceedings{chenAchievingOptimalClustering2024,
  title = {Achieving {{Optimal Clustering}} in {{Gaussian Mixture Models}} with {{Anisotropic Covariance Structures}}},
  booktitle = {Advances in {{Neural Information Processing Systems}}},
  author = {Chen, Xin and Zhang, Anderson Ye},
  year = {2024},
  volume = {37},
  publisher = {Curran Associates, Inc.},
  doi = {10.52202/079017-3610}
}

@inproceedings{akewejeGPmixModelbasedClustering2024,
  title = {Learning Mixtures of {Gaussian} Processes through Random Projection},
  author = {Akeweje, Emmanuel and Zhang, Mimi},
  booktitle = {Proceedings of the 41st International Conference on Machine Learning},
  series = {Proceedings of Machine Learning Research},
  volume = {235},
  pages = {720--739},
  year = {2024},
  publisher = {PMLR},
  url = {https://proceedings.mlr.press/v235/akeweje24a.html}
}

\appendix

\paragraph{Appendix organization.}
\Cref{sec:prelim} provides background for Gaussian measures defined on Hilbert spaces and Maximum Mean Discrepancy, while \Cref{app:mmd_closed_forms} gives the closed-form MMD expressions and fitting algorithms. \Cref{app:time-varying,sec:likelihood,app:banach_kl} cover time-varying mixtures, projected likelihood estimation, and Banach-space extensions. \Cref{app:synthetic_experiments,app:experimental_details,app:additional_diagnostics} present additional experiments, implementation details, and numerical diagnostics. Finally, \Cref{sec:proofs} contains all proofs.

\section{Background}\label{sec:prelim}

This section establishes the notation and foundational properties of Gaussian measures and kernel methods used throughout this work. 

\subsection{Gaussian measures on separable Hilbert spaces}\label{sec:prelim_gaussian}

Let \(\mathcal{X}\) be a separable real Hilbert space equipped with its Borel \(\sigma\)-algebra \(\mathcal{B}(\mathcal{X})\). A Borel probability measure \(\nu\) on \(\mathcal{X}\) is said to be \emph{Gaussian} if, for every continuous linear functional \(f \in \mathcal{X}^*\), the pushforward measure \(\nu \circ f^{-1}\) is a Gaussian measure on \(\mathbb{R}\). By the Riesz representation theorem, we routinely identify \(\mathcal{X}^*\) with \(\mathcal{X}\).

A Gaussian measure \(\nu\) is uniquely determined by its mean \(m \in \mathcal{X}\) and its covariance operator \(\mathcal{K}: \mathcal{X} \to \mathcal{X}\), which are defined via the relations
\[
\langle m, h \rangle_{\mathcal{X}} = \int_{\mathcal{X}} \langle x, h \rangle_{\mathcal{X}} \diff \nu(x), \quad h \in \mathcal{X},
\]
and
\[
\langle \mathcal{K} g, h \rangle_{\mathcal{X}} = \int_{\mathcal{X}} \langle x-m, g \rangle_{\mathcal{X}} \langle x-m, h \rangle_{\mathcal{X}} \diff \nu(x), \quad g, h \in \mathcal{X}.
\]
We denote such a measure by \(\nu = \mathcal{N}(m, \mathcal{K})\). The covariance operator \(\mathcal{K}\) is necessarily self-adjoint, positive semidefinite, and trace-class. The latter property implies that for any orthonormal basis \((e_i)_{i \ge 1}\) of \(\mathcal{X}\), the trace is finite:
\[
\operatorname{tr}(\mathcal{K}) = \sum_{i=1}^\infty \langle \mathcal{K} e_i, e_i \rangle_{\mathcal{X}} < \infty.
\]
The case \(\mathcal{K}=0\) corresponds to the Dirac measure \(\delta_m\); this case is excluded from the model class in the main text. In particular, if \(X \sim \mathcal{N}(m, \mathcal{K})\), its moments satisfy \(\mathbb{E}\|X-m\|_{\mathcal{X}}^2 = \operatorname{tr}(\mathcal{K})\) and \(\mathbb{E}\|X\|_{\mathcal{X}}^2 = \|m\|_{\mathcal{X}}^2 + \operatorname{tr}(\mathcal{K})\). 

\subsection{Kernel mean embeddings and Maximum Mean Discrepancy}\label{sec:prelim_mmd}

Let \(\kappa: \mathcal{X} \times \mathcal{X} \to \mathbb{R}\) be a continuous positive-definite kernel with an associated reproducing kernel Hilbert space (RKHS) \(\mathcal{H}_\kappa\). Its canonical feature map $\Phi(x)=\kappa(x,\cdot)$ is continuous because
\[
\|\Phi(x)-\Phi(y)\|_{\mathcal H_\kappa}^2=\kappa(x,x)+\kappa(y,y)-2\kappa(x,y).
\]
Since $\mathcal X$ is separable, the RKHS is separable as well: the span of the features of a countable dense subset of $\mathcal X$ is dense in $\mathcal H_\kappa$. Thus $\Phi$ is strongly measurable. For probability measures \(P, Q \in \mathcal{P}(\mathcal{X})\), we assume the following integrability condition holds:
\begin{equation}\label{eq:kappa_integrability}
\mathbb{E}_{x \sim P} \left[ \sqrt{\kappa(x,x)} \right] < \infty, \quad \mathbb{E}_{y \sim Q} \left[ \sqrt{\kappa(y,y)} \right] < \infty.
\end{equation}
Under this assumption, the kernel mean embeddings
\[
\mu_P \coloneqq \mathbb{E}_{x \sim P}[\kappa(x, \cdot)], \quad \mu_Q \coloneqq \mathbb{E}_{y \sim Q}[\kappa(y, \cdot)]
\]
are well-defined as Bochner integrals in \(\mathcal{H}_\kappa\), since $\|\Phi(x)\|_{\mathcal H_\kappa}=\sqrt{\kappa(x,x)}$. The squared Maximum Mean Discrepancy (MMD) between \(P\) and \(Q\) is defined as the squared distance between their embeddings:
\begin{align}
\operatorname{MMD}_\kappa^2(P, Q) &\coloneqq \|\mu_P - \mu_Q\|_{\mathcal{H}_\kappa}^2 \nonumber \\
&= \mathbb{E}_{x,x' \sim P}[\kappa(x,x')] - 2\mathbb{E}_{x \sim P, y \sim Q}[\kappa(x,y)] + \mathbb{E}_{y,y' \sim Q}[\kappa(y,y')], \label{eq:mmd_expansion}
\end{align}
where \(x, x'\) are i.i.d.\ samples from \(P\) and \(y, y'\) are i.i.d.\ samples from \(Q\). Note that if \(\kappa\) is bounded, condition \eqref{eq:kappa_integrability} is satisfied for all probability measures.

A kernel is called \emph{characteristic} if the embedding map \(P \mapsto \mu_P\) is injective. In this case, \(\operatorname{MMD}_\kappa(P, Q) = 0\) if and only if \(P = Q\), meaning the MMD induces a true metric on the space of probability measures.

\subsection{Fourier interpretation of the Gaussian MMD}\label{sec:gaussian_mmd_interpretation}

The Gaussian radial kernel, given by
\begin{equation}\label{eq:gaussian_radial_kernel}
\kappa(x,y) \coloneqq \exp\left( -\frac{\|x-y\|_{\mathcal{X}}^2}{2\sigma_\kappa^2} \right), \quad \sigma_\kappa > 0,
\end{equation}
is both bounded and characteristic on separable Hilbert spaces \cite{wynneKernelTwosampleTest2022}.

To gain intuition into the geometry induced by this kernel, we briefly consider the finite-dimensional case \(\mathcal{X} = \mathbb{R}^d\). For \(P, Q \in \mathcal{P}(\mathbb{R}^d)\), let \(\varphi_P\) and \(\varphi_Q\) denote their respective characteristic functions, \(\varphi_P(\omega) \coloneqq \int_{\mathbb{R}^d} e^{i\omega^\top x} \diff P(x)\). 

By Bochner's theorem, the shift-invariant kernel \eqref{eq:gaussian_radial_kernel} admits a Fourier representation with a Gaussian spectral density \(\psi_{\sigma_\kappa}\):
\begin{equation*}
\kappa(x,y) = \frac{\sigma_\kappa^d}{(2\pi)^{d/2}} \int_{\mathbb{R}^d} e^{i\omega^\top (x-y)} e^{-\sigma_\kappa^2\|\omega\|^2/2} \diff \omega.
\end{equation*}

By defining the signed measure \(\mu \coloneqq P-Q\) and applying Fubini's theorem to the expanded MMD objective \eqref{eq:mmd_expansion}, we obtain:
\begin{align*}
\operatorname{MMD}_{\kappa}^2(P,Q) &= \iint_{\mathbb{R}^d \times \mathbb{R}^d} \kappa(x,y) \diff \mu(x) \diff \mu(y) \\
&= \frac{\sigma_\kappa^d}{(2\pi)^{d/2}} \int_{\mathbb{R}^d} \left| \int_{\mathbb{R}^d} e^{i\omega^\top x} \diff \mu(x) \right|^2 e^{-\sigma_\kappa^2\|\omega\|^2/2} \diff \omega \\
&= \frac{\sigma_\kappa^d}{(2\pi)^{d/2}} \int_{\mathbb{R}^d} \big| \varphi_P(\omega) - \varphi_Q(\omega) \big|^2 e^{-\sigma_\kappa^2\|\omega\|^2/2} \diff \omega.
\end{align*}

This formulation reveals a frequency-domain perspective: the squared Gaussian MMD is a weighted \(L^2\) distance between the characteristic functions of the two distributions. The Gaussian weight \(e^{-\sigma_\kappa^2\|\omega\|^2/2}\) acts as a low-pass filter, effectively penalizing discrepancies in low-frequency modes (global structural differences) while naturally smoothing out high-frequency noise.

\section{Closed-form MMD expressions}\label{app:mmd_closed_forms}

To evaluate the cross-expectations $J_{i,k}$ and $I_{k,s}$ of \eqref{eq:JI_definitions} in the infinite-dimensional setting, we rely on the Fredholm determinant. Recall that for any trace-class operator $T \in \mathcal{L}(\mathcal{X})$, the determinant $\det_F(\Id + T)$ is well-defined and serves as the natural generalization of the standard matrix determinant.

\begin{prop}[MMD for the Gaussian kernel]\label{prop:mmd_closed_forms}
	Let $A \in \mathcal{L}(\mathcal{X})$ be a self-adjoint, positive semidefinite operator, and consider the Gaussian kernel induced by $A$,
 	\begin{equation}\label{eq:gaussian_kernel}
 		\kappa_A(x,y)\coloneqq\exp\left(-\frac{1}{2}\big\langle x-y,A(x-y)\big\rangle_{\mathcal{X}}\right).
	\end{equation}
	Then the expectations in \eqref{eq:JI_definitions} are given by
	\begingroup\small
	\begin{align}
		J_{i,k}
		&=
		\frac{\exp\Bigl(-\tfrac12\big\langle X_i-m_k,\, A^{1/2}\bigl(\Id + A^{1/2}\mathcal{K}_k A^{1/2}\bigr)^{-1}A^{1/2}(X_i-m_k)\big\rangle_{\mathcal{X}}\Bigr)}{\det_F\bigl(\Id + A^{1/2}\mathcal{K}_k A^{1/2}\bigr)^{1/2}}, \label{eq:J_gaussian} \\[2ex]
		I_{k,s}
		&=
		\frac{\exp\Bigl(-\tfrac12\big\langle m_k-m_s,\, A^{1/2}\bigl(\Id + A^{1/2}(\mathcal{K}_k+\mathcal{K}_s)A^{1/2}\bigr)^{-1}A^{1/2}(m_k-m_s)\big\rangle_{\mathcal{X}}\Bigr)}{\det_F\bigl(\Id + A^{1/2}(\mathcal{K}_k+\mathcal{K}_s)A^{1/2}\bigr)^{1/2}}.
		\label{eq:I_gaussian}
	\end{align}
	\endgroup
\end{prop}


\begin{prop}[MMD for the polynomial kernel]\label{prop:mmd_closed_forms_polynomial}
	Fix $p\in\mathbb{N}$, $c\ge 0$, and consider the kernel
 	\begin{equation}\label{eq:polykernel}
 	  \kappa(x,y)\coloneqq(\langle x,y\rangle_{\mathcal{X}}+c)^p.
	\end{equation}
	Then the expectations in \eqref{eq:JI_definitions} are given by
	\begingroup\small
	\begin{align*}
		J_{i,k}
		&=\sum_{t=0}^{\lfloor p/2\rfloor}\frac{p!}{(p-2t)!\,2^t\,t!}\,\langle X_i,\mathcal{K}_k X_i\rangle_{\mathcal{X}}^t\,\left(\langle X_i,m_k\rangle_{\mathcal{X}}+c\right)^{p-2t}, \\
		I_{k,s}
		&=\sum_{r=0}^{p}\binom{p}{r}c^{p-r}\,\big\langle M_{r,k}, M_{r,s}\big\rangle_{\mathcal{X}^{\otimes r}},\qquad \text{with }M_{r,k}\coloneqq\mathbb{E}_{y\sim \nu_k}[y^{\otimes r}]\in\mathcal{X}^{\otimes r},\quad M_{0,k}=1.
	\end{align*}
	\endgroup
\end{prop}

The formal proofs of Propositions \ref{prop:mmd_closed_forms} and \ref{prop:mmd_closed_forms_polynomial} are given in \Cref{sec:proof_mmd_closed_forms,sec:proof_mmd_closed_forms_polynomial}.

\subsection{Numerical computations and algorithm description}\label{app:explicit_formulas_rm}
Given an orthonormal basis $(e_r)_{r=1}^M$ of $\mathcal{X}_M$, we represent observations and means by their coordinate vectors
\begin{equation}\label{eq:coord_vectors}
\mathbf{x}_i \coloneqq \big(\langle X_i,e_r\rangle_{\mathcal{X}}\big)_{r=1}^M\in\R^M,
\qquad
\mathbf{m}_k \coloneqq \big(\langle m_k,e_r\rangle_{\mathcal{X}}\big)_{r=1}^M\in\R^M,
\end{equation}
and the covariance operator $\mathcal{K}_{k,M}$ by its matrix representation $\mathbf{K}_k\in\R^{M\times M}$ with entries
\begin{equation}\label{eq:cov_matrix}
(\mathbf{K}_k)_{r\ell}\coloneqq \langle \mathcal{K}_{k,M} e_\ell,e_r\rangle_{\mathcal{X}}
\qquad 1\le r,\ell\le M.
\end{equation}
In these coordinates, the pushforward measure $\nu_{k,M}$ becomes the multivariate normal $\mathcal{N}(\mathbf{m}_k,\mathbf{K}_k)$ on $\R^M$.
Using these coordinate representations, we now provide explicit formulas for the finite-dimensional approximations.

\paragraph{Projected responsibilities.}
The pushforward $\nu_{k,M}=\nu_k\circ \Pi_M^{-1}$ is a multivariate normal distribution on $\mathbb{R}^M$,
\[
\nu_{k,M}=\mathcal{N}(\mathbf{m}_k,\mathbf{K}_k).
\]
If each $\mathbf{K}_k$ is positive definite, then each $\nu_{k,M}$ admits the Lebesgue density
\[
p_k^{(M)}(\mathbf{x})
\coloneqq (2\pi)^{-M/2}\det(\mathbf{K}_k)^{-1/2}
\exp\!\Big\{-\tfrac12(\mathbf{x}-\mathbf{m}_k)^\top \mathbf{K}_k^{-1}(\mathbf{x}-\mathbf{m}_k)\Big\}.
\]
In this case the projected responsibilities satisfy, for $i\in\{1,\dots,n\}$ and $k\in\{1,\dots,K\}$,
\[
\gamma^{(M)}_{i,k}
\coloneqq \gamma_k^{(M)}(\mathbf{x}_i)
=\frac{\pi_k p_k^{(M)}(\mathbf{x}_i)}{\sum_{s=1}^K \pi_s p_s^{(M)}(\mathbf{x}_i)}.
\]
If some $\mathbf{K}_k$ are singular, the corresponding $\nu_{k,M}$ are supported on affine subspaces; the Radon--Nikodym definition of $\gamma_k^{(M)}$ still applies and one may compute densities on the supports or use a small ridge regularization.

\vspace{0.4em}
\noindent\textbf{1. Gaussian radial kernel (Proposition \ref{prop:mmd_closed_forms})}
Let $\mathbf{A}\in\R^{M\times M}$ denote the matrix representation of the compressed operator $A_M\coloneqq\Pi_M A|_{\mathcal{X}_M}\colon\mathcal{X}_M\to\mathcal{X}_M$ in the basis $(e_r)_{r=1}^M$, with entries $\mathbf{A}_{r\ell}=\langle A e_\ell,e_r\rangle_{\mathcal{X}}$. The closed forms of Proposition \ref{prop:mmd_closed_forms} become the finite-dimensional approximations
 \begin{align*}
 J^{(M)}_{i,k}
&\coloneqq \det\!\big(I_M+\mathbf{A}^{1/2}\mathbf{K}_k\mathbf{A}^{1/2}\big)^{-1/2}
\\
&\qquad \times
\exp\!\Big\{-\tfrac{1}{2}(\mathbf{x}_i-\mathbf{m}_k)^\top
\mathbf{A}^{1/2}\big(I_M+\mathbf{A}^{1/2}\mathbf{K}_k\mathbf{A}^{1/2}\big)^{-1}\mathbf{A}^{1/2}(\mathbf{x}_i-\mathbf{m}_k)\Big\},\\
I^{(M)}_{k,s}
&\coloneqq \det\!\big(I_M+\mathbf{A}^{1/2}(\mathbf{K}_k+\mathbf{K}_s)\mathbf{A}^{1/2}\big)^{-1/2}
\\
&\qquad \times
\exp\!\Big\{-\tfrac{1}{2}(\mathbf{m}_k-\mathbf{m}_s)^\top
\mathbf{A}^{1/2}\big(I_M+\mathbf{A}^{1/2}(\mathbf{K}_k+\mathbf{K}_s)\mathbf{A}^{1/2}\big)^{-1}\mathbf{A}^{1/2}(\mathbf{m}_k-\mathbf{m}_s)\Big\},
\end{align*}
and we set
\[
J^{(M)}_k\coloneqq \frac1n\sum_{i=1}^n J^{(M)}_{i,k},
\qquad
\mathbf{J}^{(M)}\coloneqq (J^{(M)}_1,\dots,J^{(M)}_K)^\top,
\qquad
\mathbf{I}^{(M)}\coloneqq (I^{(M)}_{k,s})_{k,s=1}^K.
\]
For the isotropic case $\mathbf{A}=\sigk^{-2}I_M$, one writes $\alpha=-1/(2\sigk^2)$ and recovers the equivalent forms $\det(I_M-2\alpha\mathbf{K}_k)^{-1/2}$ and $(I_M-2\alpha\mathbf{K}_k)^{-1}$ in the exponents.

For numerically stable log-domain evaluation, compute the symmetric eigendecompositions
\[
\mathbf{A}^{1/2}\mathbf{K}_k\mathbf{A}^{1/2} = \mathbf{V}_k\mathrm{diag}(\lambda_{k,1},\dots,\lambda_{k,M})\mathbf{V}_k^\top,
\]
\[
\mathbf{A}^{1/2}(\mathbf{K}_k+\mathbf{K}_s)\mathbf{A}^{1/2} = \mathbf{V}_{k,s}\mathrm{diag}(\lambda_{k,s,1},\dots,\lambda_{k,s,M})\mathbf{V}_{k,s}^\top,
\]
with orthogonal matrices $\mathbf{V}_k,\mathbf{V}_{k,s}$ and $\lambda_{\cdot,\cdot}\ge 0$. Note that $\det(I_M+\mathbf{A}^{1/2}\mathbf{K}_k\mathbf{A}^{1/2})=\prod_r(1+\lambda_{k,r})$. Writing $\mathbf{z}_{i,k}\coloneqq \mathbf{V}_k^\top\mathbf{A}^{1/2}(\mathbf{x}_i-\mathbf{m}_k)$ and $\mathbf{d}_{k,s}\coloneqq \mathbf{V}_{k,s}^\top\mathbf{A}^{1/2}(\mathbf{m}_k-\mathbf{m}_s)$, we obtain the stable log-domain expressions
\begin{align*}
\log J^{(M)}_{i,k}
&= -\frac12\sum_{r=1}^M \log(1+\lambda_{k,r})
\;-\;\frac{1}{2}\sum_{r=1}^M \frac{(\mathbf{z}_{i,k})_r^2}{1+\lambda_{k,r}},\\
\log I^{(M)}_{k,s}
&= -\frac12\sum_{r=1}^M \log(1+\lambda_{k,s,r})
\;-\;\frac{1}{2}\sum_{r=1}^M \frac{(\mathbf{d}_{k,s})_r^2}{1+\lambda_{k,s,r}}.
\end{align*}
When $\mathbf{A}=\sigk^{-2}I_M$, the matrices $\mathbf{V}_k$ coincide with the eigenvectors of $\mathbf{K}_k$ and the expressions reduce to those with $\alpha=-1/(2\sigk^2)$.
To reduce cost, one may truncate these sums to $r\le R$ (e.g.\ the leading $R$ eigenvalues), i.e.\ replace the relevant matrices by their rank-$R$ spectral approximations.

\vspace{0.4em}
\noindent\textbf{2. Polynomial kernel (Proposition \ref{prop:mmd_closed_forms_polynomial})}
Fix an integer $p\ge 1$ and $c\ge 0$ and consider $\kappa(x,y)=(\langle x,y\rangle_{\mathcal{X}}+c)^p$.
In the projected space, this becomes $\kappa(\mathbf{x},\mathbf{y})=(\mathbf{x}^\top\mathbf{y}+c)^p$ on $\mathbb{R}^M$.

\paragraph{Computing $J^{(M)}_{i,k}$.}
For $y\sim \nu_{k,M}$, define the scalars
\[
\mu^{(M)}_{i,k}\coloneqq \mathbf{x}_i^\top\mathbf{m}_k+c,
\qquad
v^{(M)}_{i,k}\coloneqq \mathbf{x}_i^\top \mathbf{K}_k\,\mathbf{x}_i,
\]
so that $\mathbf{x}_i^\top y+c$ is a one-dimensional Gaussian with mean $\mu^{(M)}_{i,k}$ and variance $v^{(M)}_{i,k}$.
Then Proposition \ref{prop:mmd_closed_forms_polynomial} yields the efficient closed form
\[
J^{(M)}_{i,k}
=
\sum_{t=0}^{\lfloor p/2\rfloor}\frac{p!}{(p-2t)!\,2^t\,t!}\,\big(v^{(M)}_{i,k}\big)^{t}\,\big(\mu^{(M)}_{i,k}\big)^{p-2t}.
\]

\paragraph{Computing $I^{(M)}_{k,s}$.}
Let $y\sim \nu_{k,M}$ and $y'\sim \nu_{s,M}$ be independent in $\mathbb{R}^M$, and set $Z_{k,s}\coloneqq y^\top y'$.
Then
\[
I^{(M)}_{k,s}
=
\sum_{r=0}^{p}\binom{p}{r}c^{p-r}\,\mathbb{E}\!\left[Z_{k,s}^{\,r}\right].
\]
Equivalently, by Proposition \ref{prop:mmd_closed_forms_polynomial},
\[
I^{(M)}_{k,s}
=
\sum_{r=0}^{p}\binom{p}{r}c^{p-r}\,
\big\langle \mathbf{M}^{(M)}_{r,k},\mathbf{M}^{(M)}_{r,s}\big\rangle_{(\mathbb{R}^M)^{\otimes r}},
\]
where
\[
\mathbf{M}^{(M)}_{r,k}\coloneqq \mathbb{E}_{y\sim \nu_{k,M}}[y^{\otimes r}] \in (\mathbb{R}^M)^{\otimes r}.
\]
For small degrees, these moments admit simple closed forms in terms of \((\mathbf{m}_k,\mathbf{K}_k)\). In particular,
\begin{align*}
\mathbb{E}[Z_{k,s}]
&= \mathbf{m}_k^\top \mathbf{m}_s,\\
\mathbb{E}[Z_{k,s}^2]
&= (\mathbf{m}_k^\top \mathbf{m}_s)^2
+ \mathbf{m}_k^\top \mathbf{K}_s\,\mathbf{m}_k
+ \mathbf{m}_s^\top \mathbf{K}_k\,\mathbf{m}_s
+ \mathrm{tr}(\mathbf{K}_k\mathbf{K}_s),\\
\mathbb{E}[Z_{k,s}^3]
&= (\mathbf{m}_k^\top \mathbf{m}_s)^3
+3(\mathbf{m}_k^\top \mathbf{m}_s)\Big(\mathbf{m}_k^\top \mathbf{K}_s\,\mathbf{m}_k+\mathbf{m}_s^\top \mathbf{K}_k\,\mathbf{m}_s
+\mathrm{tr}(\mathbf{K}_k\mathbf{K}_s)\Big)
+6\,\mathbf{m}_k^\top \mathbf{K}_s\mathbf{K}_k\,\mathbf{m}_s.
\end{align*}
For arbitrary degree $p$, the implementation computes these moments exactly using cumulants of a Gaussian quadratic form. Set
\[
\mathbf a=\begin{pmatrix}\mathbf m_k\\\mathbf m_s\end{pmatrix},\qquad
\mathbf B=\begin{pmatrix}0&\mathbf K_k\\\mathbf K_s&0\end{pmatrix},\qquad
\mathbf F=\begin{pmatrix}0&I_M\\I_M&0\end{pmatrix}.
\]
Expanding the log moment-generating function
\[
\log\mathbb E[e^{tZ_{k,s}}]=
-\tfrac12\log\det(I_{2M}-t\mathbf B)
+\tfrac t2\mathbf a^\top\mathbf F(I_{2M}-t\mathbf B)^{-1}\mathbf a
\]
for $t$ near zero gives the cumulants and raw-moment recurrence
\begin{align*}
\chi_r&=\frac{(r-1)!}{2}\operatorname{tr}(\mathbf B^r)
+\frac{r!}{2}\mathbf a^\top\mathbf F\mathbf B^{r-1}\mathbf a,\qquad r\ge1,\\
u_0&=1,\qquad
u_r=\sum_{j=1}^r\binom{r-1}{j-1}\chi_j u_{r-j}
=\mathbb E[Z_{k,s}^r].
\end{align*}
These identities also hold for singular covariances. The implementation uses $M$-dimensional blocks: odd trace terms vanish and $\operatorname{tr}(\mathbf B^{2j})=2\operatorname{tr}((\mathbf K_k\mathbf K_s)^j)$. Computing all cumulants and moments through degree $p$ costs $O(pM^3+p^2)$ per component pair.

\subsection{Algorithm descriptions}\label{app:algorithms}

We present pseudocode for the fitting procedure described in \Cref{sec:fitting}.
The algorithm operates in the projected coordinate space $\R^M$ using the finite-dimensional quantities $(\mathbf{J}^{(M)}, \mathbf{I}^{(M)})$ from \Cref{app:explicit_formulas_rm}, which we denote $(\mathbf{J}, \mathbf{I})$ for brevity.

\begin{algorithm}[t]
\caption{Alternating optimization for projected MMD Gaussian mixtures}
\label{alg:mmd_fitting}
\begin{algorithmic}[1]
\Require Projected data $\{\mathbf{x}_{i,M}\}_{i=1}^n \subset \mathbb{R}^M$; components $K$; learning rate $\eta$; ridge $\varepsilon>0$; iterations $T$
\Ensure Fitted weights $\hat{\pi}$ and components $\{(\hat{\mathbf{m}}_{k,M},\hat{\mathbf{K}}_{k,M})\}_{k=1}^K$

\State $\{\mathbf{m}_{k,M}\}_{k=1}^K \gets \textsc{KMeans}(\{\mathbf{x}_{i,M}\}_{i=1}^n, K)$ \Comment{Initialize means from data}
\State Assign $c_i\gets\arg\min_k \|\mathbf{x}_{i,M}-\mathbf{m}_{k,M}\|$ and let $S_k=\{i:c_i=k\}$ \Comment{Nearest-center partition}
\State $\mathbf{K}_{k,M}\gets \tfrac{1}{|S_k|-1}\!\sum_{i\in S_k}(\mathbf{x}_{i,M}-\mathbf{m}_{k,M})(\mathbf{x}_{i,M}-\mathbf{m}_{k,M})^\top + \varepsilon I_M$ \Comment{Within-cluster covariance}
 \State $\mathbf{L}_k\gets \mathrm{chol}(\mathbf{K}_{k,M})$ \Comment{Cholesky parametrization ensures $\mathbf{K}_{k,M}\succ 0$}
 \State $\pi \gets \mathbf{1}/K$ \Comment{Uniform weight initialization}

\For{$t=1,\dots,T$}
    \State Compute $\mathbf{I}^{(M)} \in \mathbb{R}^{K \times K}$, $\mathbf{J}^{(M)} \in \mathbb{R}^K$ using \Cref{app:explicit_formulas_rm} \Comment{Cross-expectations}
    \State $\pi \gets \arg\min_{\pi\in\Delta^{K-1}} \pi^\top \mathbf{I}^{(M)}\pi - 2(\mathbf{J}^{(M)})^\top \pi$ \Comment{Exact QP update \eqref{eq:mmd_quadratic}}
    \State $\mathcal{L}_M \gets \pi^\top \mathbf{I}^{(M)}\pi - 2(\mathbf{J}^{(M)})^\top \pi$ \Comment{Objective at fixed $\pi$}
    \State $\mathbf{m}_{k,M} \gets \mathbf{m}_{k,M} - \eta \nabla_{\mathbf{m}_{k,M}} \mathcal{L}_M$ \Comment{Gradient step on means}
    \State $\mathbf{L}_k \gets \mathbf{L}_k - \eta \nabla_{\mathbf{L}_k} \mathcal{L}_M$ \Comment{Gradient step on Cholesky factors}
    \State $\mathbf{K}_{k,M} \gets \mathbf{L}_k \mathbf{L}_k^\top + \varepsilon I_M$ \Comment{Reconstruct PSD covariances}
\EndFor

\State \Return $\hat{\pi} \gets \pi$ and $\{(\hat{\mathbf{m}}_{k,M}, \hat{\mathbf{K}}_{k,M})\} \gets \{(\mathbf{m}_{k,M}, \mathbf{K}_{k,M})\}$
\end{algorithmic}
\end{algorithm}

\Cref{alg:mmd_fitting} alternates between the analytical weight update~\eqref{eq:mmd_quadratic} and a gradient step on the component parameters $(\mathbf{m}_{k,M}, \mathbf{L}_k)$, with covariances parameterized through their Cholesky factors to enforce positive definiteness. The exact gradients are obtained from the differentiable closed forms of Propositions \ref{prop:mmd_closed_forms} and \ref{prop:mmd_closed_forms_polynomial}.
An alternative is a fully gradient-based variant in which $\pi$ is updated jointly with the components, either by projecting onto the simplex or by using logits $\ell\in\R^K$ with $\pi_k=\operatorname{softmax}(\ell)_k$.

\section{Time-varying mixtures}\label{app:time-varying}

This appendix details the algorithmic implementation and theoretical results deferred from \Cref{sec:time_varying}. We outline the projected $\mathbb{R}^M$ fitting procedure, and subsequently establish three guarantees: the consistency of the integrated empirical objective, the consistency of its projected $\mathbb{R}^M$ approximation, and the density of time-varying Gaussian mixtures with shared components.

\begin{algorithm}[htbp]
\caption{Time-varying MMD-based mixture fitting}
\label{alg:temporal_fitting}
\begin{algorithmic}[1]
\Require Projected temporal data $\{\mathbf{x}_{i,M}^{(l)}\}_{i=1}^{n_l} \subset \mathbb{R}^M$ across time slices $t_1, \dots, t_L$; number of components $K$; learning rate $\eta$; epochs $E$; neural ODE parameters $\vartheta$.
\Ensure Fitted parameters $\hat{\vartheta}$ and shared components $\{(\hat{\mathbf{m}}_k, \hat{\mathbf{K}}_k)\}_{k=1}^K$.

\State Initialize $\{(\mathbf{m}_k, \mathbf{K}_k)\}_{k=1}^K$ using the pooled data $\bigcup_{l=1}^L \{\mathbf{x}_{i,M}^{(l)}\}_{i=1}^{n_l}$.
\State Initialize $\vartheta$.

\For{\textit{epoch} $= 1,\dots, E$}
    \State Compute $\pi(t_l) \gets \operatorname{softmax}(\mathbf{z}(t_l))$ for $l \in \{1,\dots,L\}$ via the ODE.
    \State Compute $\mathbf{I}^{(M)} \in \mathbb{R}^{K \times K}$ via \eqref{eq:proj.cross.exp}.
    
    \For{$l = 1,\dots, L$}

\State Compute $\mathbf{J}_l^{(M)} \in \mathbb{R}^K$ where
$[\mathbf{J}_l^{(M)}]_k \gets
\frac{1}{n_l}\sum_{i=1}^{n_l} J_{i,k}^{(M)}(t_l)$
via \eqref{eq:proj.cross.exp}.

\State $\mathcal{L}_l^{(M)} \gets
\pi(t_l)^\top \mathbf{I}^{(M)}\pi(t_l)
- 2(\mathbf{J}_l^{(M)})^\top \pi(t_l)$.
\Comment{MMD$^2$ up to a constant}

\EndFor

\State $\mathcal{L} \gets
\frac{1}{L}\sum_{l=1}^L \mathcal{L}_l^{(M)}$.
\Comment{Approximate integrated loss up to a constant}
    \State Update $(\vartheta, \{\mathbf{m}_k, \mathbf{K}_k\}) \gets (\vartheta, \{\mathbf{m}_k, \mathbf{K}_k\}) - \eta\,\nabla \mathcal{L}$.
\EndFor

\State \Return $\hat{\vartheta} \gets \vartheta$, \; $\{(\hat{\mathbf{m}}_k, \hat{\mathbf{K}}_k)\}_{k=1}^K \gets \{(\mathbf{m}_k, \mathbf{K}_k)\}_{k=1}^K$.
\end{algorithmic}
\end{algorithm}

\Cref{alg:temporal_fitting} implements the optimization of the time-varying objective \eqref{eq:time_varying_loss} in $\mathbb{R}^M$. The component parameters $\{(\mathbf{m}_k, \mathbf{K}_k)\}_{k=1}^K$ are shared across all time slices, while the temporal variation is entirely captured by the weights $\pi(t)$, whose logits are governed by the neural ODE introduced in \Cref{sec:time_varying}. Implementation details are summarized in \Cref{tab:figure_hyperparams}.

To evaluate the MMD efficiently within the algorithm, we rely on the finite-dimensional projections defined in \eqref{eq:proj.cross.exp}. We separate the evaluation into terms that can be computed globally versus those computed per time slice. The matrix $\mathbf{I}^{(M)}$ depends only on the shared components and is computed once per epoch. In contrast, the vectors $\mathbf{J}_l^{(M)}$ evaluate the interaction between the shared components and the data specific to time slice $t_l$, requiring per-slice computation.

Motivated by the within-cluster RKHS dispersion criterion underlying kernel $k$-groups \cite{francaKernelKGroupsHartigans2017}, {we examine the MMD loss over candidate component counts as an elbow diagnostic \cite{thorndikeWhoBelongsFamily1953,tibshiraniEstimatingNumberClusters2001}, supplemented by interpretability when the elbow is ambiguous.}

\paragraph{Model selection.}
The kernel is fixed before selecting $K$ and we use the
Gaussian radial kernel by default, as it is characteristic, while
the polynomial kernel is appropriate when the focus is on matching
low-order moments. Comparisons across kernels require an external evaluation criterion, such as ARI when ground-truth labels are available.

\Cref{fig:cgm_model_selection} shows the CGM loss values for $K=1,\ldots,6$. For each representation, we used the Gaussian radial kernel and initialized the model with $k$-means++. $K=2$ for intraday curves was guided by interpretability and treatment--placebo contrast, whereas $K=3$ was used for correlation matrices and similarity-graph signals. Because group contrast informed the intraday choice, that contrast is an exploratory description rather than an independent validation of model selection.

\begin{figure}[tbp]
    \centering
    \includegraphics[width=0.7\textwidth]{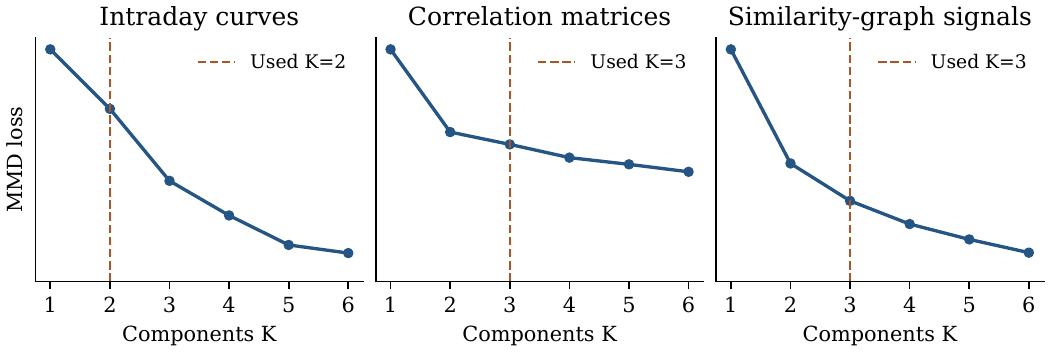}
    \caption{{CGM model-selection: normalized MMD loss versus $K$ for the three geometric representations. Dashed lines mark the component counts used in the case study.}}
    \label{fig:cgm_model_selection}
\end{figure}


\subsection*{Consistency guarantees}

In parallel with the static framework of \Cref{sec:theory}, we establish the theoretical foundations for the temporal mixture model. Unlike the static case, where sample consistency is standard, minimizing the integrated temporal objective requires controlling the time-varying empirical processes. We first establish this sample consistency, followed by the validity of the projected approximation and the expressiveness of the temporal model class.

\begin{prop}[Sample consistency of the integrated MMD]\label{prop:time_varying_mmd_infinite_data}
Let $\kappa$ be a continuous positive-definite kernel on the separable Hilbert space $\mathcal{X}$. Let $X_1,\dots,X_n$ be i.i.d.\ copies of a jointly measurable $\mathcal{X}$-valued stochastic process $X\colon [0,T]\to\mathcal{X}$ with marginal laws $P(t)$, and let $P_n(t)$ be the corresponding empirical measure. Let $Q(t)$ be a family of probability measures on $\mathcal{X}$ such that its mean embedding trajectory $t\mapsto \mu_{Q(t)}$ is strongly measurable and belongs to $L^2(0,T; \mathcal{H}_\kappa)$. If $\int_0^T \mathbb{E}_{X\sim P(t)}[\kappa(X,X)]\diff t < \infty$, we have
$$
\frac{1}{T}\int_0^T \operatorname{MMD}_\kappa^2\bigl(P_n(t),Q(t)\bigr)\diff t \xrightarrow{n\to\infty} \frac{1}{T}\int_0^T \operatorname{MMD}_\kappa^2\bigl(P(t),Q(t)\bigr)\diff t \qquad\text{a.s.}
$$
In particular, $\frac{1}{T}\int_0^T \operatorname{MMD}_\kappa^2\bigl(P_n(t),P(t)\bigr)\diff t \to 0$ almost surely.
\end{prop}

\begin{prop}[Projection consistency for temporal mixtures]\label{prop:time_varying_projection_consistency}
Let $Q_\theta(t) = \sum_{k=1}^K \pi_k(t)\nu_k$ be a time-dependent Gaussian mixture where $\pi(t) \in \Delta^{K-1}$ and $t\mapsto \pi_k(t)$ is measurable for each $k$. Let $P_{n,M}(t)$ and $Q_{\theta,M}(t)$ be the projected measures on $\mathbb{R}^M$ defined in \Cref{sec:numerics}, and let $\kappa_M$ denote the coordinate version of the projected kernel. Suppose $\kappa$ is continuous and satisfies the polynomial growth condition \eqref{eq:kernel_bound} for some $q \ge 0$. If the sample paths satisfy $\int_0^T \|X_i(t)\|_{\mathcal{X}}^{2q}\diff t < \infty$ almost surely for $i=1,\dots,n$, then
$$
\frac{1}{T}\int_0^T \operatorname{MMD}_{\kappa_M}^2\bigl(P_{n,M}(t), Q_{\theta,M}(t)\bigr)\diff t \xrightarrow{M\to\infty} \frac{1}{T}\int_0^T \operatorname{MMD}_\kappa^2\bigl(P_n(t), Q_\theta(t)\bigr)\diff t \qquad\text{a.s.}
$$
\end{prop}

Furthermore, for each fixed $t$, if $X(t) \sim Q_\theta(t)$, the projected responsibilities $\gamma_k^{(M)}(\Pi_M X(t), t)$ satisfy the martingale convergence property of Proposition \ref{prop:projection_responsibilities} with weights $\pi(t)$.

\begin{prop}[Density of time-varying Gaussian mixtures]\label{prop:time_varying_density}
Assume $\mathcal{X}\neq\{0\}$, fix $p\ge 1$, and let $P\colon [0,T]\to\mathcal{P}_p(\mathcal{X})$ be continuous with respect to the Wasserstein distance $W_p$. Then, for every $\varepsilon>0$, there exist $K\in\mathbb{N}$, shared Gaussian measures $\nu_1,\dots,\nu_K$ on $\mathcal{X}$  with nonzero covariance operator, and continuous functions $\pi_1,\dots,\pi_K\colon [0,T]\to[0,1]$ with $\sum_{k=1}^K \pi_k(t)=1$, such that
$$
\sup_{t\in[0,T]} W_p\left(P(t),\sum_{k=1}^K \pi_k(t)\nu_k\right) < \varepsilon.
$$
\end{prop}

The proofs of Propositions \ref{prop:time_varying_mmd_infinite_data}, \ref{prop:time_varying_projection_consistency} and \ref{prop:time_varying_density} are provided in \Cref{sec:proof_time_varying_mmd_infinite_data,sec:proof_time_varying_projection_consistency,sec:proof_time_varying_density}, respectively.

\section{Projected maximum likelihood estimation}\label{sec:likelihood}
Although this paper focuses on MMD-based fitting, we briefly describe the projected maximum-likelihood baseline used in \Cref{sec:experiments}; it is the only density-based competitor that is well-defined in the Hilbert-space setting of \eqref{eq:mixture_model}.
As discussed in the introduction, likelihood-based fitting in infinite dimensions requires a common dominating measure.
Here we use the Karhunen--Lo\`eve projector associated with the Gaussian prior
\[
\rho \coloneqq \mu_0=\mathcal N(0,C)
\]
on $\mathcal{X}$.
Let $(\lambda_j,\phi_j)_{j\ge 1}$ be the eigenpairs of $C$ ordered so that $\lambda_1\ge\lambda_2\ge\cdots$, and fix $M$ with $\lambda_M>0$.
Define
\[
\mathcal{X}_M \coloneqq \mathrm{span}\{\phi_1,\dots,\phi_M\},
\qquad
\Pi_M x \coloneqq \sum_{j=1}^M \langle x,\phi_j\rangle_{\mathcal{X}}\,\phi_j,
\qquad
\Pi_M^\perp \coloneqq \Id-\Pi_M.
\]
Under the decomposition $\mathcal{X}=\mathcal{X}_M\oplus \mathcal{X}_M^\perp$, the prior factorizes as
\[
\rho=\rho_M\otimes \rho_M^\perp,
\qquad
\rho_M=\mathcal N(0,\Lambda_M),
\qquad
\Lambda_M\coloneqq \mathrm{diag}(\lambda_1,\dots,\lambda_M),
\]
where $\rho_M^\perp$ is the law of $\Pi_M^\perp X$ for $X\sim \rho$.
Following \cite{luSequentialMonteCarlo2026}, we construct each component by altering only the first $M$ coordinates.
Given projected parameters $\mathbf m_k\in\R^M$ and $\mathbf K_k\in\R^{M\times M}$ with $\mathbf K_k\succ 0$, let
\[
\nu_{k,M}\coloneqq \mathcal N(\mathbf m_k,\mathbf K_k)
\qquad\text{on }\R^M,
\]
and lift it to
\[
\nu_k\coloneqq \nu_{k,M}\otimes \rho_M^\perp
\qquad\text{on }\mathcal{X}.
\]
Equivalently, $\nu_k=\mathcal N(m_k,\mathcal{K}_k)$, where
\[
m_k=\sum_{j=1}^M (\mathbf m_k)_j\,\phi_j,
\]
and $\mathcal{K}_k$ agrees with $C$ on $\mathcal{X}_M^\perp$ while its restriction to $\mathcal{X}_M$ has matrix $\mathbf K_k$ in the basis $\{\phi_1,\dots,\phi_M\}$.
Since $\nu_{k,M}$ and $\rho_M$ are nondegenerate Gaussians on $\R^M$, they are equivalent, and therefore $\nu_k\sim \rho$.
Hence every mixture
\[
Q\coloneqq \sum_{k=1}^K \pi_k \nu_k,
\qquad
\pi_k>0,\qquad \sum_{k=1}^K \pi_k=1,
\]
satisfies $Q\ll\rho$; indeed $Q\sim\rho$ because its density is strictly positive $\rho$-a.s.

For $x\in\mathcal{X}$, write
\[
\mathbf x\coloneqq \bigl(\langle x,\phi_1\rangle_{\mathcal{X}},\dots,\langle x,\phi_M\rangle_{\mathcal{X}}\bigr)^\top
\]
and let
\[
\varphi_M(\mathbf x;\mathbf m,\mathbf K)
\coloneqq
(2\pi)^{-M/2}\det(\mathbf K)^{-1/2}
\exp\!\Big\{-\tfrac12(\mathbf x-\mathbf m)^\top \mathbf K^{-1}(\mathbf x-\mathbf m)\Big\}
\]
denote the usual Gaussian density on $\R^M$.
Then
\[
p_k(x)\coloneqq \frac{\diff \nu_k}{\diff \rho}(x)
=\frac{\diff \nu_{k,M}}{\diff \rho_M}(\mathbf x)
=\frac{\varphi_M(\mathbf x;\mathbf m_k,\mathbf K_k)}{\varphi_M(\mathbf x;0,\Lambda_M)}
\]
and therefore
\[
p_k(x)=
\det(\Lambda_M)^{1/2}\det(\mathbf K_k)^{-1/2}
\exp\!\Big\{
-\tfrac12(\mathbf x-\mathbf m_k)^\top \mathbf K_k^{-1}(\mathbf x-\mathbf m_k)
+\tfrac12\mathbf x^\top \Lambda_M^{-1}\mathbf x
\Big\}.
\]
Accordingly,
\[
q(x)\coloneqq \frac{\diff Q}{\diff \rho}(x)=\sum_{k=1}^K \pi_k\,p_k(x).
\]
Since the factor $\varphi_M(\mathbf x;0,\Lambda_M)^{-1}$ is common to every component, maximizing the $\rho$-log-likelihood is equivalent to maximizing the usual projected mixture log-likelihood
\begin{equation}\label{eq:loglik_projected}
\ell_M(\pi,m,\mathcal{K})
\coloneqq \sum_{i=1}^n \log\!\Big(\sum_{k=1}^K \pi_k\,\varphi_M(\mathbf{x}_i;\mathbf{m}_k,\mathbf{K}_k)\Big),
\end{equation}
because
\[
\sum_{i=1}^n \log q(X_i)
=
\ell_M(\pi,m,\mathcal{K})
-\sum_{i=1}^n \log \varphi_M(\mathbf{x}_i;0,\Lambda_M),
\]
and the second term is independent of the mixture parameters.


The EM algorithm~\cite{dempsterMaximumLikelihoodIncomplete1977} is therefore carried out in the projected coordinates $\mathbf x_i\in\R^M$.
In the E-step,
\[
r_{ik}
=
\frac{\pi_k\,p_k(X_i)}{\sum_{\ell=1}^K \pi_\ell\,p_\ell(X_i)}
=
\frac{\pi_k\,\varphi_M(\mathbf x_i;\mathbf m_k,\mathbf K_k)}
{\sum_{\ell=1}^K \pi_\ell\,\varphi_M(\mathbf x_i;\mathbf m_\ell,\mathbf K_\ell)}.
\]
For the unregularized full-covariance model, the M-step updates are
\[
\pi_k^{\mathrm{new}}=\frac1n\sum_{i=1}^n r_{ik},
\qquad
\mathbf m_k^{\mathrm{new}}
=
\frac{\sum_{i=1}^n r_{ik}\mathbf x_i}{\sum_{i=1}^n r_{ik}},
\]
and
\[
\mathbf K_k^{\mathrm{new}}
=
\frac{\sum_{i=1}^n r_{ik}(\mathbf x_i-\mathbf m_k^{\mathrm{new}})(\mathbf x_i-\mathbf m_k^{\mathrm{new}})^\top}
{\sum_{i=1}^n r_{ik}}.
\]
In practice, one stabilizes the covariance update by replacing $\mathbf K_k^{\mathrm{new}}$ with $\mathbf K_k^{\mathrm{new}}+\varepsilon I_M$ for some $\varepsilon>0$.
This is a regularized EM step rather than the exact M-step of the unpenalized likelihood, but it guarantees $\mathbf K_k\succ 0$ and therefore preserves equivalence with the reference measure after lifting.
In the fixed-covariance benchmark variant used in our experiments (see \Cref{sec:experiments}), the same E-step is used, while only the weights and means are updated and the covariance matrix is held fixed.
After each M-step, the updated parameters are lifted back to $\mathcal{X}$ by setting
\[
m_k^{\mathrm{new}}=\sum_{j=1}^M (\mathbf m_k^{\mathrm{new}})_j\,\phi_j
\]
and defining $\mathcal{K}_k^{\mathrm{new}}$ to agree with $C$ on $\mathcal{X}_M^\perp$ and to have matrix $\mathbf K_k^{\mathrm{new}}$ on $\mathcal{X}_M$.
Algorithm~\ref{alg:projected_mle} summarizes this projected EM routine, with optional ridge regularization controlled by $\varepsilon$.

\begin{algorithm}[t]
\caption{Projected EM for the Gaussian-mixture log-likelihood}
\label{alg:projected_mle}
\begin{algorithmic}[1]
\Require Projected data $\{\mathbf{x}_i\}_{i=1}^n \subset \R^M$; number of components $K$; ridge parameter $\varepsilon \ge 0$; EM iterations $T_{\mathrm{EM}}$
\Ensure Fitted parameters $\hat{\pi},\, \{(\hat{\mathbf{m}}_k,\hat{\mathbf{K}}_k)\}_{k=1}^K$

\State $\{\mathbf{m}_k\}_{k=1}^K \gets \textsc{KMeans}(\{\mathbf{x}_i\}_{i=1}^n, K)$ \Comment{Initialize means}
\State $\mathbf{K}_k \gets \mathbf{I}_M$, \; $\pi_k \gets 1/K$ \Comment{Initialize covariances and weights}

\For{$t = 1,\dots,T_{\mathrm{EM}}$}
    \For{$i = 1,\dots,n$, $k=1,\dots,K$}
        \State $r_{ik} \gets \dfrac{\pi_k\,\varphi_M(\mathbf{x}_i;\mathbf{m}_k,\mathbf{K}_k)}{\sum_{s=1}^K \pi_s\,\varphi_M(\mathbf{x}_i;\mathbf{m}_s,\mathbf{K}_s)}$ \Comment{E-step responsibilities}
    \EndFor
    \For{$k = 1,\dots,K$}
        \State $N_k \gets \sum_{i=1}^n r_{ik}$ \Comment{Effective cluster size}
        \State $\pi_k \gets N_k / n$ \Comment{M-step: weights}
        \State $\mathbf{m}_k \gets \tfrac{1}{N_k}\sum_{i=1}^n r_{ik}\mathbf{x}_i$ \Comment{M-step: means}
        \State $\mathbf{K}_k \gets \tfrac{1}{N_k}\sum_{i=1}^n r_{ik}(\mathbf{x}_i-\mathbf{m}_k)(\mathbf{x}_i-\mathbf{m}_k)^\top + \varepsilon\mathbf{I}_M$ \Comment{M-step: ridge-regularized covariances}
    \EndFor
\EndFor

\State \Return $\hat{\pi} \gets \pi$, \; $\{(\hat{\mathbf{m}}_k,\hat{\mathbf{K}}_k)\}_{k=1}^K \gets \{(\mathbf{m}_k,\mathbf{K}_k)\}_{k=1}^K$
\end{algorithmic}
\end{algorithm}

The next result shows that the projected likelihood ratios are consistent as $M\to\infty$.

\begin{prop}[Consistency of projected likelihood ratios]\label{prop:projection_likelihood}
Let $Q$ be a mixture \eqref{eq:mixture_model} with $Q\ll\rho$, and let $q\coloneqq dQ/d\rho$.
For each $M$, set $\rho_M\coloneqq \rho\circ\Pi_M^{-1}$, $Q_M\coloneqq Q\circ\Pi_M^{-1}$, and define $q_M\coloneqq dQ_M/d\rho_M$.
Let $X$ be an $\mathcal{X}$-valued random variable with distribution $\rho$ and denote
$$X^{(M)}\coloneqq \Pi_M X.$$
Then $q_M(X^{(M)})=\mathbb{E}_{\rho}[q(X)\mid X^{(M)}]$ a.s. and $q_M(X^{(M)})\xrightarrow{M\to\infty} q(X)$ a.s. and in $L^1(\rho)$.
Consequently, for i.i.d.\ $X_1,\dots,X_n\sim Q$,
\[
\sum_{i=1}^n \log q_M(\Pi_M X_i)\longrightarrow \sum_{i=1}^n \log q(X_i)
\qquad Q^n\text{-a.s.}
\]
\end{prop}
The proof is given in \Cref{sec:proof_projection_likelihood}.

\section{Extension to Banach spaces}\label{app:banach_kl}

This appendix outlines how our framework extends to real separable Banach spaces $(\mathcal X, \|\cdot\|_{\mathcal X})$. Let $\mathcal X^*$ be the continuous dual and $\langle x, f \rangle$ the dual pairing. A Borel probability measure $\nu$ on $\mathcal X$ is Gaussian if $\nu \circ f^{-1}$ is Gaussian for all $f \in \mathcal X^*$. If \(\nu\) is centered, its characteristic functional is
\[
\widehat{\nu}(f)
\coloneqq
\int_{\mathcal X} e^{i\langle x,f\rangle}\diff\nu( x)
=
\exp\!\Big(-\tfrac12\,C_\nu(f,f)\Big),
\qquad f\in\mathcal X^*,
\]
where the covariance form \(C_\nu:\mathcal X^*\times\mathcal X^*\to\mathbb R\) is given by
\[
C_\nu(f,g)
\coloneqq
\int_{\mathcal X}\langle x,f\rangle\,\langle x,g\rangle\diff\nu( x).
\]
By Fernique's theorem, $\nu$ has a well-defined covariance operator $R_\nu: \mathcal X^* \to \mathcal X$ given by $R_\nu f = \int_{\mathcal X} \langle x, f \rangle x \diff \nu(x)$. The associated Cameron--Martin space $H(\nu)$ is the Hilbert completion of $R_\nu(\mathcal X^*)$ under the inner product $\langle R_\nu f, R_\nu g \rangle_{H(\nu)} = \langle R_\nu f, g \rangle$.

\paragraph{Karhunen--Lo\`eve decomposition.}
Assume that \(H(\nu)\) is infinite-dimensional.
Following Bay and Croix \cite{bayKarhunenLoeveDecomposition2019}, one constructs vectors \(x_n\in\mathcal X\), dual functionals \(x_n^*\in\mathcal X^*\), and normalized coordinates
\[
h_n\coloneqq \sqrt{\lambda_n}\,x_n \in H(\nu),
\qquad
h_n^*\coloneqq \lambda_n^{-1/2}x_n^* \in \mathcal X^*,
\]
from successive maximizers of the Rayleigh quotient associated with the covariance operators of the residual Gaussian measures; see \cite[Theorem~3.3]{bayKarhunenLoeveDecomposition2019}.
The resulting coordinates satisfy \(R_\nu h_n^*=h_n\), and the random variables \(\langle \cdot,h_n^*\rangle\) are independent \(\mathcal N(0,1)\) under \(\nu\).

\begin{cor}[Banach-space Karhunen--Lo\`eve decomposition {\cite[Corollary~3.8]{bayKarhunenLoeveDecomposition2019}}]\label{cor:kl_banach}
Let \(\nu\) be a centered Gaussian measure on a real separable Banach space \(\mathcal X\), assume \(H(\nu)\) is infinite-dimensional, and let \((\lambda_n,x_n,h_n,h_n^*)_{n\ge 0}\) be constructed as above.
Then
\[
x=\sum_{n=0}^\infty \langle x,h_n^*\rangle\,h_n
\]
converges in \(\mathcal X\) for \(\nu\)-a.e.\ \(x\in\mathcal X\), and the coordinate maps \(x\mapsto \langle x,h_n^*\rangle\) are independent \(\mathcal N(0,1)\) under \(\nu\).
Equivalently, if \((\xi_n)_{n\ge 0}\) are i.i.d.\ \(\mathcal N(0,1)\), then
\[
\sum_{n=0}^\infty \sqrt{\lambda_n}\,\xi_n\,x_n
\]
converges almost surely in \(\mathcal X\) and has law \(\nu\).
\end{cor}

In the Hilbert-space setting of the main text, we identify \(\mathcal X\simeq \mathcal X^*\) via the Riesz representation theorem.
Then \(R_\nu\) becomes the usual covariance operator, and the above construction reduces to the classical spectral Karhunen--Lo\`eve expansion.

\paragraph{Applicability.} The measure-theoretic components of our framework (including Gaussian mixtures, Radon--Nikodym responsibilities, and MMD weak density arguments) extend naturally to separable Banach spaces. While orthogonal projections are unavailable in this broader setting, one can substitute finite-rank continuous projections $P_M:\mathcal X\to\mathcal X$, such as the Bay--Croix projections from Corolary \ref{cor:kl_banach}. However, we do not pursue this extension further, as the closed-form MMD expressions and projection algorithms developed in the main text rely intrinsically on the inner product structure of Hilbert spaces.

\section{Synthetic experiments}\label{app:synthetic_experiments}

We compare against four baseline families: (i)~metric-space methods \cite{schubertFastEagerMedoids2021,murtaghAlgorithmsHierarchicalClustering2012,esterDensitybasedAlgorithmDiscovering1996,campelloHierarchicalDensityEstimates2015,gonzalezClusteringMinimizeMaximum1985}; (ii)~FDA-specific methods~\cite{ramos-carrenoScikitfdaPythonPackage2024}; (iii)~classical Euclidean methods restricted to \(\R^d\); and (iv)~a projected maximum-likelihood baseline inspired by \cite{luSequentialMonteCarlo2026} (see \Cref{sec:likelihood}).
We report Adjusted Rand Index (ARI), train with Adam (\(\mathrm{lr}=0.1\)) and \(k\)-means initialization, and run the experiments on CPU.

\emph{Model class.} Across all domains, we fit the finite-dimensional mixture of \Cref{sec:numerics} after fixing an orthonormal basis \((e_r)_{r=1}^M\) of \(\mathcal{X}_M\). We use cosine bases for \(L^2/H^1\), Wigner-D matrix coefficients for \(\mathrm{SO}(3)\), Laplacian eigenvectors for graph signals, and the canonical basis for \(\R^d\).

We evaluate the MMD-based Gaussian mixture fitting method on several Hilbert-space representations of increasing complexity.
The first five experiments are controlled mixture-recovery tasks with known ground-truth parameters.
The sixth experiment tests system identification for a Gaussian process induced by a linear SDE.
The last two experiments apply the same fitting procedure to real molecular and skeleton-sequence data.
In all cases, observations are represented in a finite-dimensional basis as in \Cref{sec:numerics}, and the model is fitted by minimizing the squared MMD objective with a Gaussian radial kernel.
Final MMD values are reported as within-experiment diagnostics and are not meant to be compared directly across domains, since the coefficient dimension, scaling, and kernel bandwidth differ across settings.


\subsection{Gaussian random variables in \texorpdfstring{$\R^d$}{Rd}}

As a sanity check, we first consider the finite-dimensional case $\mathcal{X}=\mathbb{R}^d$ with the standard inner product, where the model reduces to a classical Gaussian mixture model \cite{bishopPatternRecognitionMachine2006}.

\paragraph{Experiment.}
We generate $n=1500$ samples from a $K=3$ Gaussian mixture on $\R$ with weights $\pi=(0.5,0.3,0.2)$, means $(-3.0,\,0.5,\,4.0)$, and standard deviations $(0.6,\,0.9,\,0.5)$.
We then fit a $K=3$ mixture by minimizing $\mathrm{MMD}^2$.
\Cref{fig:rd_results} shows that the learned density closely matches the empirical histogram and the ground-truth density.

\begin{figure}[tb]
    \centering
    \includegraphics[width=0.8\textwidth]{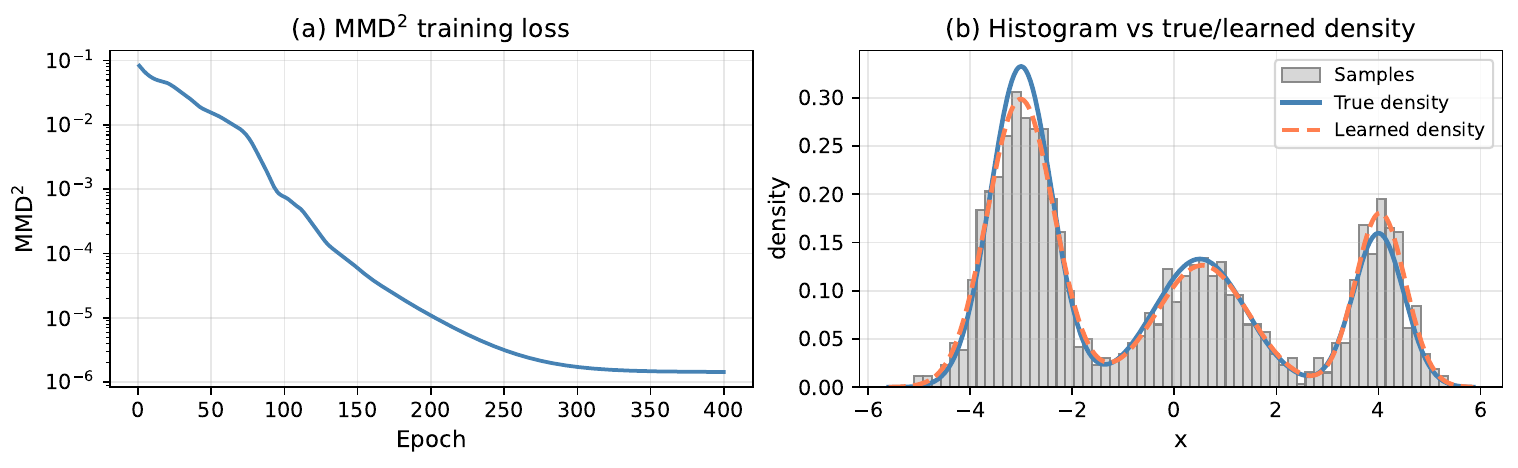}
    \caption{$\R^d$ Gaussian mixture recovery. \textbf{Left:} $\mathrm{MMD}^2$ training loss. \textbf{Right:} Empirical histogram of the samples overlaid with the true (blue) and learned (dashed coral) densities.}
    \label{fig:rd_results}
\end{figure}

\subsection{Functional data in \texorpdfstring{$L^2(0,1)$}{L2} }

We next consider functional data in $\mathcal{X}=L^2(0,1;\R^d)$.
Gaussian random elements in this space can be viewed as Gaussian processes with square-integrable sample paths, so the mixture model specializes to a mixture of Gaussian processes \cite{rasmussenGaussianProcessesMachine2008,trespMixturesGaussianProcesses2000,alvarez-lopezContinuousTemporalLearning2025}.

\paragraph{Experiment.}
We generate multivariate functional data in $\mathcal{X}=L^2(0,1;\R^2)$ from $K=5$ Gaussian components with weights $\pi=(0.30,0.25,0.20,0.15,0.10)$.
Each component has a distinct mean function and diagonal covariance profile in coefficient space.
We sample $n=500$ trajectories and project them onto a cosine basis with $R=15$ functions per spatial dimension, yielding coefficient vectors in $\R^{30}$.
\Cref{fig:l2_k5_results} shows raw trajectories, recovered mean functions, recovered weights, and the training curve.

\begin{figure}[tb]
    \centering
    \includegraphics[width=\textwidth]{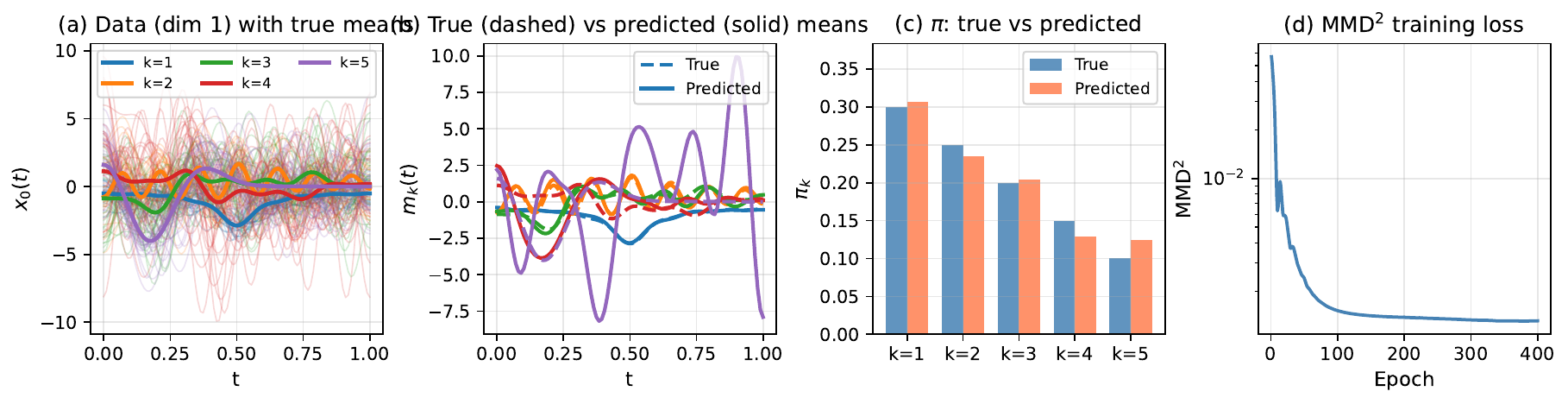}
    \caption{$L^2(0,1;\R^2)$ mixture with $K=5$. \textbf{(a)} Raw trajectories (dimension 1) overlaid with true means. \textbf{(b)} True (dashed) versus predicted (solid) mean functions. \textbf{(c)} True versus predicted mixture weights. \textbf{(d)} $\mathrm{MMD}^2$ training loss.}
    \label{fig:l2_k5_results}
\end{figure}

\subsection{Functional data in \texorpdfstring{$L^2([0,1]^2)$}{L2}}

The same construction extends to functions on two-dimensional domains by using tensor-product bases.
Given an orthonormal basis $(\varphi_r)_{r\ge 0}$ of $L^2(0,1)$, the functions
\[
\Phi_{m,n}(s,t)\coloneqq \varphi_m(s)\varphi_n(t)
\]
form an orthonormal basis of $L^2([0,1]^2)$.
This setting covers image-like signals, spatial random fields, and spatio-temporal processes.

\paragraph{Experiment.}
We consider $\mathcal{X}=L^2([0,1]^2;\R)$ with a tensor-product cosine basis, using $R_s=R_t=8$ and hence $M=64$ coefficients.
We generate $n=400$ samples from a $K=3$ Gaussian mixture with weights $\pi=(0.45,0.35,0.20)$, where each component has a distinct 2D mean surface.
\Cref{fig:l2_2d_results} compares the true and predicted mean surfaces and reports the training loss and recovered weights.

\begin{figure}[tb]
    \centering
    \includegraphics[width=\textwidth]{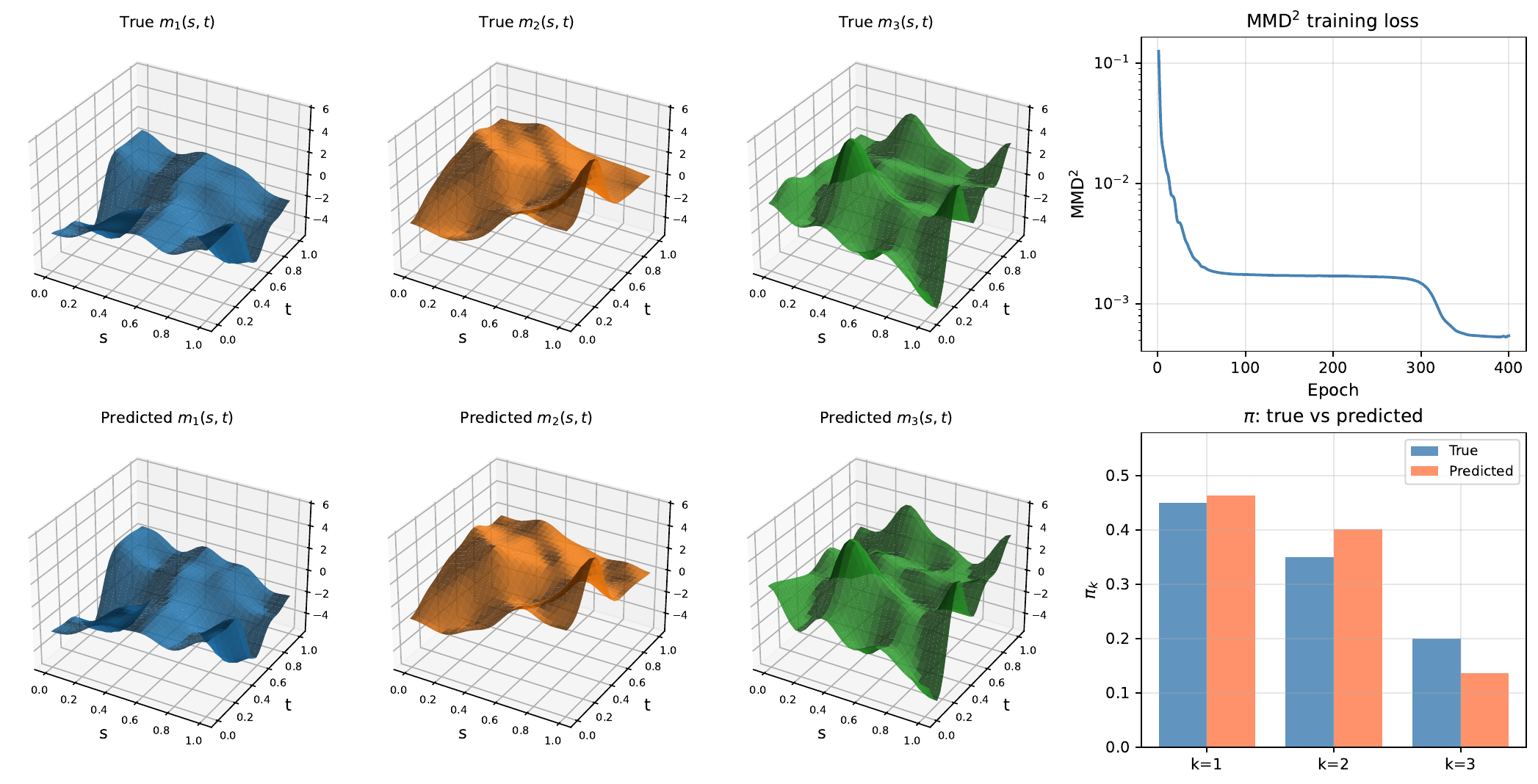}
    \caption{$L^2([0,1]^2)$ mixture with $K=3$. \textbf{Left columns:} True mean surfaces $m_k(s,t)$ (\textbf{top row}) and predicted surfaces (\textbf{bottom row}), aligned by color. \textbf{Right column:} $\mathrm{MMD}^2$ training loss (\textbf{top}) and true versus predicted mixture weights (\textbf{bottom}).}
    \label{fig:l2_2d_results}
\end{figure}

\subsection{\texorpdfstring{$L^2(\mathrm{SO}(3))$}{L2(SO(3))} rotation data}

We also consider non-Euclidean domains equipped with an $L^2$ structure.
For a compact group $G$ with Haar probability measure $\mu$, the space $L^2_\mu(G;\R^d)$ is a Hilbert space with inner product
\[
\langle f,g\rangle_{\mathcal{X}}
\coloneqq \int_G f(x)^\top g(x)\diff\mu(x).
\]
For $G=\mathrm{SO}(3)$, the Peter--Weyl theorem provides an orthonormal basis in terms of Wigner D-matrices.

\paragraph{Experiment.}
We consider rotation data on $\mathrm{SO}(3)$, relevant for orientation estimation in robotics and molecular modeling.
We generate $n=200$ rotations from $K=3$ concentrated components and encode them using real Wigner-basis features up to degree $L_{\max}=3$, giving $M=84$ coefficients.
\Cref{fig:so3_results} displays each rotation through the image of the reference direction $[1,0,0]$ on $S^2$, together with the true and learned component directions, the training loss, and the recovered weights.

\begin{figure}[tb]
    \centering
    \includegraphics[width=\textwidth]{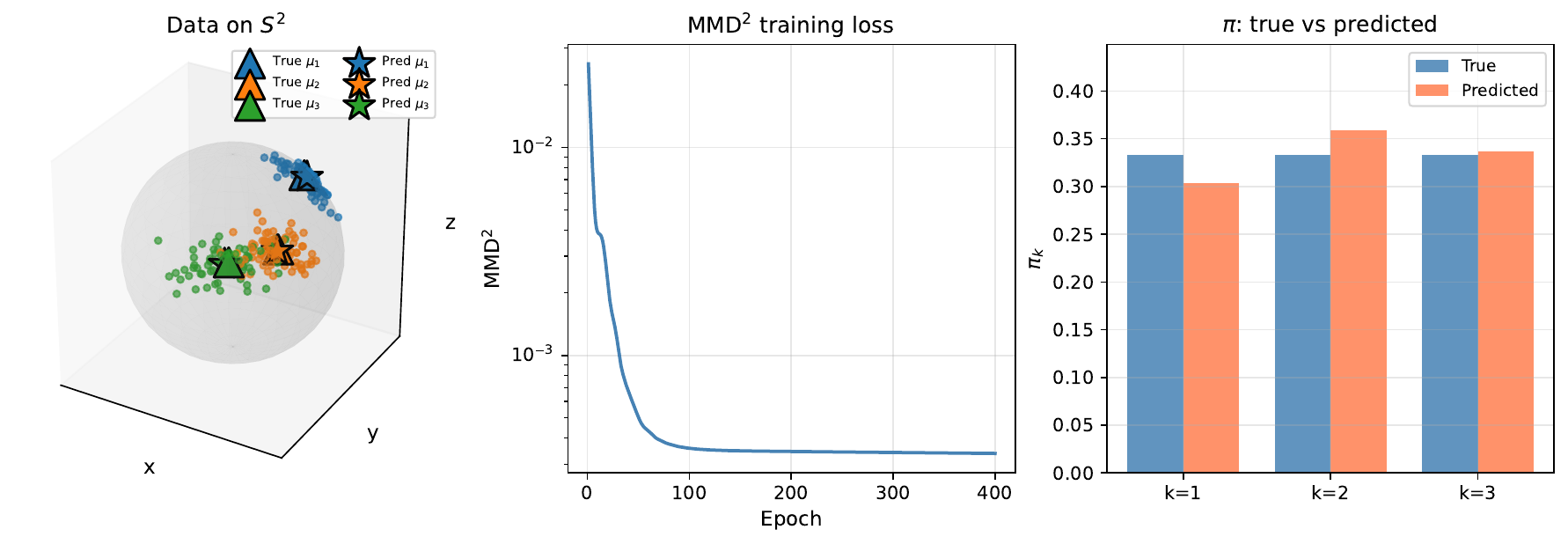}
    \caption{$L^2(\mathrm{SO}(3))$ mixture with $K=3$. \textbf{Left:} Data projected on $S^2$ with true and learned component directions. \textbf{Middle:} $\mathrm{MMD}^2$ training loss. \textbf{Right:} True versus predicted mixture weights.}
    \label{fig:so3_results}
\end{figure}

\subsection{Graph signals}

We next test the method on graph-structured data, where the Hilbert-space geometry is induced by the graph Laplacian.
Let $G=(V,E)$ be a finite weighted graph with Laplacian $L=D-W$.
For $\alpha>0$, we equip $\mathbb{R}^{|V|}$ with the inner product
\[
\langle f,g\rangle_{\mathcal{X}}
\coloneqq f^\top (L+\alpha I)g.
\]
This geometry gives larger norm to high-frequency graph components, while the finite-dimensional projection below restricts the representation to the leading Laplacian modes.

\paragraph{Experiment.}
We generate $n=150$ graph signals on a random Erd\H{o}s--R\'enyi graph with $|V|=30$ nodes and edge probability $0.25$.
The signals are projected onto the first $M=15$ Laplacian eigenvectors, using regularization $\alpha=0.1$.
The mixture has $K=3$ components with uniform weights.
\Cref{fig:graph_results} shows the true and predicted mean graph signals, together with the training loss and recovered weights.

\begin{figure}[tb]
    \centering
    \includegraphics[width=\textwidth]{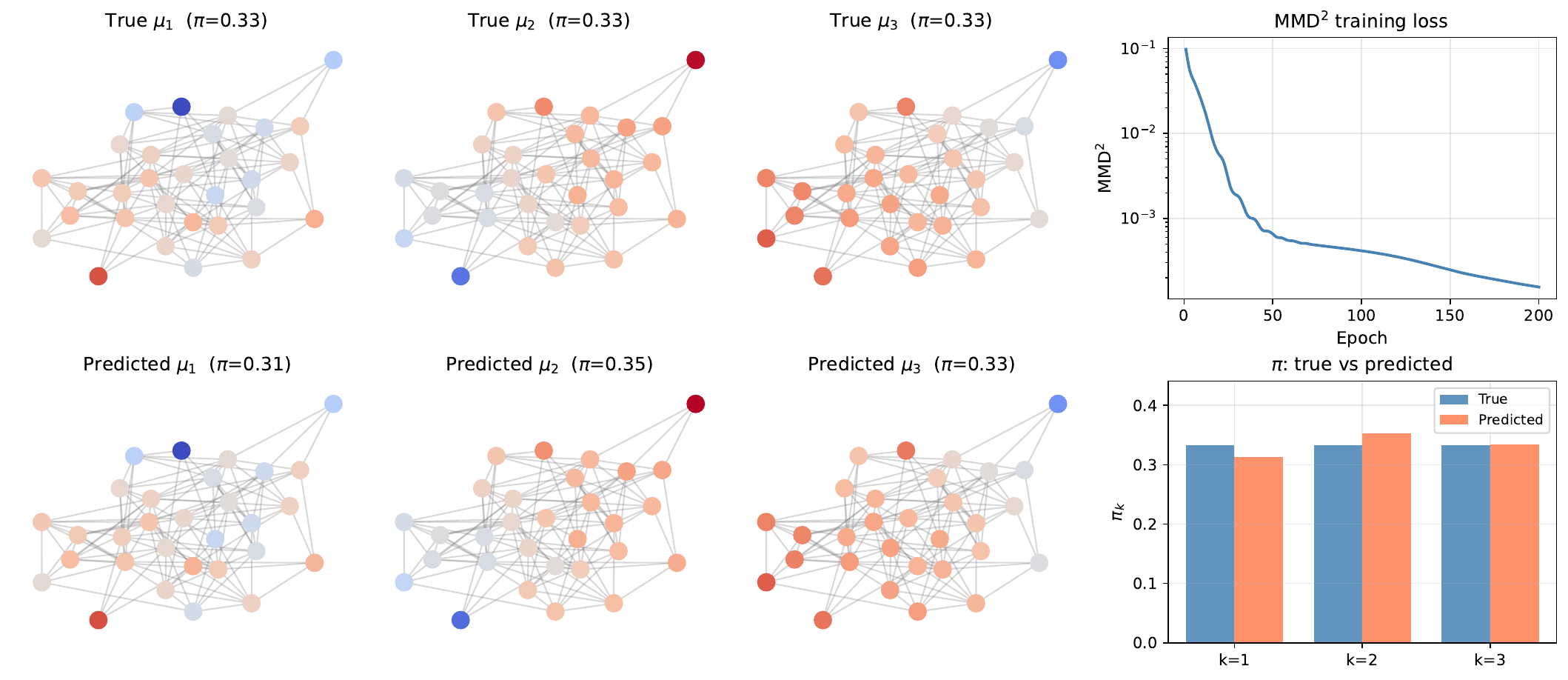}
    \caption{Graph-signal mixture with $K=3$. \textbf{Left block:} True (\textbf{top}) versus predicted (\textbf{bottom}) mean signals per component, plotted on the shared Erd\H{o}s--R\'enyi graph using a common colormap. \textbf{Right column:} $\mathrm{MMD}^2$ training loss (\textbf{top}) and true versus predicted weights (\textbf{bottom}).}
    \label{fig:graph_results}
\end{figure}

\subsection{Linear SDE: system identification}

As a parametric Gaussian-process example, we consider the linear stochastic differential equation
\[
\diff x(t) = \big(Ax(t)+Bu(t)\big)\diff t + G\diff W(t),\qquad t\in [a,b].
\]
For Gaussian initial condition and deterministic input $u$, the solution is a Gaussian process.
Its mean $m(t)$ and covariance $\Sigma(t)=\mathrm{Cov}[x(t),x(t)]$ satisfy
 \[
 \dot m = A m + B u,\qquad m(a)=m_0,
 \qquad
\dot\Sigma = A\Sigma + \Sigma A^\top + G G^\top,\qquad \Sigma(a)=\Sigma_0.
 \]
 The resulting projected mean and covariance can therefore be computed analytically in any chosen $L^2$ basis.

\paragraph{Experiment.}
We treat $(A,B,G)$ as learnable parameters and identify them from $n_s=20$ sampled trajectories of a $4$-dimensional linear SDE over $[0,5]$, driven by the input $u(t)=(\sin t,\cos 2t)$.
The trajectories are projected onto an $L^2$ cosine basis with $R=7$ functions per state dimension, giving $M=28$ coefficients.
We match the empirical coefficient distribution to the single Gaussian induced
by the linear SDE. \Cref{fig:lti_results} shows the sample paths, training loss, recovered mean function, and diagonal of the projected covariance.

\begin{figure}[tb]
    \centering
    \includegraphics[width=\textwidth]{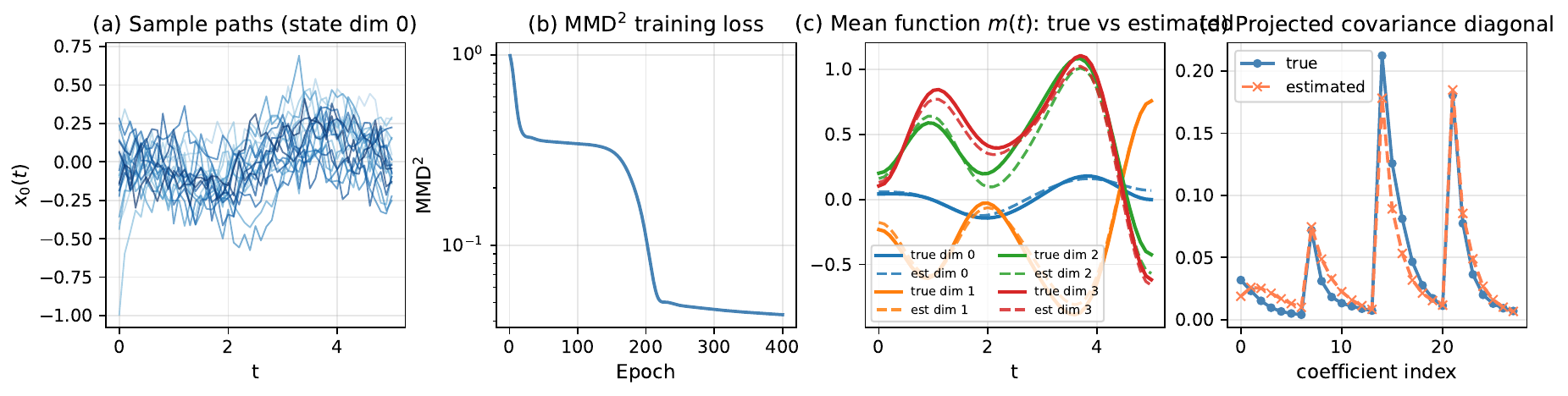}
    \caption{Linear SDE system identification via MMD. \textbf{(a)} Sample paths of the first state dimension. \textbf{(b)} $\mathrm{MMD}^2$ training loss. \textbf{(c)} True versus estimated mean function $m(t)$ for every state dimension. \textbf{(d)} Diagonal of the projected covariance matrix.}
    \label{fig:lti_results}
\end{figure}

\subsection{3D molecular data (QM9)}

We also test the method on real molecular structures using permutation-invariant graph fingerprints derived from 3D atomic coordinates.

\paragraph{Experiment.}
We sample $n=500$ molecules from QM9, build WL-hash fingerprints of dimension $128$ from distance-threshold molecular graphs, and project them onto the first $R=64$ DCT coefficients.
We fit a $K=5$ Gaussian mixture with diagonal covariance.
\Cref{fig:atomic_representatives} shows the three molecules closest to each learned component mean in coefficient space.
The HOMO--LUMO gap and assignment responsibility are reported for each representative molecule.

\begin{figure}[tb]
    \centering
    \includegraphics[width=\textwidth]{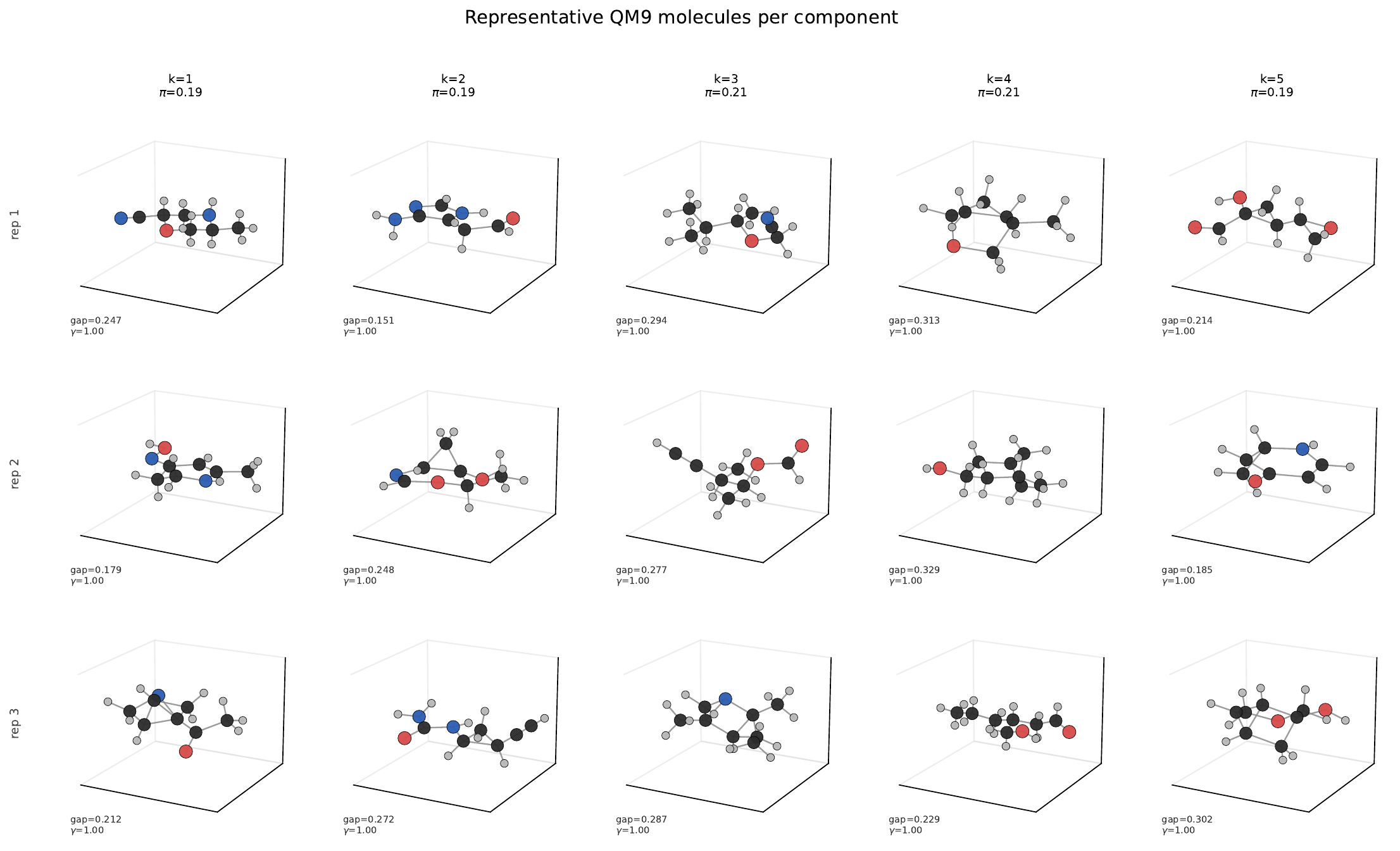}
    \caption{Representative QM9 molecules per component. Columns correspond to learned components $k=1,\ldots,K$, and rows display the three closest molecules to each component mean in coefficient space. Each panel reports the HOMO-LUMO gap and assignment responsibility $\gamma_{i,k}$.}
    \label{fig:atomic_representatives}
\end{figure}

\subsection{NTU RGB+D skeleton sequences}

Finally, we apply the method to human motion data, treating skeleton sequences as $\R^{75}$-valued paths corresponding to $25$ joints with $3$ coordinates each.

\paragraph{Experiment.}
We load $n=150$ action sequences from the NTU RGB+D dataset \cite{shahroudyNTURGB+DLarge2016}, resample each sequence to $64$ frames, subtract the mean trajectory, and project onto an $L^2$ cosine basis with $R=10$.
This gives coefficient vectors of dimension $M=750$.
We fit a $K=5$ Gaussian mixture with diagonal covariance by minimizing $\mathrm{MMD}^2$.
\Cref{fig:ntu_representatives} shows the most representative skeleton sequences per component, selected as those closest to the learned component mean in coefficient space.

\begin{figure}[tb]
    \centering
    \includegraphics[width=0.9\textwidth]{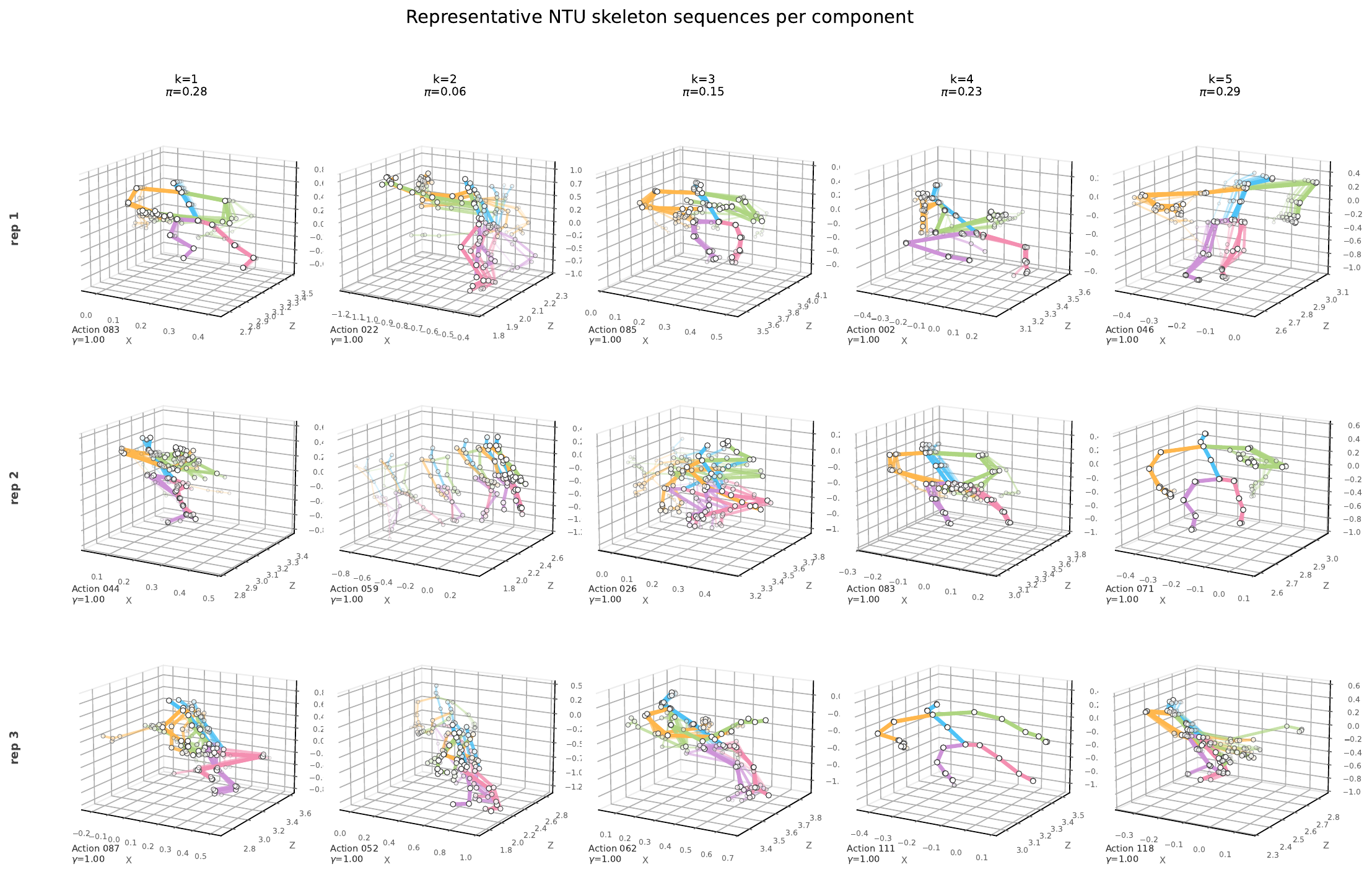}
    \caption{Most representative NTU skeleton sequences per component. Columns correspond to learned components $k=1,\ldots,K$, and rows display the three closest sequences to each component mean. Each cell shows $6$ evenly-spaced snapshots overlaid with increasing opacity, along with the action label and the assignment responsibility $\gamma_{i,k}$.}
    \label{fig:ntu_representatives}
\end{figure}


\subsection{Summary}

\Cref{tab:figure_hyperparams} summarizes the experimental settings and final $\mathrm{MMD}^2$ values.
The controlled experiments show that the method recovers the main mixture structure across Euclidean, functional, group-valued, and graph-based representations.
The real-data experiments illustrate that the same coefficient-space fitting procedure also yields coherent representative components on molecular and skeleton-sequence data.

\section{Experimental details and compute resources}\label{app:experimental_details}

We record the implementation details used for the experiments in \Cref{sec:experiments,app:synthetic_experiments}. Unless stated otherwise, all models are trained on projected coefficient vectors in \(\mathbb{R}^M\), use \texttt{torch.float64}, use fixed random seeds, run on CPU, and initialize mixture means with \(k\)-means++. 

\subsection{Hilbert spaces and bases}\label{app:bases}

Each experiment fits the projected mixture model of \Cref{sec:numerics} after fixing a separable Hilbert space \(\mathcal{X}\) and an orthonormal basis \((e_r)_{r\ge 1}\), which determines the truncated subspace \(\mathcal{X}_M=\mathrm{span}\{e_1,\dots,e_M\}\). We define here the Hilbert spaces and bases used in \Cref{tab:benchmark_hyperparams,tab:figure_hyperparams}, so that the table entries refer to this section instead of repeating the definitions.

\paragraph{(B1) Euclidean canonical basis.} For \(\mathcal{X}=\R^d\) we use the canonical basis \((e_r)_{r=1}^d\), so \(M=d\).

\paragraph{(B2) \(L^2\) cosine basis.} For \(\mathcal{X}=L^2(0,1;\R)\) with \(\langle f,g\rangle=\int_0^1 f(t)g(t)\diff t\), we use the orthonormal cosine basis
\[
e_0(t)=1,\qquad e_r(t)=\sqrt{2}\cos(\pi r t),\quad r\ge 1,
\]
truncated at \(r=R-1\), so \(M=R\). For vector-valued data \(\mathcal{X}=L^2(0,1;\R^d)\), we tensor with the canonical basis of \(\R^d\), so \(M=Rd\).

\paragraph{(B3) Tensor-product cosine basis on \([0,1]^2\).} For \(\mathcal{X}=L^2([0,1]^2;\R)\) we use \(\Phi_{m,n}(s,t)=e_m(s)\,e_n(t)\) with \(0\le m<R_s\), \(0\le n<R_t\), so \(M=R_sR_t\).

\paragraph{(B4) \(H^1\) cosine basis.} For \(\mathcal{X}=H^1(0,1;\R)\) with \(\langle f,g\rangle=\int_0^1 fg + \int_0^1 f'g'\), we use the \(H^1\)-rescaled cosines
\[
\tilde e_r(t)=\frac{e_r(t)}{\sqrt{1+\pi^2 r^2}},\qquad r=0,\dots,R_s-1,
\]
so \(M=R_s\); the \(\tilde e_r\) are orthonormal in \(H^1\).

\paragraph{(B5) Real Wigner--D basis on \(\mathrm{SO}(3)\).} For \(\mathcal{X}=L^2(\mathrm{SO}(3);\R)\) under the Haar inner product, the Peter--Weyl theorem yields an orthonormal basis given by the (real) Wigner D-matrix entries,
\[
e_{\ell,m,n}(R)=\sqrt{2\ell+1}\,D^{\ell}_{m,n}(R),\qquad \ell\ge 0,\ -\ell\le m,n\le \ell.
\]
Truncating at \(\ell\le L_{\max}\) gives \(M=\sum_{\ell=0}^{L_{\max}}(2\ell+1)^2\); for \(L_{\max}=3\), \(M=84\).

\paragraph{(B6) Graph Laplacian eigenbasis.} For a finite weighted graph \(G=(V,E)\) with Laplacian \(L=D-W\) and regularization \(\alpha>0\), we equip \(\mathcal{X}=\R^{|V|}\) with \(\langle f,g\rangle=f^\top(L+\alpha I)g\). Diagonalizing \(L=U\Lambda U^\top\) with eigenpairs \((\lambda_j,u_j)\), the rescaled eigenvectors \(e_j=u_j/\sqrt{\lambda_j+\alpha}\) form an orthonormal basis of \(\mathcal{X}\). We keep the \(M\) eigenvectors associated with the smallest eigenvalues.

\paragraph{(B7) Frobenius basis on \(\mathrm{Sym}(d)\).} For \(\mathcal{X}=\mathrm{Sym}(d)\subset\R^{d\times d}\) with \(\langle A,B\rangle=\mathrm{tr}(A^\top B)\), we use the canonical symmetric basis \(\{E_{ii}\}\cup\{(E_{ij}+E_{ji})/\sqrt{2}\}_{i<j}\) (\texttt{vech} embedding), so \(M=\binom{d+1}{2}\).

\paragraph{(B8) Weisfeiler--Lehman (WL) with kernel PCA / DCT.} For molecular graphs (MUTAG, QM9) we compute WL hash fingerprints in \(\R^D\), then project onto the leading \(M\) kernel-PCA directions (MUTAG) or DCT coefficients (QM9) to obtain coefficient vectors in \(\R^M\). The basis is fixed in coefficient space and inherited from the projection.

\subsection{Experimental setting}\label{app:experimental_setting}

The clustering experiments are unsupervised: all methods are fit without access to labels, and labels are used only after fitting to compute Adjusted Rand Index (ARI). There are no train/test splits because the reported task is clustering or mixture recovery on a fixed sample; variability is instead measured across independently generated datasets, benchmark datasets, random restarts, or subsampled runs, as specified in \Cref{tab:benchmark_hyperparams}.

\begin{table}[t]
    \centering
    \caption{Benchmark hyperparameters and evaluation protocol. Bases (B1)--(B8) are defined in \Cref{app:bases}. All MMD-GMM fits use Adam and \(k\)-means++ initialization; polynomial kernels use degree \(2\) and \(c=1\).}
    \label{tab:benchmark_hyperparams}
    \scriptsize
    \setlength{\tabcolsep}{3pt}
    \begin{tabular}{@{}p{0.10\textwidth}p{0.27\textwidth}p{0.60\textwidth}}
        \toprule
        \textbf{Experiment} & \textbf{Data and Runs} & \textbf{Representation and MMD-GMM training}\\
        \midrule
        \(\mathbb{R}^d\) toy data &
        Five labeled \texttt{sklearn} datasets~\cite{scikit-learn}, \(n=500\), \(d=2\), \(K\in\{2,3\}\); seed \(30\), MMD seed \(42\). &
        \textit{Representation:} canonical basis (B1) on standardized coordinates, \(M=2\). \textit{Training:}
        \(300\) epochs, learning rate \(0.05\), full covariance; Gaussian bandwidth by \(0.5\) times the median pairwise-distance heuristic. \\
        \addlinespace
        \(L^2\) functional data &
        Growth {and Canadian Weather}~\cite{ramsayFunctionalDataAnalysis2005}, Waveform~\cite{breimanClassificationRegressionTrees2017}, Phoneme~\cite{ferratyNonparametricFunctionalData2006}, Kneading~\cite{predaCategoricalFunctionalData2021}, and ECG200~\cite{olszewskiGeneralizedFeatureExtraction2001}; ARI averaged over the {six datasets}; seed \(42\). &
        \textit{Representation:} \(L^2\) cosine basis (B2), \(R=15\) (\(R=10\) for Waveform), coefficient normalization per dataset. \textit{Training:}
        \(400\) epochs, learning rate \(0.1\), diagonal covariance; Gaussian bandwidth by median heuristic. \\
        \addlinespace
        Glucodensity clustering &
        PEDAP CGM cohort~\cite{wadwa2023trial,alvarez-lopezContinuousTemporalLearning2025}, \(98\) patients, \(K=2\); \(5\) MMD runs; seed \(42\). &
        \textit{Representation:} \(H^1\) cosine basis (B4) with \(R_s=8\) on overlapping \(4\)-day sliding windows (stride \(1\) day, \(16\) time bins); patient-level features for our method. \textit{Training:}
        \(600\) epochs, learning rate \(0.01\), diagonal covariance; cosine annealing to \(10^{-4}\); Gaussian bandwidth by median heuristic.\\
        \addlinespace
        \(\mathrm{SO}(3)\) rotations &
        \(10\) generated datasets, \(n=200\), \(K=3\), noise concentration \(8.0\); seeds \(42,\ldots,51\). &
        \textit{Representation:} real Wigner--D basis (B5) with \(L_{\max}=3\), \(M=84\). \textit{Training:}
        \(400\) epochs, learning rate \(0.1\), diagonal covariance; Gaussian bandwidth by median heuristic.  \\
        \addlinespace
        MUTAG molecular graphs &
        MUTAG~\cite{debnathStructureactivityRelationshipMutagenic1991}, \(n=188\), \(K=2\), \(5\) runs; seed \(42\). &
        \textit{Representation:} WL fingerprints with kernel PCA (B8), fingerprint dimension \(512\), radius \(3\), reduced to \(M=32\). \textit{Training:}
        \(500\) epochs, learning rate \(0.05\), \(5\) runs; Gaussian bandwidth selected from \(\{0.1,0.5,1,2,4\}\) plus the median heuristic by a \(60\)-epoch MMD proxy sweep.  \\
        \bottomrule
    \end{tabular}
\end{table}

The standalone mixture-recovery figures and case-study panels in \Cref{sec:experiments,sec:glucodensity,app:synthetic_experiments} that are not part of the unified ARI benchmark follow the configurations summarized in \Cref{tab:figure_hyperparams}.

\begin{table}[t]
    \centering
    \caption{Hyperparameters for the additional mixture-recovery and case-study figures. Bases (B1)--(B8) are defined in \Cref{app:bases}.}
    \label{tab:figure_hyperparams}
    \scriptsize
    \setlength{\tabcolsep}{4pt}
    \begin{tabular}{@{}p{0.20\textwidth}cccccp{0.36\textwidth}@{}}
        \toprule
        \textbf{Experiment} & \textbf{Basis} & \(\boldsymbol{n}\) & \(\boldsymbol{K}\) & \(\boldsymbol{M}\) & \textbf{Epochs} & \textbf{Kernel/optimizer details} \\
        \midrule
        \(1\)D Gaussian mixture & (B1) & \(1500\) & \(3\) & \(1\) & \(400\) & Gaussian \(\sigma=1.0\), Adam learning rate \(0.05\). \\
        \(L^2(0,1;\mathbb{R}^2)\) recovery & (B2) & \(500\) & \(5\) & \(30\) & \(400\) & Cosine basis with \(R=15\) per coordinate; Gaussian \(\sigma=1.2\), Adam learning rate \(0.1\). \\
        \(L^2([0,1]^2)\) recovery & (B3) & \(400\) & \(3\) & \(64\) & \(400\) & Tensor-product cosine \(R_s=R_t=8\); Gaussian \(\sigma=2.0\), Adam learning rate \(0.1\). \\
        \(L^2(\mathrm{SO}(3))\) recovery & (B5) & \(200\) & \(3\) & \(84\) & \(400\) & Real Wigner--D, \(L_{\max}=3\); Gaussian \(\sigma=10.0\), Adam learning rate \(0.05\). \\
        Graph-signal recovery & (B6) & \(150\) & \(3\) & \(15\) & \(200\) & Erd\H{o}s--R\'enyi graph, \(\alpha=0.1\); Gaussian \(\sigma=2.0\), Adam learning rate \(0.1\). \\
        Linear SDE identification & (B2) & \(20\) paths & \(1\) & \(28\) & \(400\) & Cosine basis with \(R=7\) per state dim.; Gaussian bandwidth \(2\,\mathrm{sd}(X)\), Adam learning rate \(10^{-2}\). \\
        QM9 molecules & (B8) & \(500\) & \(5\) & \(64\) & \(150\) & WL fingerprints + first \(64\) DCT coefficients; Gaussian \(\sigma=2.0\), Adam learning rate \(0.05\). \\
        NTU representatives & (B2) & \(150\) & \(5\) & \(750\) & \(250\) & Per-joint cosine basis \(R=10\) on \(\R^{75}\)-valued paths; Gaussian \(\sigma=2.0\), Adam learning rate \(0.05\). \\
        CGM temporal \(H^1\) mixture & (B4) & \(98\) patients & \(2\) & \(8\) & \(400\) & Overlapping \(4\)-day sliding windows with stride \(1\) day, \(16\) time bins, \(H^1\) cosine basis with \(R_s=8\); 2-layer neural ODE weights, hidden width \(64\), RK4, Gaussian \(\sigma=0.8\), Adam learning rate \(0.01\). \\
        CGM correlation mixture & (B7) & \(21{,}889\) windows & \(3\) & \(300\) & \(200\) & Overlapping \(W=14\)-day sliding windows with stride \(1\) day; each window gives a \(24\times 24\) inter-hour correlation matrix with linear shrinkage toward the identity (intensity \(\alpha=0.1\)), \texttt{vech} embedding into \(\mathrm{Sym}(24)\); 2-layer neural ODE logits, hidden width \(64\), RK4, Adam learning rate \(10^{-2}\). \\
        CGM patient graph mixture & (B6) & \(76\) patients & \(3\) & \(304\) & \(200\) & Overlapping \(21\)-day sliding windows with stride \(1\) day, \(10\) time bins; pairwise \(H^1\) similarities use \(R_s=16\) cosine modes; graph Laplacian basis rank \(R=4\); 2-layer neural ODE logits, hidden width \(64\), RK4, Adam learning rate \(10^{-2}\). \\
        \bottomrule
    \end{tabular}
\end{table}

\subsection{Statistical reporting}\label{app:statistical_reporting}

All entries in \Cref{tab:unified_benchmark} are reported as mean ARI with one standard deviation in subscript, computed directly from the list of ARI scores saved by each benchmark script. The standard deviation is a descriptive measure of run-to-run or dataset-to-dataset variability, not a standard error and not a formal confidence interval. The variability source is the labelled dataset collection for the \(\R^d\) and \(L^2\) columns, independent generated datasets for \(\mathrm{SO}(3)\), and multiple runs for glucodensity and MUTAG. Deterministic single-dataset baselines have zero reported standard deviation when only one fitted value is produced.

\subsection{Compute resources}\label{app:compute_resources}

All reported experiments were run on CPU only; no GPU was used or required. The local machine used for reproduction is an Ubuntu 24.04.4 LTS workstation equipped with two Intel Xeon E5-2650 v2 CPUs at \(2.60\,\mathrm{GHz}\) ($16$ physical cores in total), $188$ GiB of RAM, and a network-mounted filesystem. The software environment relies on \texttt{Python} 3.11.5 with the following core libraries: \texttt{PyTorch} 1.13.1, \texttt{NumPy} 1.26.4, \texttt{SciPy} 1.12.0, \texttt{scikit-learn} 1.5.2, \texttt{pandas} 1.5.3, \texttt{scikit-fda} 0.10.1, \texttt{torch-geometric} 2.7.0, \texttt{aeon} 0.11.1, \texttt{rdata} 1.0.0, \texttt{torchdiffeq} 0.2.5, and \texttt{torchsde} 0.2.6. The end-to-end reproduction pipeline runs as a single CPU-worker workflow and typically requires $2$--$4$ wall-clock hours, depending on cached downloads and optional notebook regeneration.

\section{Additional comparisons and diagnostics}\label{app:additional_diagnostics}

This section reports additional diagnostics: the stability of fitted partitions across initializations, the scaling of runtime and memory with the number of components, and the numerical stability of the closed-form expectations.

\subsection{Partition stability across initializations}\label{app:partition_stability}

To measure sensitivity to initialization and the effect of the projection dimension, we refit each model from five independent $k$-means++ initializations at each projection resolution $M\in\{1,5,10\}$ and compare the resulting partitions. The Gaussian bandwidth is the median of all pairwise distances between the standardized coefficient vectors. 

\Cref{tab:partition_stability} reports the mean and standard deviation of the ten pairwise ARIs among the five fitted partitions, rather than ARI against ground-truth labels. The ten pairs share partitions, so their standard deviation describes dispersion and is not a standard error. Agreement generally increases with $M$: the average across datasets rises steadily, and five of the six datasets become more stable at every step.
\begin{table}[tbp]
    \centering
    \caption{Initialization stability: pairwise ARI (mean $\pm$ SD over ten partition pairs from five initializations). $M$ denotes cosine modes per channel. The final row averages the six dataset means at each $M$.}
    \label{tab:partition_stability}
    \small
    \begin{tabular}{lccc}
    \toprule
    Dataset & $M=1$ & $M=5$ & $M=10$ \\
    \midrule
    Growth ($K=2$) & $0.359\pm0.369$ & $0.859\pm0.079$ & $0.974\pm0.021$ \\
    Canadian Weather ($K=4$) & $0.508\pm0.145$ & $0.716\pm0.107$ & $0.731\pm0.189$ \\
    Waveform ($K=3$) & $0.260\pm0.254$ & $0.910\pm0.035$ & $0.598\pm0.168$ \\
    Phoneme ($K=5$) & $0.606\pm0.181$ & $0.671\pm0.140$ & $0.687\pm0.074$ \\
    Kneading ($K=3$) & $0.337\pm0.283$ & $0.467\pm0.152$ & $0.633\pm0.157$ \\
    ECG200 ($K=2$) & $0.137\pm0.201$ & $0.458\pm0.450$ & $0.898\pm0.059$ \\
    \midrule
    Mean & $0.368$ & $0.680$ & $0.754$ \\
    \bottomrule
    \end{tabular}
\end{table}

\subsection{Scaling in the number of components}\label{app:scaling}

We fit balanced, well-separated Gaussian mixtures in $\mathbb R^{20}$ with $n=10^5$ samples and an increasing number of components $K$, using five $k$-means++ initializations per $K$. All fits use the Gaussian kernel with a median-heuristic bandwidth, diagonal covariances, and joint Adam updates in float64 on CPU, omitting the data--data constant. The cross-expectations $\mathbf J$ are evaluated as matrix products, which avoids forming an $n\times K\times M$ tensor and keeps memory at $O(nM+nK+K^2M)$.

\begin{table}[htbp]
\centering
\caption{Component sweep at $n=10^5$, $M=20$: ARI over five initializations, median time per fit (initialization and optimization), and peak memory above the pre-data baseline.}
\label{tab:large_k}
\small
\setlength{\tabcolsep}{4pt}
\resizebox{\textwidth}{!}{%
\begin{tabular}{lrrrrrr}
\toprule
Components $K$ & 5 & 10 & 20 & 50 & 100 & 200 \\
\midrule
ARI (mean $\pm$ SD) & $1.000\pm0.000$ & $1.000\pm0.000$ & $0.977\pm0.028$ & $0.987\pm0.017$ & $0.979\pm0.008$ & $0.976\pm0.004$ \\
Time/fit (s) & 9.04 & 15.06 & 24.18 & 85.04 & 176.01 & 239.74 \\
Memory (MiB) & 185 & 188 & 355 & 295 & 568 & 983 \\
\bottomrule
\end{tabular}}
\end{table}

Recovery remains near perfect across the sweep, while runtime and memory grow moderately with $K$ and stay well within the capacity of a single workstation. These measurements show feasibility under controlled separation, not recovery guarantees for arbitrary large mixtures, and absolute timings depend on the execution environment.

\subsection{Numerical stability}\label{app:numerical_stability}

For the isotropic Gaussian kernel, the full-covariance expectations involve
\[
 A_k=I_M+\mathbf K_{k,M}/\sigma_\kappa^2,
 \qquad
 A_{k,s}=I_M+(\mathbf K_{k,M}+\mathbf K_{s,M})/\sigma_\kappa^2.
\]
Positive semidefiniteness implies $\lambda_{\min}(A_k)\geq1$ and
\[
 \operatorname{cond}_2(A_k)
 =\frac{1+\lambda_{\max}(\mathbf K_{k,M})/\sigma_\kappa^2}
 {1+\lambda_{\min}(\mathbf K_{k,M})/\sigma_\kappa^2}
 \leq 1+\frac{\operatorname{tr}(\mathbf K_{k,M})}{\sigma_\kappa^2},
\]
with the analogous bound for $A_{k,s}$. The identity shift prevents singularity, but does not guarantee good conditioning at small bandwidths. The spectral formulas in \Cref{app:explicit_formulas_rm} can be evaluated by eigendecomposition; the code instead uses Cholesky log-determinants and solves, without explicit inverses. These matrices are $M\times M$, not $n\times n$. Diagonal covariances require only elementwise operations in the optimization loop.


\section{Proofs}\label{sec:proofs}

\subsection{Proof of Proposition \ref{prop:projection_consistency}}\label{sec:proof_projection_consistency}

\begin{proof}
Since $(e_j)_{j\ge1}$ is an orthonormal basis of $\mathcal X$, the orthogonal projections $\Pi_M$ converge strongly to the identity. That is, for every $x\in\mathcal X$, $\Pi_M x \to x$ in $\mathcal X$ as $M\to\infty$.

We first prove the almost sure convergence of $J^{(M)}_{i,k}$. Fix $i\in\{1,\dots,n\}$ and $k\in\{1,\dots,K\}$. Fix $\omega$ such that $x_i\coloneqq X_i(\omega)\in\mathcal X$. Then $\Pi_M x_i \to x_i$ in $\mathcal X$. By the continuity of $\kappa$, for every $y\in\mathcal X$, we have $\kappa(\Pi_M x_i,\Pi_M y)\to \kappa(x_i,y)$. Furthermore, using the growth assumption \eqref{eq:kernel_bound} and the contraction property of the orthogonal projector ($\|\Pi_M x\|_{\mathcal X} \le \|x\|_{\mathcal X}$), we can bound the kernel evaluation as:
\[
    |\kappa(\Pi_M x_i,\Pi_M y)|
    \le
    C\bigl(1+\|\Pi_M x_i\|_{\mathcal X}^q\bigr)\bigl(1+\|\Pi_M y\|_{\mathcal X}^q\bigr)
    \le
    C\bigl(1+\|x_i\|_{\mathcal X}^q\bigr)\bigl(1+\|y\|_{\mathcal X}^q\bigr).
\]
Since Gaussian measures on separable Hilbert spaces have finite moments of all orders by Fernique's theorem, the right-hand side is $\nu_k$-integrable with respect to $y$ for our fixed $x_i$. Therefore, by the dominated convergence theorem,
\[
    J^{(M)}_{i,k}(\omega)
    =
    \int_{\mathcal X}\kappa(\Pi_M x_i,\Pi_M y)\diff \nu_k(y)
    \to
    \int_{\mathcal X}\kappa(x_i,y)\diff \nu_k(y)
    =
    J_{i,k}(\omega),
\]
which establishes $J^{(M)}_{i,k}\to J_{i,k}$ almost surely.

Next, we establish the convergence of $I^{(M)}_{k,s}$. Fix $k,s\in\{1,\dots,K\}$. By the continuity of $\kappa$, $\kappa(\Pi_M y,\Pi_M y')\to \kappa(y,y')$ for all $y,y'\in\mathcal X$. Similarly to the previous case, applying \eqref{eq:kernel_bound} and the contraction property yields:
\[
    |\kappa(\Pi_M y,\Pi_M y')|
    \le
    C\bigl(1+\|y\|_{\mathcal X}^q\bigr)\bigl(1+\|y'\|_{\mathcal X}^q\bigr).
\]
By Fernique's theorem, this dominating function is integrable with respect to the product measure $\nu_k\otimes\nu_s$. Invoking the dominated convergence theorem once more, we obtain
\[
    I^{(M)}_{k,s}
    =
    \int_{\mathcal X\times\mathcal X}\kappa(\Pi_M y,\Pi_M y')\diff (\nu_k\otimes\nu_s)(y,y')
    \to
    \int_{\mathcal X\times\mathcal X}\kappa(y,y')\diff (\nu_k\otimes\nu_s)(y,y')
    =
    I_{k,s}.
\]

Finally, since $\kappa(\Pi_M X_i,\Pi_M X_j)\to \kappa(X_i,X_j)$ a.s.\ for all $i,j$, and the sums in \eqref{eq:projected_mmd} are finite with weights $\pi_k$ independent of $M$, combining these results yields \eqref{eq:proj.conv.as}.
\end{proof}

\subsection{Proof of Proposition \ref{prop:projection_responsibilities}}\label{sec:proof_projection_responsibilities}

\begin{proof}
Fix $k\in\{1,\dots,K\}$. We first show that the exact responsibility $\gamma_k(X)$ is a version of the conditional expectation $\mathbb{E}[\mathbf{1}_{\{Z=k\}} \mid \sigma(X)]$. Let $Q_\theta=\sum_{s=1}^K \pi_s\nu_s$ be the marginal law of $X$. For any Borel set $A\in\mathcal{B}(\mathcal{X})$, the definition of $\gamma_k = \diff(\pi_k\nu_k)/\diff Q_\theta$ yields
\[
    \int_A \gamma_k(x)\diff Q_\theta(x) = \pi_k\nu_k(A) = \mathbb{P}(X\in A,\ Z=k).
\]
Equivalently, for every $A\in\mathcal{B}(\mathcal{X})$,
\[
    \mathbb{E}\bigl[\gamma_k(X)\mathbf{1}_{\{X\in A\}}\bigr] = \mathbb{E}\bigl[\mathbf{1}_{\{Z=k\}}\mathbf{1}_{\{X\in A\}}\bigr].
\]
Since $\gamma_k(X)$ is $\sigma(X)$-measurable, this identity implies $\gamma_k(X) = \mathbb{E}[\mathbf{1}_{\{Z=k\}} \mid \sigma(X)]$ almost surely.

By an identical argument on the projected space, equipping $\mathcal{X}_M$ with the pushforward measures $\nu_{s,M}$ and $Q_{\theta,M}$, the projected derivative $\gamma_k^{(M)} = \diff(\pi_k\nu_{k,M})/\diff Q_{\theta,M}$ satisfies
\[
    \int_B \gamma_k^{(M)}(u)\diff Q_{\theta,M}(u) = \mathbb{P}(X^{(M)}\in B,\ Z=k)
\]
for any $B\in\mathcal{B}(\mathcal{X}_M)$. Equivalently, for every $B\in\mathcal{B}(\mathcal{X}_M)$,
\[
    \mathbb{E}\bigl[\gamma_k^{(M)}(X^{(M)})\,\mathbf{1}_{\{X^{(M)}\in B\}}\bigr] = \mathbb{E}\bigl[\mathbf{1}_{\{Z=k\}}\,\mathbf{1}_{\{X^{(M)}\in B\}}\bigr].
\]
Since $\gamma_k^{(M)}(X^{(M)})$ is $\sigma(X^{(M)})$-measurable, this confirms that $$\gamma_k^{(M)}(X^{(M)}) = \mathbb{E}[\mathbf{1}_{\{Z=k\}} \mid \sigma(X^{(M)})]$$ almost surely, establishing \eqref{eq:martingale_property}.

Next, define the filtration $\mathcal{F}_M \coloneqq \sigma(X^{(M)})$. The property $\Pi_M = \Pi_M \Pi_{M+1}$ ensures that $X^{(M)} = \Pi_M X^{(M+1)}$, so $X^{(M)}$ is $\sigma(X^{(M+1)})$-measurable, meaning the filtration is increasing ($\mathcal{F}_M \subseteq \mathcal{F}_{M+1}$). Let $\mathcal{F}_\infty \coloneqq \sigma(\cup_{M=1}^\infty \mathcal{F}_M)$. Since each projector $\Pi_M$ is continuous, $X^{(M)}$ is $\sigma(X)$-measurable, yielding $\mathcal{F}_\infty \subseteq \sigma(X)$. Conversely, since $\Pi_M x \to x$, we have $X^{(M)}(\omega) \to X(\omega)$ pointwise in $\mathcal{X}$. Since $\mathcal{X}$ is a metric space, the pointwise limit of $\mathcal{F}_\infty$-measurable maps into $\mathcal{X}$ is $\mathcal{F}_\infty$-measurable. This forces $\sigma(X) \subseteq \mathcal{F}_\infty$, concluding that $\mathcal{F}_\infty = \sigma(X)$.

Finally, we apply Lévy's upward convergence theorem. Since the target $\mathbf{1}_{\{Z=k\}}$ is integrable, the martingale sequence $\mathbb{E}[\mathbf{1}_{\{Z=k\}} \mid \mathcal{F}_M]$ converges almost surely and in $L^1$ to $\mathbb{E}[\mathbf{1}_{\{Z=k\}} \mid \mathcal{F}_\infty]$. By our previous identities,
\[
    \mathbb{E}[\mathbf{1}_{\{Z=k\}}\mid \mathcal{F}_\infty] = \mathbb{E}[\mathbf{1}_{\{Z=k\}}\mid \sigma(X)] = \gamma_k(X) \qquad\text{a.s.}
\]
Substituting these into the limit, we obtain, as $M\to\infty$,
\[
    \gamma_k^{(M)}(X^{(M)}) \to \gamma_k(X) \qquad\text{a.s.\ and in }L^1.
\]
\end{proof}

\subsection{Proof of Proposition \ref{prop:approximation}}\label{sec:proof_density}

\begin{proof}
\emph{Proof of (i).}

Fix $p\ge 1$, $P\in\mathcal P_p(\mathcal X)$, and $\varepsilon>0$. Let $\mathcal K_0$ be any nonzero trace-class covariance operator on $\mathcal X$; for instance, $\mathcal K_0=e\otimes e$ for a unit vector $e\in\mathcal X$, which exists because $\mathcal X\neq\{0\}$. It is enough to prove density for mixtures with component covariances of the form $\sigma^2\mathcal K_0$, since these form a subclass of the mixtures in \eqref{eq:mixture_model}.

Because $\mathcal X$ is a separable Hilbert space, it is Polish. Hence finitely supported probability measures are dense in $\mathcal P_p(\mathcal X)$ for $W_p$. Therefore, there exists $\nu=\sum_{i=1}^N w_i\delta_{x_i}$ such that $W_p(P,\nu)<\varepsilon/2.$

For $\sigma>0$, define the Gaussian smoothing
$$
\nu_\sigma \coloneqq \sum_{i=1}^N w_i\,\mathcal N(x_i,\sigma^2\mathcal K_0).
$$
Let $I$ be an index with $\mathbb P(I=i)=w_i$, and let $Z\sim\mathcal N(0,\mathcal K_0)$, independent of $I$. Then $X=x_I$ has law $\nu$, while $Y=x_I+\sigma Z$ has law $\nu_\sigma$. This coupling gives
$$
W_p^p(\nu,\nu_\sigma) \le \mathbb E\|X-Y\|_{\mathcal X}^p = \sigma^p\mathbb E\|Z\|_{\mathcal X}^p.
$$
Since $\mathcal K_0$ is trace-class, Fernique's theorem implies $\mathbb E\|Z\|_{\mathcal X}^p<\infty$. Hence $W_p(\nu,\nu_\sigma)\to0$ as $\sigma\downarrow0$. Choosing $\sigma$ small enough yields $W_p(\nu,\nu_\sigma)<\varepsilon/2$, and the triangle inequality gives
$$
W_p(P,\nu_\sigma)<\varepsilon.
$$
Thus, finite Gaussian mixtures with nonzero covariance operators are $W_p$-dense in $\mathcal P_p(\mathcal X)$.

\medskip
\emph{Proof of (ii).}
Fix $P\in\mathcal{P}(\mathcal{X})$ and draw $K$ i.i.d. samples $X_1,\dots,X_K \sim P$. Let the empirical measure be $\widehat{P}_K \coloneqq \frac{1}{K}\sum_{k=1}^K\delta_{X_k}$. Since $\mathbb{E}[\kappa(X_k,\cdot)]=\mu_P$ and we assume $\sup_x\kappa(x,x)\le 1$, the expected squared MMD is bounded by the variance of the mean embedding:
$$
\mathbb{E}\big[\operatorname{MMD}_\kappa^2(P,\widehat{P}_K)\big] = \frac{1}{K} \mathbb{E}\big[\|\kappa(X_1,\cdot)-\mu_P\|_{\mathcal{H}_\kappa}^2\big] \le \frac{1}{K}\mathbb{E}[\kappa(X_1,X_1)] \le \frac{1}{K}.
$$
By the probabilistic method, there exists a specific realization $m_1,\dots,m_K\in\mathcal{X}$ such that the deterministic empirical measure $\mu_K \coloneqq \frac{1}{K}\sum_{k=1}^K\delta_{m_k}$ satisfies $\operatorname{MMD}_\kappa(P,\mu_K) \le 1/\sqrt{K}$.

Fix any nonzero trace-class covariance operator $\mathcal{K}_0$ (such as $e\otimes e$ for a unit vector $e$, which exists because $\mathcal X\neq\{0\}$) and define $Q_{K,\sigma} \coloneqq \frac{1}{K}\sum_{k=1}^K\mathcal{N}(m_k,\sigma^2\mathcal{K}_0)$. Then $Q_{K,\sigma}$ converges weakly to $\mu_K$ as $\sigma\downarrow 0$. Because $\kappa$ is bounded and continuous, weak convergence implies convergence in MMD. Thus, we can choose $\sigma>0$ small enough such that $\operatorname{MMD}_\kappa(Q_{K,\sigma},\mu_K) \le 1/\sqrt{K}$.

Applying the triangle inequality yields
$$
\operatorname{MMD}_\kappa(P,Q_{K,\sigma}) \le \operatorname{MMD}_\kappa(P,\mu_K) + \operatorname{MMD}_\kappa(\mu_K,Q_{K,\sigma}) \le \frac{2}{\sqrt{K}},
$$
which completes the proof, as $Q_{K,\sigma}$ is a valid mixture of the form \eqref{eq:mixture_model}.
\end{proof}

\subsection{Proof of Proposition \ref{prop:time_varying_mmd_infinite_data}}\label{sec:proof_time_varying_mmd_infinite_data}
\begin{proof}
\textbf{Step 1.}
Let $\Phi(x) = \kappa(x,\cdot) \in \mathcal{H}_\kappa$ be the canonical feature map. Because $\mathcal{X}$ is separable and $\kappa$ is continuous, the reproducing kernel Hilbert space $\mathcal{H}_\kappa$ is separable. Define the Bochner space
$$
\mathbb{H} \coloneqq L^2(0,T; \mathcal{H}_\kappa).
$$
For each sample path, define the random element $Y_i(t) \coloneqq \Phi(X_i(t))$. Since $\Phi$ is continuous, $(\omega,t)\mapsto \Phi(X_i(t,\omega))$ is jointly measurable. Together with the separability of $\mathcal{H}_\kappa$ and the integrability estimate below, this defines a strongly measurable $\mathbb{H}$-valued random variable. By assumption, its squared norm is integrable:
$$
\mathbb{E}[\|Y_1\|_{\mathbb{H}}^2] = \mathbb{E}\!\left[\int_0^T \|\Phi(X_1(t))\|_{\mathcal{H}_\kappa}^2\diff t\right] = \mathbb{E}\!\left[\int_0^T \kappa(X_1(t),X_1(t))\diff t\right] < \infty.
$$
Thus, $Y_1, \dots, Y_n$ are i.i.d.\ $\mathbb{H}$-valued random variables with finite first and second moments.

\textbf{Step 2.}
By Jensen's inequality and Tonelli's theorem, the population mean embedding trajectory $\mu_P \coloneqq (\mu_{P(t)})_{t \in [0,T]}$ satisfies $\int_0^T \|\mu_{P(t)}\|^2_{\mathcal{H}_\kappa} \diff t \le \mathbb{E}[\|Y_1\|_{\mathbb{H}}^2] < \infty$, so $\mu_P \in \mathbb{H}$. By the Bochner--Fubini theorem, $\mu_P$ is exactly the expected value $\mathbb{E}[Y_1]$ in $\mathbb{H}$. 
Similarly, the empirical mean embedding trajectory $\mu_{P_n}$ corresponds to the sample average:
$$
\mu_{P_n} = \frac{1}{n}\sum_{i=1}^n Y_i \quad \text{in } \mathbb{H}.
$$
By assumption, the target embedding trajectory $\mu_Q \coloneqq (\mu_{Q(t)})_{t \in [0,T]}$ also belongs to $\mathbb{H}$.

\textbf{Step 3.}
By the strong law of large numbers for Hilbert-space-valued random variables, the sample average converges to its expectation almost surely:
$$
\mu_{P_n} \xrightarrow{n\to\infty} \mu_P \qquad \text{in } \mathbb{H} \text{ a.s.}
$$
Since $\mu_{P_n} - \mu_Q \to \mu_P - \mu_Q$ in $\mathbb{H}$, the continuity of the squared norm gives $\|\mu_{P_n} - \mu_Q\|_{\mathbb{H}}^2 \to \|\mu_P - \mu_Q\|_{\mathbb{H}}^2$. Because the integrated squared MMD is precisely the squared distance in $\mathbb{H}$ scaled by $1/T$, this immediately implies that
$$
\frac{1}{T}\int_0^T \operatorname{MMD}_\kappa^2\bigl(P_n(t),Q(t)\bigr)\diff t \xrightarrow{n\to\infty} \frac{1}{T}\int_0^T \operatorname{MMD}_\kappa^2\bigl(P(t),Q(t)\bigr)\diff t
$$
almost surely. Finally, setting $Q(t) = P(t)$ directly yields $\frac{1}{T}\|\mu_{P_n} - \mu_P\|_{\mathbb{H}}^2 \to 0$ almost surely, completing the proof.
\end{proof}

\subsection{Proof of Proposition \ref{prop:time_varying_projection_consistency}}\label{sec:proof_time_varying_projection_consistency}
\begin{proof}
For each $M$, the identification of $\mathcal{X}_M$ with $\mathbb{R}^M$ via the basis $(e_r)_{r=1}^M$ maps the projected measures $P_{n,M}(t) \coloneqq \frac{1}{n}\sum_{i=1}^n \delta_{\Pi_M X_i(t)}$ and $Q_{\theta,M}(t) \coloneqq \sum_{k=1}^K \pi_k(t)\nu_{k,M}$ to their finite-dimensional coordinate representations $\mathbf{P}_{n,M}(t)$ and $\mathbf{Q}_{\theta,M}(t)$. By defining $\kappa_M(\mathbf{x},\mathbf{y}) \coloneqq \kappa(\sum_{r=1}^M x_r e_r, \sum_{r=1}^M y_r e_r)$, we have
$$
\operatorname{MMD}_{\kappa_M}^2\bigl(\mathbf{P}_{n,M}(t),\mathbf{Q}_{\theta,M}(t)\bigr) = \operatorname{MMD}_\kappa^2\bigl(P_{n,M}(t),Q_{\theta,M}(t)\bigr)
$$
for every $t$. Expanding the exact and projected MMD objectives, we define the cross-expectations
$$
J_{i,k}^{(M)}(t) \coloneqq \mathbb{E}_{Y\sim \nu_k}\big[\kappa(\Pi_M X_i(t),\Pi_M Y)\big], \qquad J_{i,k}(t) \coloneqq \mathbb{E}_{Y\sim \nu_k}\big[\kappa(X_i(t),Y)\big],
$$
$$
I_{k,s}^{(M)} \coloneqq \mathbb{E}_{Y\sim \nu_k,\,Y'\sim \nu_s}[\kappa(\Pi_M Y,\Pi_M Y')], \qquad I_{k,s} \coloneqq \mathbb{E}_{Y\sim \nu_k,\,Y'\sim \nu_s}[\kappa(Y,Y')].
$$

Fix a sample path satisfying $\int_0^T \|X_i(t)\|_{\mathcal{X}}^{2q}\diff t < \infty$ for every $i=1,\dots,n$. For every $t$, $\Pi_M X_i(t) \to X_i(t)$ in $\mathcal{X}$. Hence, by the continuity of $\kappa$, the data--data interactions converge pointwise in $t$:
$$
\kappa(\Pi_M X_i(t),\Pi_M X_j(t)) \to \kappa(X_i(t),X_j(t)).
$$
For fixed $t$, the integrand in $J_{i,k}^{(M)}(t)$ satisfies $\kappa(\Pi_M X_i(t),\Pi_M Y) \to \kappa(X_i(t),Y)$ for every $Y \in \mathcal{X}$, and is bounded by $|\kappa(\Pi_M X_i(t),\Pi_M Y)| \le C(1+\|X_i(t)\|_{\mathcal{X}}^q)(1+\|Y\|_{\mathcal{X}}^q)$. Then dominated convergence applies because Gaussian measures on Hilbert spaces have finite $q$-moments, giving $J_{i,k}^{(M)}(t) \to J_{i,k}(t)$. The exact same argument applies to the component--component term, yielding $I_{k,s}^{(M)} \to I_{k,s}$.

We next justify passing the limit through the time integral using the Dominated Convergence Theorem. By the polynomial growth condition on $\kappa$ and the contraction property of $\Pi_M$,
$$
|\kappa(\Pi_M X_i(t),\Pi_M X_j(t))| \le C\bigl(1+\|X_i(t)\|_{\mathcal{X}}^q\bigr)\bigl(1+\|X_j(t)\|_{\mathcal{X}}^q\bigr).
$$
This upper bound belongs to $L^1(0,T)$ by Cauchy--Schwarz and the sample path assumption. Likewise, if $Y_k \sim \nu_k$,
$$
|J_{i,k}^{(M)}(t)| \le C\bigl(1+\|X_i(t)\|_{\mathcal{X}}^q\bigr)\,\mathbb{E}\bigl[1+\|Y_k\|_{\mathcal{X}}^q\bigr],
$$
which is integrable on $[0,T]$. Finally, the component--component term is uniformly bounded by $C\,\mathbb{E}[1+\|Y_k\|_{\mathcal{X}}^q]\,\mathbb{E}[1+\|Y_s\|_{\mathcal{X}}^q]$, which is finite and independent of $M$.

Since $0 \le \pi_k(t)\pi_s(t) \le 1$, we can apply the Dominated Convergence Theorem to integrate term by term. Integrating the finite-dimensional MMD expansion
$$
\begin{aligned}
\operatorname{MMD}_\kappa^2\bigl(P_{n,M}(t),Q_{\theta,M}(t)\bigr) &= \frac{1}{n^2}\sum_{i,j=1}^n \kappa(\Pi_M X_i(t),\Pi_M X_j(t)) \\
&\quad - \frac{2}{n}\sum_{i=1}^n\sum_{k=1}^K \pi_k(t)J_{i,k}^{(M)}(t) + \sum_{k,s=1}^K \pi_k(t)\pi_s(t)I_{k,s}^{(M)}
\end{aligned}
$$
over $[0,T]$ yields convergence for fixed $n, K$, and $\theta$:
$$
\int_0^T \operatorname{MMD}_\kappa^2\bigl(P_{n,M}(t),Q_{\theta,M}(t)\bigr)\diff t \xrightarrow{M\to\infty} \int_0^T \operatorname{MMD}_\kappa^2\bigl(P_n(t),Q_\theta(t)\bigr)\diff t \qquad\text{a.s.}
$$
Dividing by $T$ and using the coordinate identification $\kappa_M$ completes the proof.
\end{proof}

\subsection{Proof of Proposition \ref{prop:time_varying_density}}\label{sec:proof_time_varying_density}

\begin{proof}
Fix $p\ge 1$, let $P\colon [0,T]\to\mathcal{P}_p(\mathcal{X})$ be continuous with respect to $W_p$, and fix $\varepsilon>0$. Because $[0,T]$ is compact and $P$ is continuous, the trajectory image $\mathcal{C}\coloneqq P([0,T])$ is compact in the metric space $(\mathcal{P}_p(\mathcal{X}),W_p)$.

By Proposition \ref{prop:approximation}, for every $\mu\in\mathcal{C}$, there exists a finite Gaussian mixture $q_\mu$ such that $W_p(\mu,q_\mu) < \varepsilon/4$. The open balls $U_\mu\coloneqq B_{W_p}(\mu,\varepsilon/4) \cap \mathcal{C}$ form an open cover of $\mathcal{C}$. By compactness, we can extract a finite subcover centered at $\mu_1,\dots,\mu_J\in\mathcal{C}$, with corresponding approximating mixtures $q_1,\dots,q_J$.

Let $\{\lambda_j\}_{j=1}^J$ be a continuous partition of unity on $\mathcal{C}$ subordinate to this finite subcover. That is, each $\lambda_j\colon \mathcal{C}\to[0,1]$ is continuous, $\sum_{j=1}^J \lambda_j(\mu)=1$ for all $\mu\in\mathcal{C}$, and $\lambda_j(\mu)=0$ whenever $\mu \notin U_j$.

Since each $q_j$ is a finite Gaussian mixture, we can pool all distinct Gaussian components appearing across $q_1,\dots,q_J$ into a single shared dictionary $\{\nu_1,\dots,\nu_K\}$. We can then express each local mixture over this shared basis as
$$
q_j = \sum_{k=1}^K a_{j,k}\nu_k, \qquad \text{where } a_{j,k}\ge 0 \text{ and } \sum_{k=1}^K a_{j,k}=1.
$$
We now construct the temporal weights by interpolating the coefficients:
$$
\pi_k(t)\coloneqq \sum_{j=1}^J \lambda_j(P(t))\,a_{j,k}, \qquad t\in[0,T],\quad k=1,\dots,K.
$$
Because $P$ and the partition functions $\lambda_j$ are continuous, the trajectories $t \mapsto \pi_k(t)$ are continuous. Furthermore, $\pi_k(t)\ge 0$ and $\sum_{k=1}^K \pi_k(t) = 1$. The temporal mixture is therefore given by
$$
Q(t)\coloneqq \sum_{k=1}^K \pi_k(t)\nu_k = \sum_{j=1}^J \lambda_j(P(t))\,q_j.
$$

To bound the distance between $P(t)$ and $Q(t)$, fix $t\in[0,T]$. We define the set of active indices at time $t$ as $A(t) \coloneqq \{j : \lambda_j(P(t)) > 0\}$. For each $j \in A(t)$, $P(t)\in U_j$, meaning $W_p(P(t),\mu_j) < \varepsilon/4$. By the triangle inequality,
$$
W_p(P(t),q_j) \le W_p(P(t),\mu_j) + W_p(\mu_j,q_j) < \varepsilon/2.
$$
For each $j \in A(t)$, choose an optimal transport plan $\gamma_{j,t} \in \Pi(P(t),q_j)$, which satisfies $\int \|x-y\|_{\mathcal{X}}^p\diff \gamma_{j,t}(x,y) = W_p^p(P(t),q_j) < (\varepsilon/2)^p$. We construct a mixed coupling:
$$
\gamma_t\coloneqq \sum_{j \in A(t)} \lambda_j(P(t))\,\gamma_{j,t}.
$$
Because each $\gamma_{j,t}$ has first marginal $P(t)$ and second marginal $q_j$, the convex combination $\gamma_t$ has first marginal $P(t)$ and second marginal $\sum_{j \in A(t)} \lambda_j(P(t))\,q_j = Q(t)$. Thus, $\gamma_t \in \Pi(P(t),Q(t))$.

Evaluating the transport cost of this coupling yields
$$
\begin{aligned}
W_p^p(P(t),Q(t)) &\le \int_{\mathcal{X}\times\mathcal{X}} \|x-y\|_{\mathcal{X}}^p\diff \gamma_t(x,y) \\
&= \sum_{j \in A(t)} \lambda_j(P(t)) \int_{\mathcal{X}\times\mathcal{X}} \|x-y\|_{\mathcal{X}}^p\diff \gamma_{j,t}(x,y).
\end{aligned}
$$
Since the integrals are strictly bounded by $(\varepsilon/2)^p$ and the sum of the active weights is 1, we conclude
$$
W_p^p(P(t),Q(t)) < \sum_{j \in A(t)} \lambda_j(P(t))\,(\varepsilon/2)^p = (\varepsilon/2)^p.
$$
Hence $W_p(P(t),Q(t))<\varepsilon/2$ for every $t\in[0,T]$, and therefore $\sup_{t\in[0,T]}W_p(P(t),Q(t))\le\varepsilon/2<\varepsilon$, as required.
\end{proof}

\subsection{Proof of Proposition \ref{prop:mmd_closed_forms}}\label{sec:proof_mmd_closed_forms}
\begin{proof}
This proof is inspired by \cite[Theorem~18]{wynneKernelTwosampleTest2022}.
Since $A$ is bounded, self-adjoint, and positive semidefinite, its square root $T\coloneqq A^{1/2}$ is a bounded, self-adjoint, positive semidefinite operator, and
 \[
 \kappa_A(x,y)=\exp\left(-\frac12\|Tx-Ty\|_{\mathcal X}^2\right).
\]
By the stability of Gaussian measures under bounded linear maps, if $y \sim \nu_k = \mathcal{N}(m_k,\mathcal{K}_k)$, the transformed variable is distributed as $Ty \sim \mathcal{N}(Tm_k,\,T\mathcal{K}_kT)$. Because $\mathcal{K}_k$ is trace-class and $T$ is bounded, the operator $T\mathcal{K}_kT$ remains trace-class. Thus, the general case reduces to evaluating the isotropic kernel $\kappa(x,y) = \exp(-\frac{1}{2}\|x-y\|_{\mathcal{X}}^2)$ on the transformed variables $Tx$ and $Ty$.

To evaluate the isotropic case, we rewrite the data-to-model expectation by centering the integration variable. Let $h \coloneqq X_i - m_k$ and write $y = m_k + \xi$, where $\xi \sim \mathcal{N}(0, \mathcal{K}_k)$. Since $X_i - y = h - \xi$, the expectation becomes:
\[
    J_{i,k} = \mathbb{E}_{y \sim \nu_k}\left[\exp\left(-\frac{1}{2}\|X_i - y\|_{\mathcal{X}}^2\right)\right] = \mathbb{E}_{\xi}\left[\exp\left(-\frac{1}{2}\|\xi - h\|_{\mathcal{X}}^2\right)\right].
\]
By the spectral theorem for compact self-adjoint operators, there exists an orthonormal eigenbasis $(e_r)_{r \ge 1}$ of $\mathcal{K}_k$ with corresponding eigenvalues $(\lambda_r)_{r \ge 1}$. Expanding $\xi$ and $h$ in this basis yields
\[
    \xi = \sum_{r=1}^\infty \sqrt{\lambda_r} Z_r e_r, \qquad h = \sum_{r=1}^\infty h_r e_r,
\]
where $Z_r \stackrel{\mathrm{i.i.d.}}{\sim} \mathcal{N}(0,1)$.  For any $N \in \mathbb{N}$, define $\xi^{(N)} \coloneqq \sum_{r=1}^N \sqrt{\lambda_r} Z_r e_r$ and $h^{(N)} \coloneqq \sum_{r=1}^N h_r e_r$. By Parseval's identity, the squared norm of the difference is simply the limit of its partial sums:
\[
    \|\xi^{(N)} - h^{(N)}\|_{\mathcal{X}}^2 = \sum_{r=1}^N (\sqrt{\lambda_r} Z_r - h_r)^2 \xrightarrow{N \to \infty} \sum_{r=1}^\infty (\sqrt{\lambda_r} Z_r - h_r)^2 = \|\xi - h\|_{\mathcal{X}}^2 \quad \text{a.s.}
\]
Since this convergence holds almost surely and $0 \le \exp\left(-\frac{1}{2}\|\xi^{(N)} - h^{(N)}\|_{\mathcal{X}}^2\right) \le 1$, the dominated convergence theorem allows us to compute the expectation as the limit of the finite-dimensional projections:
\begin{align*}
    J_{i,k}
    = \lim_{N \to \infty} \mathbb{E}\left[\exp\left(-\frac{1}{2}\|\xi^{(N)} - h^{(N)}\|_{\mathcal{X}}^2\right)\right] = \lim_{N \to \infty} \prod_{r=1}^N \mathbb{E}\left[\exp\left(-\frac{1}{2}(\sqrt{\lambda_r}Z_r - h_r)^2\right)\right].
\end{align*}
Expanding the square inside the exponential yields $-\frac{1}{2}(\sqrt{\lambda_r}Z_r - h_r)^2 = -\frac{\lambda_r}{2}Z_r^2 + h_r\sqrt{\lambda_r}Z_r - \frac{h_r^2}{2}$. Factoring out the constant term, we obtain:
$$
    \mathbb{E}\left[\exp\left(-\frac{1}{2}(\sqrt{\lambda_r}Z_r - h_r)^2\right)\right] 
    = e^{-h_r^2/2} \mathbb{E}\left[\exp\left(h_r\sqrt{\lambda_r}Z_r - \frac{\lambda_r}{2}Z_r^2\right)\right].
$$
Using the Gaussian identity $\mathbb{E}[\exp(aZ + bZ^2)] = (1-2b)^{-1/2}\exp\bigl(\frac{a^2}{2(1-2b)}\bigr)$ for $Z \sim \mathcal{N}(0,1)$, which holds if $2b<1$, with $a = h_r\sqrt{\lambda_r}$ and $b = -\lambda_r/2$, each factor evaluates to:
$$
    e^{-h_r^2/2} (1+\lambda_r)^{-1/2} \exp\left(\frac{h_r^2\lambda_r}{2(1+\lambda_r)}\right)
    = (1+\lambda_r)^{-1/2} \exp\left(-\frac{1}{2}\frac{h_r^2}{1+\lambda_r}\right).
$$
Substituting the one-dimensional evaluations back into the finite-dimensional product yields:
\begin{equation*}
    \mathbb{E}\left[\exp\left(-\frac{1}{2}\|\xi^{(N)} - h^{(N)}\|_{\mathcal{X}}^2\right)\right] 
    = \left(\prod_{r=1}^N (1+\lambda_r)^{-1/2}\right) \exp\left(-\frac{1}{2}\sum_{r=1}^N \frac{h_r^2}{1+\lambda_r}\right).
\end{equation*}

Passing to the limit as $N \to \infty$, the infinite product converges to $\det_F(\Id + \mathcal{K}_k)^{-1/2}$. Simultaneously, by the spectral decomposition of $(\Id + \mathcal{K}_k)^{-1}$, the series in the exponent evaluates to
\begin{equation*}
    \sum_{r=1}^\infty \frac{h_r^2}{1+\lambda_r} = \big\langle h, (\Id + \mathcal{K}_k)^{-1} h \big\rangle_{\mathcal{X}}.
\end{equation*}
Recalling that $h = X_i - m_k$, we conclude:
\begin{equation*}
    J_{i,k} = \det\nolimits_F(\Id + \mathcal{K}_k)^{-1/2} \exp\left(-\frac{1}{2} \big\langle X_i - m_k, (\Id + \mathcal{K}_k)^{-1} (X_i - m_k) \big\rangle_{\mathcal{X}}\right).
\end{equation*}
Applying this isotropic solution to the transformed variables $TX_i$ and $Ty \sim \mathcal{N}(Tm_k, T\mathcal{K}_kT)$, we obtain the intermediate form:
$$
    J_{i,k} = \det\nolimits_F(\Id + T\mathcal{K}_k T)^{-1/2} \exp\left(-\frac{1}{2} \big\langle T(X_i - m_k), (\Id + T\mathcal{K}_k T)^{-1} T(X_i - m_k) \big\rangle_{\mathcal{X}}\right).
$$
Substituting $T = A^{1/2}$ recovers Equation \eqref{eq:J_gaussian}.

Finally, for the cross-expectation $I_{k,s}$, let $y \sim \nu_k$ and $y' \sim \nu_s$ be independent. Their difference $z \coloneqq y - y'$ is distributed as $\mathcal{N}(m_k - m_s, \mathcal{K}_k + \mathcal{K}_s)$. Since $\kappa_A(y,y') = \kappa_A(z,0)$, applying the formula for $J_{i,k}$ with $X_i=0$ immediately yields:
\begin{equation*}
    I_{k,s} = \frac{\exp\Bigl(-\frac{1}{2}\big\langle m_k-m_s,\, A^{1/2}\bigl(\Id + A^{1/2}(\mathcal{K}_k+\mathcal{K}_s)A^{1/2}\bigr)^{-1}A^{1/2}(m_k-m_s)\big\rangle_{\mathcal{X}}\Bigr)}{\det_F\bigl(\Id + A^{1/2}(\mathcal{K}_k+\mathcal{K}_s)A^{1/2}\bigr)^{1/2}},
\end{equation*}
which completes the proof.

\end{proof}

\subsection{Proof of Proposition \ref{prop:mmd_closed_forms_polynomial}}\label{sec:proof_mmd_closed_forms_polynomial}
\begin{proof}
Write $y=m_k+\xi$, where $\xi\sim\mathcal N(0,\mathcal K_k)$. Then
$$
Z\coloneqq \langle X_i,y\rangle_{\mathcal X}+c
=
\underbrace{\langle X_i,m_k\rangle_{\mathcal X}+c}_{\mu_{i,k}}
+
\langle X_i,\xi\rangle_{\mathcal X}
$$
is a real Gaussian random variable with mean $\mu_{i,k}\in\R$ and variance
$$
v_{i,k}
\coloneqq
\mathbb E[\langle X_i,\xi\rangle_{\mathcal X}^2]
=
\langle X_i,\mathcal K_kX_i\rangle_{\mathcal X}.
$$
Thus,
$$
Z \stackrel{d}{=} \mu_{i,k}+\sqrt{v_{i,k}}\,G,
\qquad
G\sim\mathcal N(0,1),
$$
and therefore
$$
J_{i,k}
=
\mathbb E[Z^p]
=
\sum_{r=0}^p \binom{p}{r}\mu_{i,k}^{\,p-r}v_{i,k}^{r/2}\,\mathbb E[G^r].
$$
Since $\mathbb E[G^r]=0$ for odd $r$ and $\mathbb E[G^{2t}]=\frac{(2t)!}{2^t t!},$ we obtain
$$
J_{i,k}
=
\sum_{t=0}^{\lfloor p/2\rfloor}
\frac{p!}{(p-2t)!\,2^t\,t!}
\,\langle X_i,\mathcal K_kX_i\rangle_{\mathcal X}^{\,t}
\bigl(\langle X_i,m_k\rangle_{\mathcal X}+c\bigr)^{p-2t}.
$$

Now let $y\sim \nu_k$ and $y'\sim \nu_s$ be independent. By definition of the polynomial kernel, 
$$
I_{k,s} = \mathbb E\big[(\langle y,y'\rangle_{\mathcal X} + c)^p\big].
$$
Applying the binomial theorem to expand the power, and using the linearity of expectation, 
$$
I_{k,s}
=
\mathbb E\Bigg[\sum_{r=0}^p \binom{p}{r}c^{p-r}\langle y,y'\rangle_{\mathcal X}^r\Bigg]
=
\sum_{r=0}^p \binom{p}{r}c^{p-r}\,
\mathbb E\big[\langle y,y'\rangle_{\mathcal X}^r\big].
$$
To evaluate this expectation, recall that $\mathcal X^{\otimes r}$ denotes the $r$-th tensor power of $\mathcal X$. For any vector $y \in \mathcal X$, its $r$-fold tensor product is $y^{\otimes r} \coloneqq y \otimes \dots \otimes y$. By the standard definition of the inner product on $\mathcal{X}^{\otimes r}$, we have the fundamental identity:
$$
\langle y,y'\rangle_{\mathcal X}^r
=
\langle y^{\otimes r},(y')^{\otimes r}\rangle_{\mathcal X^{\otimes r}},
$$
where we adopt the convention $\mathcal X^{\otimes 0}=\mathbb R$ and $y^{\otimes 0}=1$ for $r=0$.

Since Gaussian measures on separable Hilbert spaces have finite moments of all orders, the Bochner moments
$$
M_{r,k}\coloneqq \mathbb E[y^{\otimes r}] \in \mathcal X^{\otimes r},
\qquad
M_{r,s}\coloneqq \mathbb E[(y')^{\otimes r}] \in \mathcal X^{\otimes r}
$$
are well-defined. Using independence, the expectation of the inner product factors as the inner product of the expectations:
$$
\mathbb E\big[\langle y^{\otimes r},(y')^{\otimes r}\rangle_{\mathcal X^{\otimes r}}\big]
=
\big\langle \mathbb E[y^{\otimes r}],\,\mathbb E[(y')^{\otimes r}]\big\rangle_{\mathcal X^{\otimes r}}
=
\langle M_{r,k},M_{r,s}\rangle_{\mathcal X^{\otimes r}}.
$$
Substituting this back into the binomial expansion yields
$$
I_{k,s}
=
\sum_{r=0}^{p}\binom{p}{r}c^{p-r}\,\langle M_{r,k},M_{r,s}\rangle_{\mathcal X^{\otimes r}},
$$
which completes the proof.
\end{proof}

\subsection{Proof of Proposition \ref{prop:projection_likelihood}}\label{sec:proof_projection_likelihood}
\begin{proof}
Let $\rho$ be as in the proposition and define $q=dQ/d\rho\in L^1(\rho)$, so that $\int q\diff \rho=Q(\mathcal{X})=1$.
Since $Q\ll\rho$, the projected measure $Q_M=Q\circ\Pi_M^{-1}$ satisfies $Q_M\ll\rho_M$, so $q_M=dQ_M/d\rho_M$ is well-defined.
Let $X$ denote the identity random element on the probability space $(\mathcal{X},\mathcal{B}(\mathcal{X}),\rho)$ and set $X^{(M)}\coloneqq \Pi_M X$.
Write $\mathcal{F}_M\coloneqq \sigma(X^{(M)})$ and $\mathcal{F}_\infty\coloneqq \sigma(\cup_{M\ge 1}\mathcal{F}_M)$.
The same projection argument used in the proof of Proposition \ref{prop:projection_responsibilities} gives $\mathcal{F}_\infty=\sigma(X)=\mathcal{B}(\mathcal{X})$.

\medskip
\noindent\textbf{Step 1 (identification as a conditional expectation).}
By definition, $q_M=dQ_M/d\rho_M$ and for each $B\in\mathcal{B}(\mathcal{X}_M)$,
\[
\int_B q_M(u)\diff \rho_M(u)
= Q_M(B)
= Q(\Pi_M^{-1}(B))
= \int_{\Pi_M^{-1}(B)} q(x)\diff \rho(x).
\]
Since $\rho_M=\rho\circ \Pi_M^{-1}$, the left-hand side equals
\[
\int_{\mathcal{X}} \mathbf{1}_{\{X^{(M)}\in B\}}\,q_M(X^{(M)})\diff \rho.
\]
Combining these identities, we obtain
\[
\int_{\mathcal{X}} \mathbf{1}_{\{X^{(M)}\in B\}}\,q_M(X^{(M)})\diff \rho
= \int_{\mathcal{X}} \mathbf{1}_{\{X^{(M)}\in B\}}\,q(X)\diff \rho
\qquad\forall B\in\mathcal{B}(\mathcal{X}_M),
\]
so $q_M(X^{(M)})$ is a version of $\mathbb{E}_{\rho}[q(X)\mid \mathcal{F}_M]$.
Equivalently, $q_M(\Pi_M x)=\mathbb{E}_{\rho}[q(X)\mid \mathcal{F}_M](x)$ for $\rho$-a.e.\ $x$.

\medskip
\noindent\textbf{Step 2 (martingale convergence).}
Since $(\mathcal{F}_M)_{M\ge 1}$ is an increasing filtration and $q(X)\in L^1(\rho)$, the martingale convergence theorem yields
\[
\mathbb{E}_{\rho}[q(X)\mid \mathcal{F}_M]\longrightarrow \mathbb{E}_{\rho}[q(X)\mid \mathcal{F}_\infty]=q(X)
\qquad \rho\text{-a.s. and in }L^1(\rho).
\]
Combining with Step~1 shows $q_M(\Pi_M x)\to q(x)$ for $\rho$-a.e.\ $x$, and in $L^1(\rho)$.

\medskip
\noindent\textbf{Step 3 (log-likelihood convergence under $Q$).}
Since $Q\ll \rho$, the almost sure convergence in Step~2 also holds $Q$-a.s.
Moreover, $q>0$ $Q$-a.s.\ and $q_M>0$ $Q_M$-a.s.; hence $q_M(\Pi_M X)>0$ for $Q$-a.e.\ $X$, and $\log q(X)$ and $\log q_M(\Pi_M X)$ are well-defined $Q$-a.s.
Therefore, for i.i.d.\ $X_1,\dots,X_n\sim Q$,
\[
\log q_M(\Pi_M X_i)\longrightarrow \log q(X_i)
\qquad\text{a.s. for each }i,
\]
and summing over $i=1,\dots,n$ yields the claimed convergence of the log-likelihoods.
\end{proof}


\end{document}